\documentclass[11pt]{article}
\pdfoutput=1
\usepackage[T1]{fontenc}
\usepackage{lmodern}
\usepackage{slantsc}
\usepackage[protrusion=true,expansion=true]{microtype}
\usepackage{amsmath,amssymb,amsfonts,amsthm}
\usepackage{subcaption}
\usepackage{graphicx}
\usepackage{fullpage}
\usepackage{setspace}
\usepackage[backref=page]{hyperref}
\usepackage{color}
\usepackage{wrapfig}
\usepackage{tikz}
\usetikzlibrary{decorations.pathreplacing}
\usepackage{algorithm}
\usepackage[noend]{algpseudocode}
\usepackage[framemethod=tikz]{mdframed}
\usepackage{xspace}
\usepackage{pgfplots}
\usepackage{framed}
\usepackage{thmtools}
\usepackage{thm-restate}
\usepackage{tabu}
\usepackage{fancyhdr}
\pgfplotsset{compat=1.5}
\usepackage{bbm}
\usepackage{enumitem}

\usepackage{tcolorbox}
\tcbuselibrary{skins,breakable}
\tcbset{enhanced jigsaw}

\newtheorem{theorem}{Theorem}
\newtheorem{corollary}[theorem]{Corollary}
\newtheorem{lemma}{Lemma}[section]
\newtheorem{proposition}{Proposition}[section]
\newtheorem{definition}{Definition}

\newtheorem{observation}[lemma]{Observation}

\newenvironment{proofof}[1]{\begin{trivlist} \item {\bf Proof
#1:~~}}
  {\qed\end{trivlist}}

\newcommand{\namedref}[2]{\hyperref[#2]{#1~\ref*{#2}}}

\newcommand{\lemlab}[1]{\label{lem:#1}}
\newcommand{\lemref}[1]{\namedref{Lemma}{lem:#1}}

\newcommand{\alglab}[1]{\label{alg:#1}}
\renewcommand{\algref}[1]{\namedref{Algorithm}{alg:#1}}

\newcommand{\deflab}[1]{\label{def:#1}}
\newcommand{\defref}[1]{\namedref{Definition}{def:#1}}

\renewcommand{\algorithmicrequire}{\textbf{Input:}\,}
\renewcommand{\algorithmicensure}{\textbf{Output:}\,}

\newcommand{\PPr}[1]{\ensuremath{\mathbf{Pr}\left[#1\right]}}

\newcommand{\Ex}[1]{\ensuremath{\mathbb{E}\left[#1\right]}}

\renewcommand{\O}[1]{\ensuremath{\mathcal{O}\left(#1\right)}}
\newcommand{\tO}[1]{\ensuremath{\tilde{\mathcal{O}}\left(#1\right)}}
\newcommand{\eps}{\varepsilon}

\def \calA    {\mdef{\mathcal{A}}}

\def \calD    {\mdef{\mathcal{D}}}
\def \calE    {\mdef{\mathcal{E}}}
\def \calF    {\mdef{\mathcal{F}}}

\def \calH    {\mdef{\mathcal{H}}}
\def \calI    {\mdef{\mathcal{I}}}

\def \calK    {\mdef{\mathcal{K}}}
\def \calL    {\mdef{\mathcal{L}}}

\def \calP    {\mdef{\mathcal{P}}}
\def \calQ    {\mdef{\mathcal{Q}}}

\def \calU    {\mdef{\mathcal{U}}}

\def \calW    {\mdef{\mathcal{W}}}
\def \calX    {\mdef{\mathcal{X}}}

\newcommand{\mdef}[1]{{\ensuremath{#1}}\xspace}  

\DeclareMathOperator*{\argmin}{argmin}
\DeclareMathOperator*{\argmax}{argmax}
\DeclareMathOperator*{\polylog}{polylog}
\DeclareMathOperator*{\poly}{poly}

\newcommand{\flr}[1]{\mdef{\left\lfloor#1\right\rfloor}}              

\newcommand{\ignore}[1]{}

\newif\ifnotes\notestrue 
\ifnotes
\newcommand{\samson}[1]{\textcolor{blue}{{\bf (Samson:} {#1}{\bf ) }} \marginpar{\tiny\bf
             \begin{minipage}[t]{0.5in}
               \raggedright S:
            \end{minipage}}}
\newcommand{\david}[1]{\textcolor{purple}{{\bf (David:} {#1}{\bf ) }} \marginpar{\tiny\bf
             \begin{minipage}[t]{0.5in}
               \raggedright D:
            \end{minipage}}} 
\else
\newcommand{\samson}[1]{}
\newcommand{\david}[1]{}
\fi

\makeatletter
\renewcommand*{\@fnsymbol}[1]{\textcolor{mahogany}{\ensuremath{\ifcase#1\or *\or \dagger\or \ddagger\or
 \mathsection\or \triangledown\or \mathparagraph\or \|\or **\or \dagger\dagger
   \or \ddagger\ddagger \else\@ctrerr\fi}}}
\makeatother

\providecommand{\email}[1]{\href{mailto:#1}{\nolinkurl{#1}\xspace}}

\definecolor{mahogany}{rgb}{0.75, 0.25, 0.0}
\definecolor{darkblue}{rgb}{0.0, 0.0, 0.55}
\definecolor{darkpastelgreen}{rgb}{0.01, 0.75, 0.24}
\definecolor{darkgreen}{rgb}{0.0, 0.2, 0.13}
\definecolor{darkgoldenrod}{rgb}{0.72, 0.53, 0.04}
\definecolor{darkred}{rgb}{0.55, 0.0, 0.0}
\definecolor{forestgreenweb}{rgb}{0.13, 0.55, 0.13}
\definecolor{greencss}{rgb}{0.0, 0.5, 0.0}
\definecolor{bleudefrance}{rgb}{0.19, 0.55, 0.91}

\hypersetup{
     colorlinks   = true,
     citecolor    = mahogany,
	 linkcolor	  = forestgreenweb
}

\fancypagestyle{pg}
{
\lhead{}
\rhead{}
\cfoot{--\ \thepage\ --}

}

\AtBeginDocument{%
  \DeclareFontShape{T1}{lmr}{m}{scit}{<->ssub*lmr/m/scsl}{}%
}

\newenvironment{tbox}{\begin{tcolorbox}[
		enlarge top by=5pt,
		enlarge bottom by=5pt,
		 breakable,
		 boxsep=0pt,
                  left=4pt,
                  right=4pt,
                  top=10pt,
                  arc=0pt,
                  boxrule=1pt,toprule=1pt,
                  colback=white
                  ]
	}
{\end{tcolorbox}}

\newcommand{\bern}[1]{\textnormal{\textsf{Bern}}(#1)\xspace}

\def\NoNumber#1{{\def\alglinenumber##1{}\\ #1}\addtocounter{ALG@line}{-1}}

\definecolor{DarkRed}{rgb}{0.5,0.1,0.1}
\definecolor{RUScarlet}{rgb}{0.95,0.1,0.1}
\definecolor{DarkBlue}{rgb}{0.1,0.1,0.5}

\usepackage{placeins}

\newcommand{\bracket}[1]{\left[#1\right]}
\newcommand{\paren}[1]{\ensuremath{\left(#1\right)}\xspace}
\newcommand{\card}[1]{\left\vert{#1}\right\vert}

\newcommand{\pw}{\textsf{pw}}
\newcommand{\baseline}{\textsc{Baseline}}

\newcommand{\ALG}{\ensuremath{\textnormal{\texttt{ALG}}}\xspace}
\newcommand{\hedge}{\textsc{Hedge}\xspace}

\newcommand{\istar}{\ensuremath{i^{*}}\xspace}

\newcommand{\benchmark}{\ensuremath{\mathsf{BM}}}

\newcommand{\Otilde}{\ensuremath{\widetilde{O}}\xspace}

\newcommand{\Nmeta}{\ensuremath{N^{\textnormal{meta}}}}

\renewcommand{\leq}{\leqslant}
\renewcommand{\le}{\leq}
\renewcommand{\geq}{\geqslant}
\renewcommand{\ge}{\geq} 

\newcommand{\expect}[1]{\mathbb{E}\bracket{#1}}
\newcommand{\expectrand}[2]{\mathbb{E}_{#1}\bracket{#2}}

\usepackage{tcolorbox}
\tcbuselibrary{skins,breakable}
\tcbset{enhanced jigsaw}

\definecolor{DarkRed}{rgb}{0.5,0.1,0.1}
\definecolor{RURed}{rgb}{0.95,0.1,0.1}
\definecolor{DarkBlue}{rgb}{0.1,0.1,0.5}

\newtheorem{mdresult}{Result}

\definecolor{Red}{rgb}{0.9,0,0}

\usepackage{nameref}
\usepackage[nameinlink]{cleveref}

\crefname{property}{property}{Property}
\creflabelformat{property}{(#1)#2#3}

\crefname{equation}{eq}{Eq}
\creflabelformat{equation}{(#1)#2#3}

\crefname{thm}{theorem}{Theorem}
\creflabelformat{Theorem}{(#1)#2#3}

\crefname{algocf}{Algorithm}{Algorithm}
\creflabelformat{algocf}{#1#2#3}

\theoremstyle{definition}
\newtheorem{mdalg}{Algorithm}

\begin{document}

\title{Online Learning with Limited Information\\ in the Sliding Window Model}

\author{
Vladimir Braverman \\ Johns Hopkins University \\ 
\email{vova@cs.jhu.edu}
\and
Sumegha Garg \\ Rutgers University \\ 
\email{sumegha.garg@rutgers.com}
\and
Chen Wang \\ Rensselaer Polytechnic Institute \\ 
\email{wangc33@rpi.edu}
\and 
David P. Woodruff \\ Carnegie Mellon University \\ 
\email{dwoodruf@andrew.cmu.edu}
\and
Samson Zhou \\ Texas A\&M University \\ 
\email{samsonzhou@gmail.com}
}
\date{} 

\maketitle
\thispagestyle{empty}

\begin{abstract}
Motivated by recent work on the experts problem in the streaming model, we consider the experts problem in the sliding window model. The sliding window model is a well-studied model that captures applications such as traffic monitoring, epidemic tracking, and automated trading, where recent information is more valuable than older data. Formally, we have $n$ experts, $T$ days, the ability to query the predictions of $q$ experts on each day, a limited amount of memory, and should achieve the (near-)optimal regret $\sqrt{nW}\text{polylog}(nT)$ regret over any window of the last $W$ days. While it is impossible to achieve such regret with $1$ query, we show that with $2$ queries we can achieve such regret and with only $\text{polylog}(nT)$ bits of memory. Not only are our algorithms optimal for sliding windows, but we also show for every interval $\mathcal{I}$ of days that we achieve $\sqrt{n|\mathcal{I}|}\text{polylog}(nT)$ regret with $2$ queries and only $\text{polylog}(nT)$ bits of memory, providing an exponential improvement on the memory of previous interval regret algorithms. Building upon these techniques, we address the bandit problem in data streams, where $q=1$, achieving $n T^{2/3}\text{polylog}(T)$ regret with $\text{polylog}(nT)$ memory, which is the first sublinear regret in the streaming model in the bandit setting with polylogarithmic memory; this can be further improved to the optimal $\mathcal{O}(\sqrt{nT})$ regret if the best expert's losses are in a random order.
\end{abstract}

\clearpage
\setcounter{page}{1}

\allowdisplaybreaks

\section{Introduction}
\label{sec:intro}
The online learning with experts problem is a fundamental framework for sequential prediction, where an algorithm utilizes expert forecasts to make decisions over time. 
Over a time horizon of $T$ days (or time steps), an algorithm must make predictions about unknown outcomes based on the advice of $n$ predefined experts, also referred to as ``arms'' in the multi-armed bandits literature.\footnote{We use the terms ``experts'' and ``arms'' interchangeably.} 
On each day, the algorithm observes the predictions of a set of experts and then produces its own prediction based on both past observations and the current expert forecasts. 
After making its prediction, the algorithm receives feedback in the form of a cost indicating the accuracy of its decision, as well as the costs for a subset of expert predictions queried that day.  
Generally, the costs are restricted to be in some range $[0,\rho]$ for a parameter $\rho>0$, which we normalize to $1$. 
The iterative process continues over time, enabling the algorithm to refine its strategy. 

Formally, on each day $t\in[T]$, the algorithm chooses to play an arm $i_t\in[n]$, thus incurring loss $\ell^t(i_t)$. 
Additionally, it can query the losses of an additional set $S_t\subseteq[n]$ of arms with no loss, though we say that $\max_{t\in[T]}|S_t|$ is the query complexity of the algorithm. 
The sequence of expert predictions and losses is fixed in advance, and we make no distributional assumptions; that is, the input is worst-case but oblivious (non-adaptive) to the algorithm’s outputs. 
The goal is to minimize the \emph{regret}, defined as the cumulative gap between the loss sequence we obtain and that of the best arm in hindsight, i.e.,
\begin{equation}
\label{equ:regret}
\sum_{t\in[T]}\ell^t(i_t) - \min_{i\in [n]} \sum_{t\in[T]}\ell^t(i).
\end{equation}
In the case where the algorithm can freely observe all experts at all times, i.e., $S_t=[n]$ for all $t\in[T]$, the randomized weighted majority algorithm can achieve regret $\O{\sqrt{T\log n}}$~\cite{LittlestoneW94}. 
There are a number of subsequent variants, such as Hedge~\cite{Cesa-BianchiFHHSW97,FreundS97} and follow the perturbed leader \cite{KalaiV05}, that also achieve regret $\O{\sqrt{T\log n}}$, which is known to be optimal~\cite{cover1966behavior}. 
However, these algorithms follow a strategy that involves maintaining the cumulative cost for each expert, which requires $\Omega(n)$ bits of space, as well as $\Omega(n)$ update time. 

In many practical settings, the number $n$ of experts and the time horizon $T$ can be large, so that storing and processing all expert predictions can be both computationally-intensive and memory-expensive. 
Hence, a line of recent work~\cite{SrinivasWXZ22,AamandCNS23,pr23,PengZ23,WoodruffZZ23a} has focused on the online learning with experts problem with bounded memory, in particular in the \emph{data stream model}, where the algorithm processes the input in a single pass, while using working memory sublinear in the size $n$ of the input. 
In particular, an algorithm cannot store past expert predictions in full, so it cannot use the previous approaches and instead must strategically decide which information to retain while ensuring its regret remains competitive. 
Nevertheless, \cite{pr23} introduced an algorithm that uses $\polylog(nT)$ space and achieves $\sqrt{nT}\cdot\polylog(nT)$ regret. 
Thus, in some sense, there do not seem to be expensive tradeoffs between regret and space complexity.   

\paragraph{The sliding window model and interval regret.}
Although the data stream model provides a framework for analyzing algorithms that process large-scale data with limited memory, it does not capture the reality that in many applications, recent information can often be more valuable than older data~\cite{BabcockBDMW02,BabcockDMO03,PapapetrouGD15}, e.g., traffic monitoring, where real-time congestion data informs routing decisions; epidemic tracking, where recent infection rates guide public health responses; and automated trading, where the latest market signals drive investment strategies. 
This motivates the \emph{sliding window model}, where the dataset at any time consists of only the most recent $W$ updates in the stream. 
Here, the parameter $W > 0$ defines the window size, and updates older than the $W$ most recent stream elements are considered expired. 
The goal is to aggregate statistics about the active data using space that is sublinear in $W$. 
The sliding window model can be seen as a generalization of the streaming model, as the parameter $W$ can be set to the entire stream length, and is particularly relevant for time-sensitive scenarios, e.g., data summarization~\cite{ChenNZ16,EpastoLVZ17}, event detection in social media~\cite{osborne2014real}, and network traffic monitoring~\cite{cormode2005s,cormode2007streaming,Cormode13}. 

From a practical perspective, the sliding window model captures real-world constraints on data retention. 
For example, as a result of regulatory policies such as the General Data Protection Regulation (GDPR) that impose strict limits on how long certain user data can be stored~\cite{GDPR2016}, Apple stores user information for $3$ months~\cite{apple-data}, ChatGPT stores user conversations for at most $30$ days~\cite{openai-data}, and Google preserves browsing data for up to $9$ months~\cite{google-data}. 
By appropriately setting the window size parameter $W$, the sliding window model effectively captures such constraints and has been widely studied across various domains~\cite{DatarGIM02,LeeT06a,LeeT06b,BravermanO07,CrouchMS13,BravermanLLM15,BravermanLLM16,BravermanWZ21,BravermanGLWZ18,WoodruffZ21,BorassiELVZ20,EpastoMMZ22a,JayaramWZ22,BlockiLMZ23,WoodruffY23,WoodruffZZ23,Cohen-AddadJYZZ25}. 
We remark that due to the implicit expiration of data outside the active window, the sliding window model typically demands algorithmic techniques that differ from those used in the standard streaming model.  
For example, linear sketching methods that aggregate updates across the entire stream are generally incompatible with the ability to discard implicitly outdated information.

In the context of the online learning with experts problem, the regret for the sliding window model is measured at each time $t\in[T]$ by comparing the performance of the algorithm with the performance of the best arm over the last $W$ times. 
Formally, the regret at time $t$ is defined as $\sum_{s = t - W + 1}^{t} \ell^s(i_s) - \min_{i \in [n]} \sum_{s = t - W + 1}^{t} \ell^s(i)$ and the overall regret of the algorithm in the sliding-window model is then defined as the maximum regret over all time steps. 
Incidentally, this notion aligns with the stronger notion of \emph{interval regret}, which is a natural refinement of regret analysis that evaluates performance over any contiguous subinterval of time rather than the entire horizon. 
Formally, for any interval $\calI=[t_1, t_2]\subseteq[T]$, the interval regret is defined as  
\begin{equation}
\label{equ:interval-regret}
\sum_{t=t_1}^{t_2} \ell^t(i_t) - \min_{i \in [n]} \sum_{t=t_1}^{t_2} \ell^t(i).
\end{equation}  
This metric provides a more fine-grained understanding of the adaptability of an algorithm, capturing how well it performs not just in the long run but also over shorter, potentially dynamic time periods. 
Observe that the sliding-window regret is a special case of the interval regret corresponding to the case when the maximum is taken over intervals $\calI$ with length $\card{\calI}=W$. 
Surprisingly, while the sliding window model has been extensively studied, there is no previous work studying the online learning with experts problem with memory constraints, to the best of our knowledge. 

\subsection{Our Contributions}
\label{subsec:contribution}
In this work, we initiate the study of the online learning with experts problem in the sliding window model. 
We first recall that in the single-query bandit setting, any algorithm necessarily incurs interval regret $\Omega(\card{\calI}^{1-\eps})$ for all fixed constants $\eps>0$ \cite{DanielyGS15}. 
Since the interval regret corresponds exactly to the worst-case regret of the algorithm across the sliding window model, any single-query algorithm for online learning with experts in the sliding window model must incur regret $\Omega(W^{1-\eps})$ for all fixed constants $\eps>0$ \cite{DanielyGS15}. 

We show that with a single additional query, it is possible to achieve not only optimal interval regret, but also only using polylogarithmic memory. 
Specifically, one of the queries is required to follow an expert and incur some loss by the algorithm, while the other query is allowed to freely observe an arbitrary expert, which may vary between different times, without incurring loss. 
\begin{restatable}{theorem}{rootIinterval}
\label{thm:two-query-interval} 
There exists an online learning algorithm that given any instance of $n$ experts and $T$ days such that $T\geq n$ and two queries per time, achieves $\sqrt{n \card{\calI}}\cdot \polylog(T)$ interval regret for any interval $\calI$ using $\polylog(T)$ words of memory with high probability, i.e., $1-\frac{1}{\poly(nT)}$. 
\end{restatable}
We observe that since interval regret is a strictly stronger notion than expected regret over the entire $T$ days, \Cref{thm:two-query-interval} also immediately implies a two-query algorithm that achieves $\sqrt{nT}\cdot\polylog(T)$ regret for the online learning with experts problem over a time horizon of length $T$, using polylogarithmic space. 
In fact, since $\Omega(\sqrt{nT})$ regret is known to be necessary for the online learning experts problem with any constant number of queries per time~\cite{AuerCFS02}, our algorithm achieves the optimal regret, up to polylogarithmic factors. 

Importantly, while \cite{Lu0CZWH24} similarly achieved $\sqrt{nT}\cdot\polylog(T)$ interval regret using two queries, \Cref{thm:two-query-interval} uses polylogarithmic space, while the result of \cite{Lu0CZWH24} uses linear space to track the losses of all $n$ experts. 
This difference is crucial for the sliding window model, where sublinear space complexity is required. 
Thus, \Cref{thm:two-query-interval} immediately achieves the first two-query algorithm for the sliding window model.
\begin{corollary}
\label{cor:sliding:window}
There exists an algorithm for the online learning with experts problem in the sliding window model that uses two queries per time. 
For any instance of $n$ experts over a sliding window of size $W\ge n$ on a time horizon of length $T$, the algorithm achieves $\sqrt{nW}\cdot\polylog(T)$ regret using $\polylog(T)$ bits of space with high probability, i.e., $1-\frac{1}{\poly(nT)}$. 
\end{corollary}
We again emphasize that \Cref{cor:sliding:window} is optimal for query complexity and near-optimal for both regret and space complexity, since (1) \cite{AuerCFS02} shows that $\Omega(\sqrt{nW})$ regret is necessary in any setting with a constant number of queries, (2) the result of \cite{DanielyGS15} shows that even with unlimited memory, $\Omega(W^{1-\eps})$ regret is necessary with a single query per time, and (3) logarithmic memory is necessary simply to store the identity of any ``good'' expert. 
Thus, in conjunction with the results of \cite{AuerCFS02,DanielyGS15}, \Cref{cor:sliding:window} resolves the online learning with experts problem in the sliding window model. 

We remark that the techniques of \Cref{thm:two-query-interval} can be adapted to handle a number of other settings beyond interval regret and the sliding window model. 
For instance, it is natural to ask what regret is achievable by memory-bounded algorithms that only have single-query bandit feedback, without the need to guarantee the stronger notion of interval regret. 
To that end, we provide the following result.
\begin{restatable}{theorem}{Ttwooverthreeonequery}
\label{thm:one-query} 
There exists an online learning algorithm that given any instance of $n$ experts and $T$ days such that $T\geq n$ and the query access of a single expert, i.e., the bandit setting, achieves $n T^{2/3}\cdot \polylog(T)$ regret using $\polylog(nT)$ words of memory with probability at least $1-1/\poly(nT)$. 
\end{restatable}
To the best of our knowledge, the only prior algorithm for online learning with sublinear memory in the single-query bandit setting achieves $\O{L\cdot n^{1/2L} \cdot T^{1-\frac{1}{2L}}}$ regret with $\Theta(n^{1/L})$ space for integers $L\ge 1$~\cite{XuZ21}. 
Notably, their results achieve $\O{n^{1/4}T^{3/4}}$ regret with $\Theta(\sqrt{n})$ space. 
By comparison, \Cref{thm:one-query} achieves strictly stronger regret while using exponentially less space. 

Thus, given the previous discussion of interval regret, \Cref{thm:one-query} furthers the strong separation between the achievable guarantees for the single-query bandit setting for the streaming model and the sliding window model. 
Indeed, for the sliding window model, $\Omega(W^{1-\eps})$ regret is necessary for any $\eps>0$ \cite{DanielyGS15}, while \Cref{thm:one-query} achieves roughly $T^{2/3}$ regret with polylogarithmic memory, providing a strong separation for any $W = T^{2/3 + \Omega(1)}$.  

Finally, one might ask whether our bounds can be further improved beyond the worst-case, e.g., through distributional assumptions on the performance of the arms. 
For example, there is a large body of work studying the random-order model in the streaming literature~\cite{guha2009stream}. 
In the context of the online learning with experts problem, \cite{SrinivasWXZ22} observed that the random-order model corresponds to the view that any permutation over the loss vectors is equally likely in the input distribution. 
In particular, \cite{SrinivasWXZ22} observes that an \emph{exchangeability} property in terms of the losses on each day allows the random order model to subsume the i.i.d. model where each loss on each arm follows a fixed underlying distribution for all times  (stochastic experts). 
Therefore, any algorithmic upper bounds established in the random-order model also translate to the i.i.d. model. 

We show that even under a weaker distributional assumption, the optimal regret can be achieved simultaneously using both single-query bandit feedback and polylogarithmic space. 
\begin{restatable}{theorem}{onequerybestrandom}
\label{thm:one-query-best-random} 
There exists an online learning algorithm that given any instance of $n$ experts and $T$ days such that $T\geq n$ and the query access of a single expert, i.e., the bandit setting, where the loss sequence of the best expert is in random order, achieves $\sqrt{n T}\cdot \polylog(nT))$ regret using $\polylog(nT)$ words of memory with probability at least $1-1/\poly(nT)$. 
\end{restatable}
We emphasize that \Cref{thm:one-query-best-random} only requires that the best expert is in random order, which is a weaker requirement than the full random-order model. 
By comparison, \cite{SrinivasWXZ22} initially provided an algorithm that achieves expected regret $R$ using $\tO{\frac{nT}{R}}$ space for the online learning with experts problem in the random-order model, though these results were subsequently improved~\cite{PengZ23,pr23}, ultimately to an algorithm that achieves $\tO{\sqrt{nT}}$ regret using just polylogarithmic space in the arbitrary-order model by \cite{pr23}. 
We remark that this line of work uses polylogarithmic queries at all times whereas \Cref{thm:one-query-best-random} uses just single-query bandit feedback. 

\section{Technical Overview}
\label{sec:tech-overview}
In this section, we provide a technical overview of our algorithms and analysis. 
We first discuss other natural approaches, why they fail, and the implications. 

\subsection{Shortcomings of Natural Approaches}
Although there is a wide range of techniques for the streaming model, these generally do not translate to the sliding window model. 
For example, as previously discussed, linear sketches do not have an immediate way to handle data that is implicitly expired by additional updates. 
Thus we focus our discussion on the most relevant sliding window techniques. 

\paragraph{Histogram frameworks do not work.}
The most common approach for sliding window algorithms are histogram-based approaches, such as the exponential histogram~\cite{DatarGIM02}, the smooth histogram~\cite{BravermanO07}, and their adaptations~\cite{BravermanGLWZ18,BravermanLLM15,BravermanLLM16,ChenNZ16,EpastoLVZ17,BorassiELVZ20,BravermanWZ21,EpastoMMZ22a,BlockiLMZ23}. 
These frameworks convert an insertion-only streaming algorithm $\calA$ for a problem into a sliding window algorithm by running multiple instances of $\calA$ in parallel, starting at different times throughout the stream, called checkpoints. 
New checkpoints are created with each stream update, while existing checkpoints are removed when the values output by their corresponding algorithms $\calA$ are too ``close'' to each other, e.g., multiplicatively within a factor of $2$.  
As a result, if a function is well-behaved and bounded, it can be shown that at any point in time during the course of the data stream, there are only a logarithmic number of checkpoints. 
Importantly, there exist two checkpoints $t_1$ and $t_2$ that ``sandwich'' the beginning of the sliding window, so that intuitively, the value of the function on the data stream starting at $t_1$ and the value of the function on the data stream starting at $t_2$ sandwich the value of the function on the sliding window. 
Hence, it suffices to output the value of the algorithm $\calA$ that starts at time $t_2$. 

Unfortunately, there are multiple issues with adapting this approach for the purposes of the online learning with experts problem. 
First, these histogram approaches necessarily blow up the query complexity.  
By maintaining $\O{\log n}$ instances of a single-query bandit algorithm in parallel, the query complexity instead becomes $\O{\log n}$, which is prohibitively large for our goal. 
More problematically, these histogram approaches generally use $\O{\log n}$ instances of a streaming algorithm to achieve a constant-factor approximation to the overall loss. 
However, a constant-factor approximation to the loss can translate to a \emph{linear} regret on a window of size $W$, rather than the target goal of $\sqrt{W}$ regret. 
This can be fixed by maintaining $\sqrt{W}$ instances of a streaming algorithm and creating checkpoints whenever the loss changes by an additive $\sqrt{W}$ amount rather than a multiplicative amount, but then the resulting space becomes polynomial in $W$ rather than polylogarithmic in $W$. 

\paragraph{Online coresets and importance sampling do not work.}
A similar idea converts a notion of online algorithms to sliding window algorithms. 
Central to this framework is the concept of an online coreset, which is a representative subset of the data stream that preserves the relevant properties with respect to the order of the arriving elements. 
\cite{BravermanDMMUWZ20} showed that an online coreset for a problem can be used to achieve a sliding window algorithm for that problem, an idea that has been subsequently used to achieve sliding window algorithms for numerical linear algebra~\cite{BravermanDMMUWZ20,WoodruffY23} and clustering~\cite{WoodruffZZ23,Cohen-AddadJYZZ25}. 
Another common and related approach is based on sampling elements of the data stream based on how ``important'' the element appears to be. 
In particular, the notion of importance is sensitive to both uniqueness, i.e., how different an element appears from the remaining dataset, and recency, i.e., whether a stream update has appeared more recently in the data stream. 

However, both of these approaches require a construction of a coreset for the underlying problem. 
Unlike numerical linear algebra and clustering, no such coreset construction is known for online learning with experts, and it is far from clear why a small-space coreset preserving the necessary regret guarantees should exist. 
This fundamental limitation arises from the adversarial nature of expert learning, where the diversity and adaptability of expert predictions make it difficult to maintain a representative subset that accurately captures the overall loss.

\paragraph{Adaptations of interval regret algorithms.}
Another natural approach would be to adapt existing offline interval regret algorithms to use sublinear space. 
For example, \cite{DanielyGS15} gave an algorithm with $\Otilde(\sqrt{nT})$ regret using $\O{\log T}$ queries per day while \cite{Lu0CZWH24} gave an algorithm with a similar $\Otilde(\sqrt{nT})$ regret, but using only \emph{two} queries each day. 
However, both of these algorithms use $\Omega(n)$ memory to track the losses of experts across multiple time intervals, which is crucially used in the analysis to determine the probabilities of choosing an expert at each time. 
Without linear memory, it may be possible to grossly miscalculate the loss of an expert, which could result in an undesirable probability distribution and therefore, high regret. 
It is not clear at all how to adapt these approaches to sublinear memory and indeed even in other ``easier'' cases such as the online learning with experts problem with full information, i.e., $n$ queries per day, achieving near-optimal regret using sublinear memory is a significant challenge~\cite{XuZ21,SrinivasWXZ22,PengZ23,pr23,WoodruffZZ23a}. 

\paragraph{Adaptation of memory-bounded algorithms.}
One could instead start from existing streaming algorithms and adapt them to the sliding window model. 
For example, the algorithm of \cite{XuZ21} achieves $\O{L n^{1/L}}$ memory and $\O{L\cdot  T^{1-1/(2L)} n^{1/(2L)}}$ regret, for any positive integer trade-off parameter $L \geq 1$. 
Although the algorithm works in the single-query bandit setting, the algorithm partitions the $n$ experts into blocks of experts, e.g., $\sqrt{n}$ blocks of $\sqrt{n}$ experts for $L=2$. 
Each block of $\sqrt{n}$ experts forms a meta-expert and there is a hierarchical structure that plays a standard single-query bandit algorithm on the meta-experts. 
This inherently requires tracking (estimated) losses for all of the meta-experts, resulting in $\O{\sqrt{n}}$ memory for $L=2$. 
Similar limitations occur for other settings of $L$. 
A natural idea is to recurse on the hierarchical structure to have more levels, but then it is not immediate how to produce the outputs of the meta-experts in intermediate levels. 

Another line of recent work studies online learning with experts with bounded memory but full information~\cite{SrinivasWXZ22,PengZ23,pr23,WoodruffZZ23a}. 
These works maintain a hierarchical structure on meta-experts and crucially utilize the full information setting to simulate outputs of the meta-experts. 
Since the hierarchical structure has $\O{\log n}$ levels, it is not clear how to simulate the meta-experts using $\O{1}$ queries per time. 
Additionally, the hierarchical structure necessarily achieves poor interval regret, because the highest levels of the structure requires maintaining experts that are performing well over the entire time horizon and the algorithm is compelled to play those experts even when they perform poorly over a smaller interval. 

\subsection{The Sliding Window Algorithm}
Our algorithm follows the natural approach of sampling a \emph{pool} of experts tracked by the algorithm and pruning poorly performing experts within the pool~\cite{SrinivasWXZ22,PengZ23,pr23,WoodruffZZ23a}. 
Consider an algorithm that breaks up the time horizon $T$ into \emph{epochs} of some length $B$. 
Inside each epoch, the algorithm proceeds by running the multiplicative weights update (MWU) over the experts in the pool. 
At the end of an epoch, the algorithm samples some additional experts to the pool with a rate of $\frac{1}{n}$.
To ensure the overall pool size is small, the algorithm will prune the pool by removing experts that are \emph{no better than} a benchmark loss. 
Here, the benchmark loss $\calL^{\calD}(\benchmark)$ over a duration $\calD$ for expert $i$ is defined as the epoch-wise optimal loss among the arms in the pool except $i$. 
If any arm in the pool with loss $\calL^{\calD}(i)$ over duration $\calD$ performs not strictly better than the benchmark, i.e.,
\begin{equation}
\label{equ:cover-original-overview}
    \calL^{\calD}(i) \geq \calL^{\calD}(\benchmark) - 1,
\end{equation}
then arm $i$ will be evicted by the algorithm. 
In this way, the arms that stay in the pool must have exponentially smaller losses, so that the size of the pool is bounded by $\O{\log T}$ since the loss has to be in $[0, T]$. 
Furthermore, tracking the loss of $\O{\log T}$ arms over any duration can be done in a total of $\polylog(nT)$ space.

Since the algorithm runs MWU over the experts in the pool, the main concern for the regret analysis is how the algorithm could behave when the best expert $\istar$ is \emph{not} in the pool. 
To that end, it is natural to categorize each epoch as follows. 
If $\istar$ is sampled in an epoch and it would be subsequently evicted due to good performances of other arms in the pool, then the epoch is called a \emph{good epoch}. 
Otherwise, the epoch is called a \emph{bad epoch}. 
Intuitively, our algorithm must have $\sqrt{B}$ regret on a good epoch of length $B$ due to the guarantees of MWU. 
On the other hand, there is no control on the regret for the bad epochs, but it can be shown that there are at most $n\cdot\polylog(n)$ bad epochs because afterwards, the best arm $\istar$ will be in the pool. 
Hence, we can informally upper bound the regret by $((T/B)\cdot \sqrt{B} + nB)\cdot \polylog(nT)$, which can be optimized to roughly $T^{2/3}$ by the appropriate choice of $B$. 
We remark there are some subtleties with this argument, since the eviction policy is defined with respect to experts already sampled into the pool, so that there are dependencies in the analysis, but these issues can be overcome with some additional technical work of sampling independent experts to handle the eviction at the cost of pool size $\polylog(nT)$. 

However, the main bottleneck for this algorithm is due to playing MWU on the experts in the pool of size $\polylog(nT)$, which requires a prohibitively large number of queries at each time. 
An immediate thought would be to replace MWU with EXP3, which is an algorithm that achieves $\sqrt{T}$ regret (ignoring $\poly(n)$ terms) with the one-query bandit signals~\cite{AuerCFS02}. 
As such, the algorithm would still achieve roughly $\sqrt{B}$ regret for each epoch of length $B$. 
This results in a single query algorithm between the epochs. 
However, we would not be able to obtain the loss $\calL^{\calD}(i)$ for a fixed arm $i$ in the pool over a duration $\calD$. 
We could use a second query to uniformly sample the arms in the pool to produce estimates $\widetilde{\calL^{\calD}}(i)$ of their losses with $\sqrt{\card{\calD}}$ error. 
We could then use the estimated losses of arms in the pool to replace their actual losses for eviction rules. 
This decreases the query complexity to two experts per time, but the algorithm still achieves regret $T^{2/3}$. 

\paragraph{Boosting with additional queries.}
The main reason for our algorithm $\calA_1$ achieving $T^{2/3}$ regret is due to incurring $B$ regret per bad epoch. 
The natural approach would be to apply a standard boosting strategy by running another instance $\calA_2$ of our algorithm with time horizon $B$ though, and this algorithm would achieve $B^{2/3}$ regret on the bad epoch. 
We could then play an ``outer'' EXP3 on $\calA_1$ and $\calA_2$ to get overall regret $B^{2/3}$ on the bad epoch. 
Then the overall regret becomes $\frac{T}{B}\sqrt{B}+nB^{2/3}$, which can be optimized to $T^{4/7}$ omitting $n$ factors. 
Ideally, we could recursively apply this boosting procedure and obtain roughly $\sqrt{nT}$ regret after $\log(nT)$ layers. 

Unfortunately, this approach has a fatal flaw. 
In the partial information setting, the updates of both the baseline algorithms and the outer EXP3 are random. 
However, the randomness used are \emph{not} independent and interferes with each other.
In particular, the days we play each baseline algorithm are a function of the past loss sequences for each baseline algorithm, and the decision to play an algorithm will affect the future performance of the baseline algorithms (see \Cref{fig:boosting-dependency-issue} for an illustration of this issue).
This issue is specific for the partial information setting in this paper. 
In the full information setting, all baseline algorithms receive updates after each day, regardless of whether the outer MWU chooses it, and so we can independently analyze each baseline algorithm. 
This is not the case in the partial information setting, and overcoming this dependency issue is a major algorithmic challenge.

\paragraph{Removing dependencies for two-query algorithm.}
Our main idea to handle the dependency issue is to distill both the inner and outer EXP3 updates based on steps for ``exploration'' and ``exploitation''. 
We first consider the inner EXP3 subroutines, which form the baseline algorithms. 
For exploration, we estimate losses non-adaptively by sampling an expert uniformly at random on each day using the second query and using its feedback to form the estimate. 
For exploitation, we sample each expert with probability proportional to $\exp(-\widetilde{\calL}(i))$ (the estimated loss of arm $i$), but we do \emph{not} update the weights. 
This approach can also be generalized to the updates of the outer EXP3 algorithm, which plays on the baseline algorithms. 
Namely, on each day, with probability $\frac{1}{2}$, we sample a baseline algorithm to estimate its cost using the second query. 
The baseline algorithm is similarly sampled with a probability proportional to the exponential of the estimated losses. 
Asymptotically, the variances of the loss estimation for both the experts and the algorithms over a duration $\calD$ are still $\sqrt{\card{\calD}}$ (ignoring $n\polylog(nT)$ factors), allowing the good epoch to have at most $\sqrt{B}$ regret with epoch length $B$. 
The important thing here is that once we fix the randomness of the second query, the updates of the baseline algorithms become deterministic. 
Therefore, we can conduct the same recursive boosting, and obtain the near-optimal $\sqrt{T}$ regret after $\log(nT)$ layers. 
Finally, we can implement a more technical boosting procedure to further optimize the dependencies on $n$ in the regret.

An illustration of the ideas we discussed for the dependency issue can be found in \Cref{fig:boosting-dependency-issues-and-fixes}.

\begin{figure}[!ht]
\centering
\begin{subfigure}{0.46\textwidth}
  \includegraphics[scale=0.12]{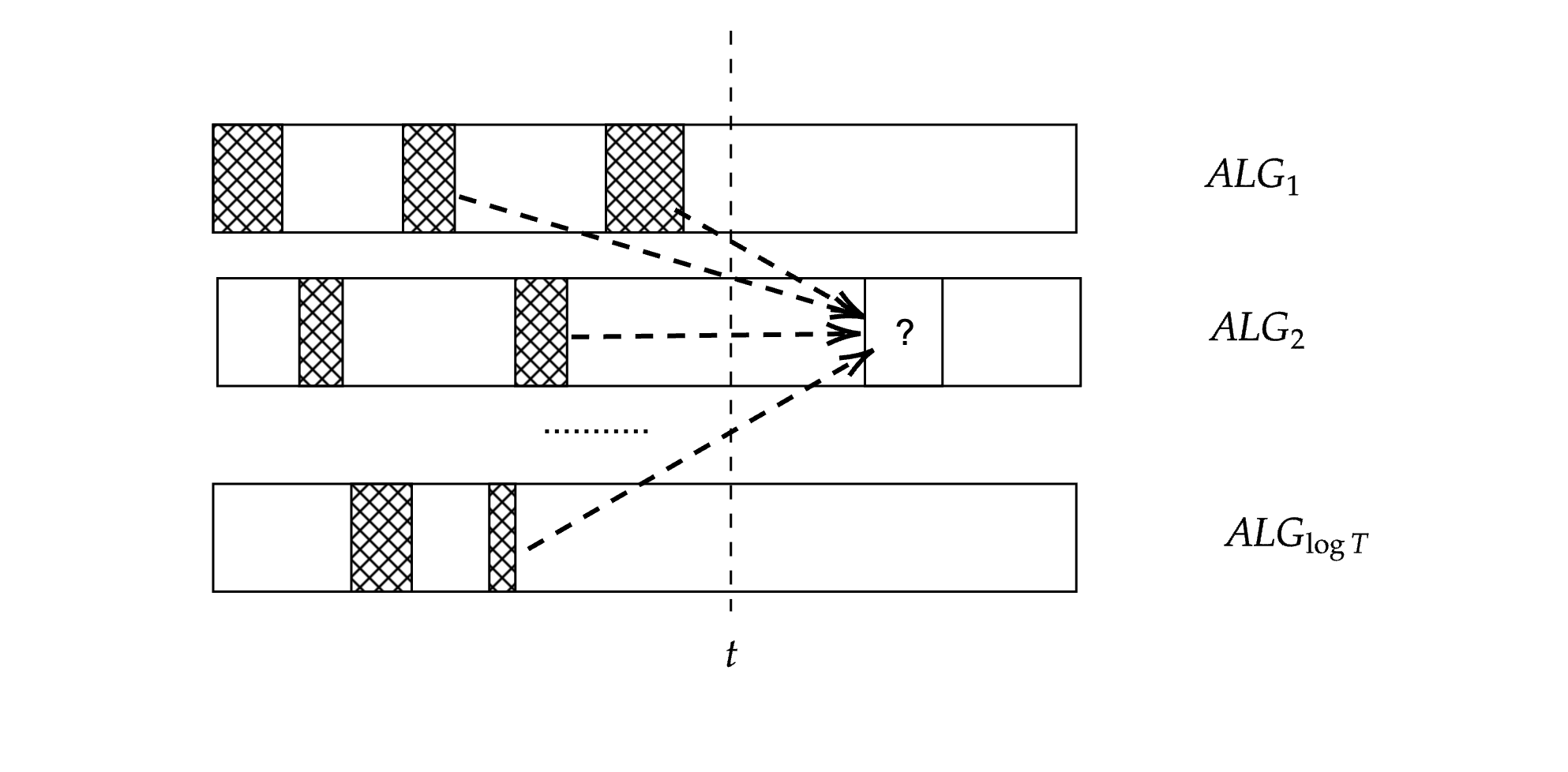}
  \caption{\centering The randomness used in the past days and other baseline algorithms interferes with the loss of a future day.}
  \label{fig:boosting-dependency-issue}
\end{subfigure}
\begin{subfigure}{0.46\textwidth}
  \includegraphics[scale=0.12]{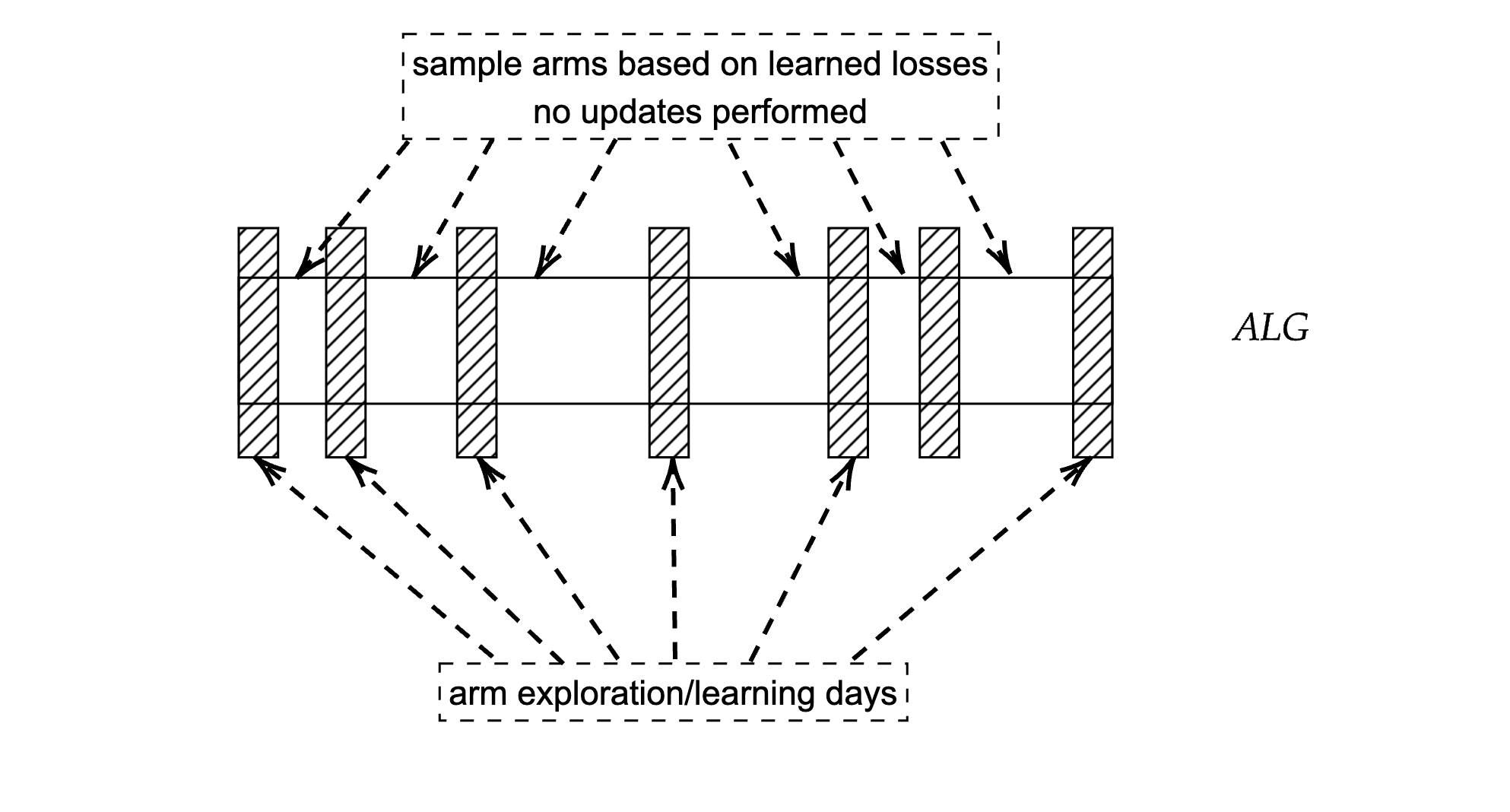}
  \caption{\centering Separating the learning and exploitation days for a baseline algorithm. The EXP3 algorithm is only updated based on loss estimations.}
  \label{fig:arm-learning-exploit-separate}
\end{subfigure}
\begin{subfigure}{0.46\textwidth}
  \includegraphics[scale=0.12]{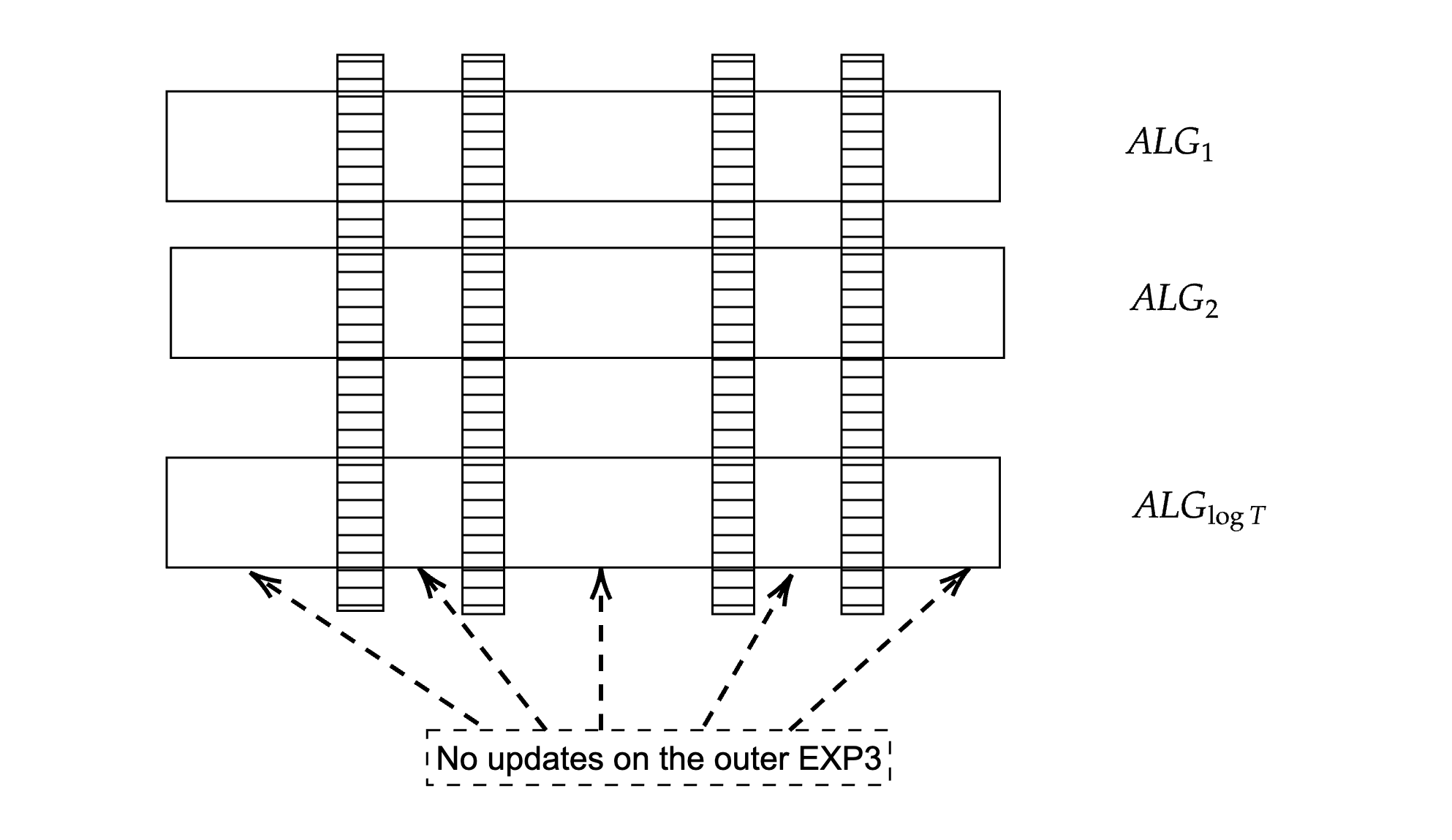}
  \caption{\centering Separating the learning and exploitation days for the outer EXP3. The outer EXP3 algorithm is only updated based on loss estimations.}
  \label{fig:algorithm-learning-exploit-separate}
\end{subfigure}
\begin{subfigure}{0.46\textwidth}
  \includegraphics[scale=0.12]{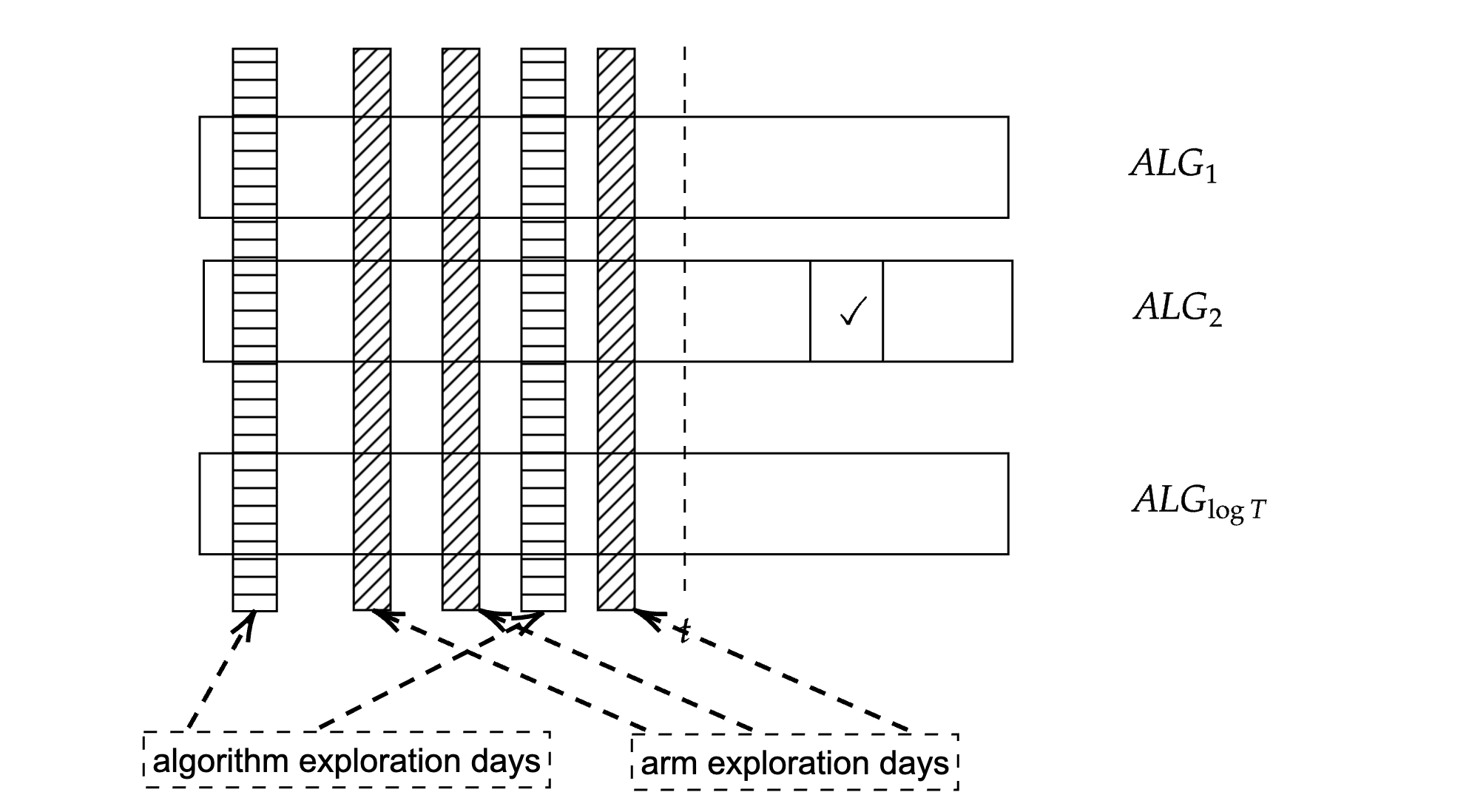}
  \caption{\centering By fixing the randomness for the second query (or exploration days), the next day we run each baseline is fixed, which allows regret control.}
  \label{fig:dependency-issue-fixed-boosting}
\end{subfigure}
\caption{An illustration of the dependency issue and the ideas to overcome the problem.}
\label{fig:boosting-dependency-issues-and-fixes}
\end{figure}

\paragraph{From two-query framework to the sliding-window algorithm.} 
We now discuss how the approach in \cite{Lu0CZWH24} could be simulated with low memory, and how the sliding-window algorithm is implemented such that the entire algorithm only takes two queries each day. 
At a high level, the algorithm in \cite{Lu0CZWH24} contains two components: the interval algorithms, which are EXP3 algorithms tailored to intervals $I$ with length $\card{I}=2^s$ for $s\in \log{T}$, and the ``outer algorithm'', which controls the weights to play each each interval algorithm. 
Intuitively, these two layers interact in a manner similar to the previously discussed recursive boosting: at each day, the outer algorithm samples an interval algorithm and plays an arm following its distribution. 
Additionally, the updates of the weights on each interval algorithm are conducted using a second query that samples from a distribution with a much smaller variance, e.g., a uniform distribution over the experts. 
This ensures the variance of one interval algorithm will not significantly affect other algorithms. 
Finally, the outer algorithm adapts the weights on each interval algorithm following the standard approach used in interval regrets by \cite{DanielyGS15}.

The key observation here is that the algorithmic framework for the sliding-window algorithm relies on the low-variance estimations of the loss sequences, and our previous framework is particularly suitable for such this property. 
Therefore, we are able to prove in a white-box manner that, since the second query is sufficient to estimate the losses with low variance, each interval algorithm still satisfies the properties we obtained in the boosting algorithm.
As a result, each of the interval algorithms achieve $\sqrt{\card{I}}$ regret on interval $I$ with $\polylog(nT)$ space, which also gives the regret bound of $\sqrt{W}$ for the sliding window model. 
Finally, since there are at most $\O{\log T}$ interval algorithms and each of them uses $\polylog(nT)$ bits of memory, the final memory bound is again $\polylog(nT)$.

\subsection{The Streaming Algorithm in the Single-Query Setting}
We now describe our framework for the single-query bandit setting, in which we obtain a streaming algorithm with $T^{2/3}$ regret. 
The algorithm follows the same framework of sampling and pruning of the pool of experts. 
However, a significant challenge here is that for an expert that has been in the pool over a duration $\calD$, we no longer have the second query to provide estimations with $\sqrt{\card{\calD}}$ error. 
As such, we need to assign exploration days that might incur regrets to estimate the loss sequences. 
Let $\gamma\in (0,1)$ be the exploration factor in EXP3, for a duration of $\calD$, so that the additive error for the loss estimation is bounded by $\O{\sqrt{\card{\calD}/\gamma}}$. 
We can incorporate such an additive error into the eviction rule, so that the good epochs of length $B$ would have regret $\O{\sqrt{B/\gamma}}$. 
On the other hand, the single-query setting pays $\gamma T$ additive regret for exploring with probability $\gamma$ across all $T$ days. 
For epochs with length $B$, the overall regret follows the form of $(\frac{T}{B}\cdot \sqrt{B/\gamma}+nB+\gamma T)\cdot \polylog(nT)$. 
By balancing $\gamma = {1}/{B^{1/3}}$ and $B = T^{3/4}$, we obtain roughly $T^{3/4}$ regret. 

\paragraph{Boosting with a better balancing.} 
To further improve the performances of the one-query algorithm, a natural approach would be to simulate the recursive boosting procedure that we developed for the two-query algorithm. 
However, we no longer have the ``free'' second query to estimate the loss sequences of the experts and the baseline algorithms, and we would need to ``sacrifice'' some exploration days and incur additional regret for those days. 
Concretely, we again enter an exploration day on each day with some probability $\gamma$. 
On the exploration days, we follow the same strategy as in the two-query algorithm to maintain estimated losses for the baseline algorithms and all experts maintained by them, i.e., with probability $\frac{1}{2}$, sample an expert uniformly at random and with probability $\frac{1}{2}$, sample a baseline algorithm uniformly at random. 
On the days that are not for exploration, we perform exploitation by sampling each baseline and expert using the estimated losses, but we do not update the weights. 

In this manner, we can maintain estimations of losses for both the experts and the baseline algorithms with an additive error at most $\sqrt{D/\gamma}$, if we explore with probability $\gamma$ over a duration $\calD$ of length $D$. 
From a better balancing, it follows that if we explore with a rate of $\gamma = 1/T^{1/3}$, then we would have additive error $\sqrt{D}\cdot T^{1/6}\leq T^{2/3}$ with high probability. 
Moreover, the regret incurred by the exploration steps is $\gamma D\leq T^{2/3}$. 
Hence, we can apply the recursive boosting procedure $\log(nT)$ times and obtain a final regret bound of $T^{2/3}$ with $\polylog(nT)$ memory, so that both the regret and the memory significantly improve upon the previous state-of-the-art for the single-query bandit setting.

\paragraph{Random-order best expert.} 
To further improve the regret to $\sqrt{T}$ when the best expert is in random order, we design a different algorithm using first principles, which does not follow the previous sampling and eviction-based framework. 
The subroutines are notably much simpler. 

At a high level, the algorithm starts with a binary search of the ``correct'' error rate of the best expert. 
Let $\gamma \sqrt{T}$ be the loss of the best expert, i.e., a loss rate of $\gamma/\sqrt{T}$, and let $C\sqrt{T}$ be the current target loss that we are aiming for in the binary search, where $1\leq C\leq \sqrt{T}$. 
To build intuition, let us first consider an idealized case where all other experts are much worse than the best expert in \emph{any interval}. 
Therefore, we could simply cycle through the experts and check whether each expert has less than $\O{C/\sqrt{T}}$ error rate. 
If the error rate is too high, we discard the expert; and if there is no expert that satisfies the desired error rate, we could increase $C$ and check again. 
On the other hand, if we find an expert with a low error rate, we could then commit to the expert for the rest of the days. 
This process only takes $\O{\log T}$ bits of space.

The analysis gets more complicated in the non-idealized case, specifically when the sub-optimal experts could have a ``hot streak'' of correct days. 
Here, a sub-optimal expert could have an error rate of less than $C/\sqrt{T}$ for a short period of time before it starts incurring much higher losses.
In this case, simply committing to the expert could incur high regret in the later days. 
This issue could be handled by \emph{dynamically checking} whether the error rate of the expert in hand is still satisfactory. 
If an expert demonstrates a much better error rate than the best expert for a short period before a much worse error, the algorithm could evict the expert and continue the process.
The regret is still low due to the fact that there are many days the expert we checked outperforms the best expert, i.e., we can account for the ``reverse regret''. 
The process stops after we check the best expert with the ``correct'' rate: since the loss sequence of the expert is in random order, we will never evict the best expert.
An amortized regret analysis then gives a regret bound of $\sqrt{T}$ for this algorithm.

\section{Preliminaries}
\label{sec:prelim}

\paragraph{Notation.} 
For a positive integer $n>0$, we use the notation $[n]$ to denote the set $\{1,2,\ldots,n\}$. 
We use the notation $\poly(n)$ to denote a fixed polynomial in $n$, whose degree can be inferred from setting appropriate constants. 
In particular, we use the notation $\polylog(n)$ to denote $\poly(\log n)$. 
We say an random event has high probability if its failure probability is $1-\frac{1}{\poly(nT)}$. 

We use $\ell^t(i)$ to denote the loss of arm $i$ at day $t$. Consequently, $\ell^{t}$ is the $n$-dimensional loss vector of day $t$. 
For the cumulative loss of arm $i$ in a duration $\calD$, we use $\calL^{\calD}(i)$.
We also write interval $\calI$ and the loss of expert $i$ over the interval $\calL^{\calI}(i)$; both notations are used depending on the context.
In general, we use $\ell$ to denote the loss of a single day and $\calL$ to denote the cumulative loss across multiple days. 
Furthermore, we sometimes write $\calL^{\calD}(\ALG)$ as a function to denote the cost of an algorithm $\ALG$ over the duration $\calD$.

\subsection{The Formal Description of the Model}
\label{subsec:model-description}
We give a more formal version of the problem descriptions. We start with the online learning problem with general worst-case loss sequences.

\paragraph{Online learning with expert using $q$ queries and $s$ space.} We consider an online learning problem over $T$ days with a set of $n$ experts. The sequence of outcomes is determined by an \emph{oblivious adversary}: for each day $t \in [T]$, the loss $\ell^t(i) \in \{0, 1\}$ for each expert $i \in [n]$ is fixed in advance, unknown to the algorithm.
An online learning algorithm $\ALG$ interacts sequentially with the instance on each day $t\in [T]$ as follows.
\begin{enumerate}
\item \textbf{Algorithm decision}: The online learning algorithm, $\ALG$, chooses one expert $i_t \in [n]$ to follow, i.e., outputs the same prediction as the expert.
\item \textbf{Incur losses}: The algorithm incurs the loss $\ell^t(i_t)$ of the chosen expert.
\item \textbf{Feedback signals}: After making the choice, the algorithm observes the loss of the chosen expert and an additional set of $(q-1)$ experts. $\ALG$ is allowed to observe the losses of the additional $(q-1)$ experts without any cost.
\end{enumerate}
The goal is to design an algorithm that minimizes for following objectives:
\begin{enumerate}[label=(\alph*).]
\item \textbf{The cumulative regret,} which is the difference between the losses of $\ALG$ and the loss of the best single expert in hindsight over the $T$ days, i.e.,
$$\sum_{t=1}^{T} \ell^t(i_t) - \min_{i \in [n]} \sum_{t=1}^{T} \ell^t(i).$$
\item \textbf{The sliding-window regret.} On each day $t$, we look at the difference between the losses in interval $[t-W+1, t]$, and we count the gap between the losses of $\ALG$ and the best arm in hindsight of $[t-W+1, t]$. In other words, the regret in the sliding window of time step $t$ is defined by 
$$\sum_{s=t-W+1}^{t} \ell^s(i_s) - \min_{i \in [n]} \sum_{s=t-W+1}^{t} \ell^s(i).$$
We define the regret in the sliding-window model as the \emph{maximum} regret of the sliding window for any $t\in [T]$.
\item \textbf{The interval regret.} Similar to the sliding-window case, for an interval $\calI=[t_1, t_2]\subseteq[T]$, the interval regret is defined as  
$$\sum_{t=t_1}^{t_2} \ell^t(i_t) - \min_{i \in [n]} \sum_{t=t_1}^{t_2} \ell^t(i).$$
Compared to the sliding-window regret, the interval regret does not require any fixed size. We define the regret in the interval regret model as the \emph{maximum} regret of any interval $\calI=[t_1, t_2]\subseteq[T]$.
\end{enumerate}

We also aim to optimize the space complexity $s$, defined as the maximum number of words that $\ALG$ utilizes at any point during its execution.

\paragraph{Online learning with random-order best expert loss sequence.} For the case of cumulative regret minimization with $n$ experts and $T$ days in online learning, we define the case of random-order best expert as follows. We let $\istar$ be the expert that attains the minimum cost, i.e., $$\istar=\argmin_{i\in [n]} \sum_{t=1}^{T} \ell^t(i).$$
We then apply a random permutation on $[T]$ for the losses of expert $\istar$. In other words, let $\sigma: 2^{[T]} \rightarrow 2^{[T]}$ be a uniform random permutation over $T$, we let $\ell^{t}(\istar)\gets \ell^{\sigma(t)}(\istar)$ before the starts of online learning.

\subsection{Concentration Inequalities}
\label{subsec:concentration-inequ}
We give the standard concentration inequalities we used in this paper.
\begin{proposition}[Bernstein's inequality, \cite{bernstein1924modification}]
\label{thm:bernstein}
Let $X_1, X_2, \ldots, X_n$ denote independent random variables such that $\left\lvert X_i - \Ex{X_i}\right\rvert|\le M$ for all $i\in[n]$. 
Let $S_n = \sum_{i=1}^n (X_i - \Ex{X_i})$ and let $\sigma^2 = \sum_{i=1}^n\text{Var}(X_i)$. 
Then, for any $t>0$:
\[\PPr{|S_n| \ge t}\le 2 \exp\left( -\frac{t^2}{2\sigma^2 + \frac{2}{3}Mt} \right).\]
\end{proposition}

\subsection{The EXP3 algorithm and learning with exploration}
\label{subsec:EXP3-and-variates}
In the setting with the single-arm bandit signal, the standard algorithm to achieve the optimal regret is the Exponential-weight algorithm for Exploration and Exploitation algorithm, abbreviated as the EXP3 algorithm~\cite{AuerCFS02}. 
The algorithm description is as \Cref{alg:EXP3}.

\begin{algorithm}[!htb]
\caption{The Exponential-weight algorithm for Exploration and Exploitation (EXP3) Algorithm}
\label{alg:EXP3}
\algorithmicrequire{A set of $n$ arms; a parameter of $T$ days.}\\
\algorithmicrequire{An exploration rate of $\gamma\in (0,1]$.}
\begin{algorithmic}[1]
\State{Maintain $w_t(1), w_t(2), \cdots, w_t(n)$ as the weights for $n$ arms.} 
\State{Initialize with $w_1(i)\gets 1$ for all $i\in [n]$.}
\For{each day $t$}
\State{Sampling from distribution $P_t$ such that for each $i\in [n]$.}
\begin{align*}
P_t(i) = (1-\gamma)\cdot \frac{w_t(i)}{\sum_{i=1}^n w_t(i)} + \frac{\gamma}{n}.
\end{align*}
\State{Let $i_t \sim P_t$ be the sampled arm, observe the loss $\ell^t(i_t)$.}
\For{each arm $i\in [n]$}
\State{Let the estimation of loss be as follows.}
\begin{align*}
\widetilde{\ell^t}(i) =
    \begin{cases}
         \ell^t(i)/P_t(i), \quad & \text{if $i=i_t$}\\
        0, \quad & \text{otherwise}
    \end{cases}
\end{align*}
\State{Update $w_{t+1}(i) = \exp\paren{-\gamma \cdot \widetilde{\ell^t}(i)}\cdot w_{t}(i)$.}
\EndFor
\EndFor
\end{algorithmic}
\end{algorithm}

The standard guarantees for the EXP3 algorithm are as follows.
\begin{proposition}
\cite{AuerCFS02}
    \label{prop:exp3-guarantee}
    Let $\gamma\in (0, \frac{1}{\sqrt{T}})$, the expected regret of \Cref{alg:EXP3} is at most $\O{\sqrt{nT\log{n}}}$.
\end{proposition}

\paragraph{The ``learning with exploration'' view of EXP3.} 
The vanilla version of \Cref{alg:EXP3} explicitly maintains weights $w_t(i)$ for each arm. An alternative view of the EXP3 algorithm is to uniformly at random sample with probability $\gamma \in (0,1]$, and sample the arm proportional to the loss otherwise. 
Importantly, the ``learning'' of algorithm, i.e., the update of the losses, is only conducted on the exploration days. 
The algorithm could be described as \Cref{alg:EXP3-exploration}.

\begin{algorithm}[!htb]
\caption{The ``learning as exploration'' version of EXP3}
\label{alg:EXP3-exploration}
\algorithmicrequire{A set of $n$ arms; a parameter of $T$ days.}\\
\algorithmicrequire{An exploration rate of $\gamma\in (0,1/2]$.}
\begin{algorithmic}[1]
\State{Maintain the estimated losses $\widetilde{\calL^{1:t}}(i)$ for each $i\in [n]$ and $t\in [T]$.}
\For{each day $t$}
\State{With probability $\gamma$, enter an \emph{exploration day}.}
\If{$t$ is an exploration day}
\State{Sample $i_t$ uniformly at random from the set of arms.}
\begin{align*}
\widetilde{\ell^t}(i) =
    \begin{cases}
         n\cdot \ell^t(i)/\gamma, \quad & \text{if $i=i_t$}\\
        0, \quad & \text{otherwise}
    \end{cases}
\end{align*}
\State{Update $\widetilde{\calL^{1:t+1}}(i_t) \gets \widetilde{\calL^{1:t}}(i_t)+\widetilde{\ell^t}(i_t)$.}
\Else
\State{Sample an arm $i_t$ from the following distribution.}
\begin{equation}
\label{equ:sample-prob}
P_t(i) = \frac{\exp(-\gamma \cdot \widetilde{\calL^{1:t}}(i))}{\sum_{i=1}^n \exp(-\gamma \cdot \widetilde{\calL^{1:t}}(i))}.
\end{equation}
\EndIf
\EndFor
\end{algorithmic}
\end{algorithm}

The regret guarantee we have for \Cref{alg:EXP3-exploration} is as follows. 
\begin{proposition}
    \label{prop:exp3-exploration}
    Let $\gamma\in (0, \frac{1}{\sqrt{T}})$, the expected regret of \Cref{alg:EXP3-exploration} is at most $(n\cdot \sqrt{T/\gamma} + \gamma\cdot T)\cdot \polylog(nT)$.
\end{proposition}

We believe \Cref{prop:exp3-exploration} is known in the literature; regardless, we provide the proof in \Cref{app:omitted-proof-prelim} for completeness. 
Furthermore, \Cref{prop:exp3-exploration} implies the following lemma.
\begin{corollary}
    \label{lem:EXP3-with-approx-loss}
    Let $\{\ell^{t}\}_{t=1}^{T}$ be the set of loss vectors, and let $\widetilde{\calL^{\calD}}$ be the estimation of the losses over a duration $\calD$. Suppose that
    \[\card{\widetilde{\calL^{\calD}}(i)-\sum_{t\in\calD}\ell^{t}(i)}\leq \sqrt{\card{\calD}/\gamma}\cdot \polylog(nT).\]
    Then, sampling with \Cref{equ:sample-prob} at every step in $D$ leads to regret $\sqrt{\card{\calD}/\gamma}\cdot n\polylog(nT)$.
\end{corollary}

\paragraph{Using an additional query.} In the multi-signal query model, where we are allowed to observe the loss of more than one arm, we could obtain similar (yet stronger) guarantees than \Cref{alg:EXP3-exploration}. In particular, we could perform exploration every day, resulting in the \Cref{alg:EXP3-exploration-two-query}.

\begin{algorithm}[!htb]
\caption{EXP3 with two queries}
\label{alg:EXP3-exploration-two-query}
\algorithmicrequire{A set of $n$ arms (experts); a parameter of $T$ days.}\\
\algorithmicrequire{An exploration rate of $\gamma\in (0,1]$; a learning rate of $\eta$.}
\begin{algorithmic}[1]
\State{Maintain the estimated losses $\widetilde{\calL^{1:t}}(i)$ for each $i\in [n]$ and $t\in [T]$.}
\For{each day $t$}
\State{Sample an arm $i_t$ from the following distribution and play $i_t$.}
\begin{align*}
P_t(i) = \frac{\exp(-\eta\cdot \widetilde{\calL^{1:t}}(i))}{\sum_{i=1}^n \exp(-\eta\cdot \widetilde{\calL^{1:t}}(i))}.
\end{align*}
\State{Sample $i^{\text{est}}_t$ uniformly at random from the set of arms, but do \emph{not} play $i^{\text{est}}_t$.}
\NoNumber{\Comment{No loss incurred since we do not play $i^{\text{est}}_t$}}
\begin{align*}
\widetilde{\ell^t}(i) =
    \begin{cases}
         n\cdot \ell^t(i), \quad & \text{if $i=i^{\text{est}}_t$}\\
        0, \quad & \text{otherwise}
    \end{cases}
\end{align*}
\State{Update $\widetilde{\calL^{1:t+1}}(i^{\text{est}}_t) \gets \widetilde{\calL^{1:t}}(i^{\text{est}}_t)+\widetilde{\ell^t}(i^{\text{est}}_t)$.}
\EndFor
\end{algorithmic}
\end{algorithm}

The guarantees for \Cref{alg:EXP3-exploration-two-query} directly follow from existing work. A proof of the following statement can be found in \Cref{app:omitted-proof-prelim}. 
\begin{proposition}
    \label{prop:two-query-EXP3}
    The expected regret of \Cref{alg:EXP3-exploration-two-query} using $\eta=\frac{1}{n}\cdot \sqrt{\frac{\log{n}}{T}}$ is at most $\O{n\cdot \sqrt{T\log{n}}}$.
\end{proposition}

\subsection{The SQUINT algorithm}
\label{subsec:squint}
We use the SQUINT algorithm for the opimtal boosting as in \cite{pr23}. 
For an algorithm whose distribution over the arms (experts) is $p_t$ at day $t$
define $v_t(i) = \sum_{i=1}^{n} p_t(i) \ell^{t}(i) - \ell^{t}(i)$ be the loss difference between expert i and the algorithm. The SQUINT algorithm (\Cref{alg:squint}) and its guarantees are as follows.

\begin{algorithm}
\label{alg:squint}
\caption{SQUINT, \cite{KoolenE15}, cf. \cite{pr23}}
\begin{algorithmic}[1]
\For{$t = 1, 2, \dots, T$}
    \State Compute $p_t \in \Delta_n$ over experts such that $p_t(i) \propto \mathbb{E}_{\eta}[\eta \cdot \exp(\eta \sum_{\tau=1}^{t-1} v_{\tau}(i) - \eta^2 \sum_{\tau=1}^{t-1} v_{\tau}^2(i))]$
    \State Sample an expert $i_t \sim p_t$ and observe the loss vector $\ell_t \in [0,1]^n$
\EndFor
\end{algorithmic}
\end{algorithm}

\begin{lemma}
\cite{KoolenE15}
\label{lem:squint}
Let the learning rate $\eta$ is sampled from the prior distribution $\gamma(\eta)$ over $[0, 1/2]$ such that $\gamma(\eta) \propto \O{\frac{1}{\eta \log^2(\eta)}}$. For any expert $i \in [n]$, the SQUINT guarantees that
\[
\mathbb{E}\left[\sum_{t=1}^{T} \ell_t(i_t) - \sum_{t=1}^{T} \ell_t(i)\right] \le \O{\sqrt{V_T^i \log(nT)}}.
\]
where $V_T^i = \sum_{t=1}^{T} (\mathbb{E}_{i_t}[l_t(i_t)] - l_t(i))^2$ is the loss variance of expert $i$.
\end{lemma}

\section{Our Baseline: A Low-Memory Algorithm with Single-Query Bandit Losses}
\label{sec:PR-alg-est-eviction}
As we discussed in \Cref{sec:tech-overview}, the baseline of our algorithms for both the sliding-window and streaming models is a variate of the algorithm in \cite{pr23} that works in the single-query bandit setting. 
The algorithm achieves $\Otilde(nT^{3/4})$ regret with $\polylog(nT)$ space with high probability.
The best algorithm that could achieve $T^{3/4}$ regret was the algorithm in \cite{XuZ21} with $\Theta(\sqrt{n})$ memory. As such, on its own, the algorithm already improves the memory efficiency for sublinear-memory online learning in the bandit setting by nearly exponential factors.

\paragraph{High-level overview and necessary notions.}
We now give a high-level our algorithm in this section. 
Roughly speaking, the algorithm modifies the framework of \cite{pr23} to the bandit setting. 
We have discussed a simplified analysis in \Cref{sec:tech-overview}; however, note that directly applying that analysis framework would result in serious dependency issues: to define ``an epoch that could evict the best expert'', we would need to fix the randomness of the experts in the pool; however, if the best expert actually joins the pool, it will interfere with other experts in the pool, which makes the argument circular.

To handle such dependency issues, we would need the same structure as \cite{pr23} that maintains an iterative merging process for the pools.
In particular, the pool is divided into $K=\log{T}$ pools $\calP = \calP_1 \cup \calP_2\cdots \cup \calP_K$, and the $k$-th subpool contains the experts added in the most recent $2^k$ epochs. 
Every $2^k$ epochs, the pool $\calP_k$ merges to pool $\calP_{k+1}$, and we will perform pruning by removing experts that are \emph{covered}, i.e., with loss sequences not competitive with the benchmark loss. 
During this merge, we sample $1/\polylog{(n)}$-fraction of the experts to the \emph{filter} set $\calF$, and the final notion of ``being able to evict'' is defined using a \emph{benchmark loss} over the experts in $\calF$.

We now introduce the necessary notions toward the formal definition of the benchmark loss. When we talk about a \emph{lifespan} (or duration) $\calD(i)$ for an expert $i\in[n]$, we mean an interval $(t_1, t_2]$ such that both $t_1$ and $t_2$ are \emph{integer multiples} of some parameter $B$. 
We use $\calD$ as a short-hand notation when the context is clear. 
Since we operate with the partial information setting, we insist on maintaining the \emph{estimations} of the loss. 
For a given interval $\calD$ and an expert $i$, we use $\widetilde{\calL^{\calD}}(i)$ to denote an estimate for the loss of expert $i$ over the duration $\calD$. It is easy to see that $\card{\widetilde{\calL^{\calD}}(i)-\sum_{t\in\calD}\ell^{t}(i)}\leq \card{\calD}$ since the losses are in $\{0,1\}$. 
We would eventually need estimations with better accuracy for our regret bounds.

\paragraph{The benchmark loss.} 
We are now ready to formally define the benchmark loss.
To better illustrate the process to obtain the benchmark loss, we write it in the form of an algorithm (\algref{alg:strong-benchmark}).
\begin{algorithm}[!htb]
\caption{The Dynamic Benchmark for expert (arm) $i$}
\alglab{alg:strong-benchmark}
\algorithmicrequire{A set $\calF$ of filter experts.}\\
\algorithmicrequire{Epoch length $B$;  a lifespan $\calD=(t_1, t_2]$ which is an \emph{integer multiplier} of $B$.}\\
\algorithmicrequire{The per-epoch estimated cost of each expert $j\in \calF$.}\\
\algorithmicensure{ A benchmark of cost.}
\begin{algorithmic}[1]
\State{Let $t(j_1)\leq t(j_2)\leq t(j_3)\leq \cdots \leq t(j_{\card{\calF}})$ be ranked by the time that the filter expert enters the pool.}
\State{Divide the lifespan $\calD$ in intervals $\calI_1, \calI_2, \cdots, \calI_{{\card{\calD}}/{B}}$.}
\State{Initialize $\widetilde{\calL^{\calD}}(\benchmark)\gets 0$.}
\For{each interval $\calI_\ell$ for $\ell \in [1, \card{L}/{B}]$}
\State{Let $\calI(t(j_{k-1}), t(j_{k}))$ be the time interval such that $t(j_{k-1})\leq (t_1 + (\ell-1)\cdot B) \leq t(j_{k})$.}
\State{$\widetilde{\calL^{\calD}}(\benchmark)\gets\min_{j\in \calF}\widetilde{\calL}^{\calI(t(j_{k-1}), t(j_{k}))}(j)+\widetilde{\calL^{\calD}}(\benchmark)$.}
\EndFor
\State{Return $\widetilde{\calL^{\calD}}(\benchmark)$.}
\end{algorithmic}
\end{algorithm}

In \algref{alg:strong-benchmark}, the filter set $\calF$ is produced in \algref{alg:filter-new}. The epoch length is a parameter by \algref{alg:eviction-with-fake-cost}, and the lifespan of an expert is a function of \algref{alg:eviction-with-fake-cost}. Compared to \cite{pr23}, we use only the \emph{estimations} of the losses to compute the benchmark loss, and the way we maintain the estimated loss can be found in \algref{alg:eviction-with-fake-cost}.

We now formally define the notion of \emph{cover} as was originally defined in \cite{pr23}. Perhaps contrary to an initial intuition, we would evict the arms more \emph{aggressively}. Our definition of cover is as follows. 
\begin{definition}[Approximate cover of an expert]
\label{def:cover}
Given a set $\calF$ of filter arms and an exponent $\rho\in(0,1)$.  
We say that an arm $i$ with lifespan $\calD$ is \emph{covered} if after running \algref{alg:strong-benchmark} with $i$ and $\calF$ to get $\widetilde{\calL^{D}(\benchmark)}$, and we have that
\begin{align*}
\widetilde{\calL^{\calD}(i)}\ge\widetilde{\calL^{\calD}(\benchmark)} - C \cdot n\log(nT)\cdot (\card{\calD})^\rho,
\end{align*}
where $C$ is an absolute constant. 
\end{definition}

We are now ready to introduce our baseline algorithm formally as in \algref{alg:eviction-with-fake-cost}. The control flow is as follows: in each epoch, the algorithm samples arms (experts) to the pool $\calP_1$ with probability $1/n$. 
Then, the algorithm performs EXP3 with $\gamma=(\frac{n}{B})^{1/3}$ as the exploration parameter for the epoch. 
By the end of the epoch, the pool $\calP_k$ will merge to $\calP_{k+1}$ if it's the $2^k$-th epoch. 
The merging procedure is controlled by the $\textsc{Merge}$ algorithm (\algref{alg:merge-new}), which in turn calls algorithm \algref{alg:filter-new} to determine covering and eviction.

\begin{algorithm}[!htb]
\caption{The baseline algorithm with single-query bandit signals}
\alglab{alg:eviction-with-fake-cost}
\textbf{Parameters}: $B$ as the length of each epoch; $\gamma$ as the probability for exploration in EXP3.
\begin{algorithmic}[1]
\State{Maintain a pool $\calP=\calP_{1}\cup \calP_{2}\cup \cdots \calP_{K}$ of arms with $K=\log{T}$.}
\State{For each arm $i \in \calP$, let $D(i)$ be the \emph{lifespan} of $i$ in the pool.}
\State{Divide the horizon into $T/B$ epochs.}
\State{Maintain the estimated cost for every epoch for arms in the pool $\calP$.}
\For{Each epoch $\tau\in [T/B]$}
\State{Sample each arm and add to $\calP_{1}$ with probability $\frac{1}{n}$.}
\State{Run EXP3 for all arms in the pool $\calP$ with parameter $\gamma=(\frac{n}{B})^{1/3}$.}
\For{each interval $\calI$ induced by the time experts enters the pool}
\State{Maintain $\widetilde{\calL^{\calI}}(i)=\sum_{t: i_t=i}\widetilde{\ell^{t}}(i)$ for each $i$, where $\widetilde{\ell^{t}}(i)$ is the fake loss of EXP3}
\EndFor
\If{$\tau=C\cdot 2^{C'}$ for some integer $C, C'$}
\State{Let $\pw(\tau)$ be the largest integer such that $\tau=C\cdot 2^{\pw(\tau)}$ for some integer $C$.}
\For{$k \in [\pw(\tau)]$}
\State{Update $\calP_{k+1}$ with \algref{alg:merge-new} as the merge of $\calP_{k+1}$ and $\calP_{k}$.}
\State{$\calP_{k}\gets \emptyset$.}
\EndFor
\EndIf
\EndFor
\end{algorithmic}
\end{algorithm}

\begin{algorithm}[!htb]
\caption{\textsc{Merge}$(\calP_A,\calP_B)$}
\alglab{alg:merge-new}
\textbf{Input}: Pools $\calP_{A}$ and $\calP_{B}$.\\
\textbf{Parameter}: $Q=16\log(nT)$.
\begin{algorithmic}[1]
\State{Initially let $\calP_{C}=\calP_{A} \cup \calP_{B}$.}
\For{$q=1:Q$}
\State{Estimate the size $s$ of $\calP_{C}$ by assigning a $\bern{\frac{1}{\log^{4}(nT)}}$ random variable for each arm in $\calP_{C}$.}
\If{$s\leq \log^{5}(nT)$}
\State{Return the pool $\calP_{C}$.}
\Else
\State{Sample $\frac{1}{\log^{4}(nT)}$-fraction of the arms to get a filter set $\calF$.}
\State{Run the \algref{alg:filter-new} with $\calP_C$ and $\calF$ to get set $\calX$.}
\State{$\calP_{C}=\calF\cup \calX$.}
\EndIf
\EndFor
\end{algorithmic}
\end{algorithm}

\begin{algorithm}[!htb]
\caption{\textsc{Filter}$(\calF,\calQ) $}
\alglab{alg:filter-new}
\textbf{Input}: Pools $\calQ$ and filter set of arms $\calF$.
\begin{algorithmic}[1]
\For{each arm $i\in\calQ$}
\State{If $i$ is covered by $\calF$ (as prescribed by Definition~\ref{def:cover}) with $\rho=2/3$, then remove $i$ from $\calQ$.}
\EndFor
\State{Return the updated set of $\calQ$.}
\end{algorithmic}
\end{algorithm}

The rest of this section is to prove the memory and regret for \algref{alg:eviction-with-fake-cost}. 

\subsection{The analysis: \texorpdfstring{$\polylog(nT)$}{polylog(nT)} memory and \texorpdfstring{$\Otilde(nT^{3/4})$}{poly(n)T3/4} regret}

Similar to the case of \cite{pr23}, we first analyze the memory complexity since the regret analysis depends on it. 
The next lemma shows that the \textsc{Merge} algorithm is efficient, which forms the backbone of our memory efficiency analysis.
\begin{lemma}
\lemlab{lem:merge:size}
With probability $1-\frac{1}{\poly(nT)}$, the output $\calP_C$ of $\textsc{Merge}(\calP_A,\calP_B)$ satisfies
\[|\calP_C|\le\max\left(2\log^9(nT),\frac{1}{4}(|\calP_A|+|\calP_B|)\right).\]
\end{lemma}
\begin{proof}
Suppose $\textsc{Filter}$ has been called for $q_{\max}\in[Q]$ times and for $q\in[q_{\max}]$, let $\calP_{C,q}$ be the pool $\calP_C$ at the beginning of the $q$-th time. 
Let $\calE$ be the event that the cost of each expert $i$ that has been in the pool for $\card{\calD}$ time is estimated correctly up to a multiplicative $\left(1+\O{\frac{n\log(nT)}{(\card{\calD})^{1-\rho}}}\right)$-approximation. 
We shall show that conditioned on $\calE$, then with probability at least $1-\frac{1}{\poly(nT)}$, we have either:
\begin{itemize}
\item 
$|\calP_{C,q}|\le2\log^9(nT)$
\item 
$|\calP_{C,q+1}| \le\left(1-\frac{1}{6\log(nT)}\right)\cdot|\calP_{C,q}|$. 
\end{itemize}
Observe that such a statement would then imply after taking a union bound over $k\in[k_{\max}]$ that
\begin{align*}
|\calP_{C,q_{\max}+1}|&\le\max\left(2\log^9(nT),\left(1-\frac{1}{6\log(nT)}\right)^{16\log(nT)}\cdot|\calP_{C,1}|\right)\\
&\le\max\left(2\log^9(nT),\frac{1}{4}(|\calP_A|+|\calP_B|)\right),
\end{align*}
since $|\calP_{C,1}|=|\calP_A|+|\calP_B|$. 

To show the claim, we fix $q\in[q_{\max}]$ and suppose by that $|\calP_{C,q}|>2\log^9(nT)$. 
Hence, it suffices to prove $|\calP_{C,q+1}| \le\left(1-\frac{1}{6\log(nT)}\right)\cdot|\calP_{C,q}|$ with high probability. 
Note that by a standard Chernoff bound, the estimated size $s$ satisfies
\[\PPr{s\le\log^5(nT)}\le\frac{1}{\poly(nT)}.\]

Now, we let $M=|\calP_{C,q}|$ and initialize $\calW_1=\calP_{C,q}$. 
For each $r\ge 1$ and each expert $i\in\calW_r$, we define
\[V_r(i)=\left\{j\in\calW_r\setminus\{i\}\,|\,\calL^{\calI(t(i), t(j_{r-1}))}(i)\ge\calL^{\calI(t(i), t(j_{r-1}))}(j)-2^{r-1}\right\},\]
so that $V_r(i)$ is the set of experts that will cover $i$ over the interval $\calI(t(i), t(j_{r-1}))$. 
We also define
\[R_r=\{i\in\calW_r\,|\,|V_r(i)|\ge\log^6(nT)\},\]
so that $R_r$ is the set of experts $i\in\calW_r$ that have $V_r(i)$ with large size. 
We shall ultimately argue that these experts are readily removed by our algorithm, as they are likely covered by the combination of the experts $j_1,\ldots,j_{r-1}$ in the filter in conjunction with some expert from $V_r$. 

Let $j_r$ be the expert in $\calW_r\setminus R_r$ that most recently entered the pool that has also been sampled into the filter $\calF_q$, so that 
\[j_r=\argmax_{j\in(\calW_r\setminus R_r)\cap\calF_q} E_j,\]
where $E_j$ is the entry time of expert $j$, if such an expert exists. 
Otherwise, we leave $j_r$ undefined. 
If $j_r$ is defined, we define the set $O_r$ to be the experts that enter the pool later than $j_r$, so that
\[O_r=\{i\in\calW_r\setminus R_r\,|\,E_i\ge E_{j_r}\}.\]
Finally, we define $\calW_{r+1}=\calW_r\setminus(R_r\cup O_r\cup V_r(j_r))$ for each $r\ge 1$. 

\paragraph{Induction base case.}
Let $r\ge 1$ and let $\calE_1$ be the event that $R_m\le\frac{1}{4\log(nT)}\cdot M$ for all $m\in[r]$, where we recall that $M=|\calP_{C,q}|$. 
We inductively show that
\begin{enumerate}
\item 
$|\calW_{r+1}|\ge\left(1-\frac{r}{2\log(nT)}\right)\cdot M$
\item
We have $j_1,\ldots,j_r\in\calF_q$ and $j_r\preceq j_{r-1}\preceq\ldots\preceq j_1$
\item 
For any $m\in[r]$ and expert $i\in\calW_{r+1}$, we have
\[\calL^{\calI(t(j_m), t(j_{m-1}))}(i)\ge\calL^{\calI(t(j_m), t(j_{m-1}))}(j_m)+2^{m-1}.\]
\end{enumerate}
We observe that the base case of $r=0$ holds immediately and suppose the statement holds up to $r-1$.

\paragraph{Inductive step.}
We now proceed with some casework on the size of $R_r$. 

\textbf{Case 1.} 
In the case where $|R_r|\ge\frac{1}{4\log(nT)}\cdot M$, we have that for each expert $i\in R_r$, each corresponding expert $j\in V_r(i)$ is sampled into the filter set $\calF_q$ with probability $\frac{1}{\log^4(nT)}$. 
Therefore, the probability that no expert $j\in V_r(i)$ is sampled into $\calF_q$ is at most $\left(1-\frac{1}{\log^4(nT)}\right)^{\log^6(nT)}\le\frac{1}{\poly(nT)}$, since $|V_r(i)|\ge\log^6(nT)$ for any $i\in R_r$ by definition. 
Hence with high probability, there exists an expert $j\in V_r(i)\cap\calF_q$. 
Then conditioned on $\calE$, we have:
\begin{align*}
\calL&^{\calI(t(i), t(j_0)}(i)=\calL^{\calI(t(i), t(j_{r-1})}(i)+\sum_{m=1}^{r-1}\calL^{\calI(t(j_m), t(j_{m-1})}(i)\\
&\ge\calL^{\calI(t(i), t(j_{r-1})}(i)+\sum_{m=1}^{r-1}\left(\calL^{\calI(t(j_m), t(j_{m-1})}(i)\right)-n\log(nT)\left(\frac{t(j_{m-1})-t(j_m)}{B}\right)^{\rho}\\
&\ge\left(\calL^{\calI(t(i), t(j_{r-1})}(j)-2^{r-1}\right)
+\sum_{m=1}^{r-1}\left(\calL^{\calI(t(j_m), t(j_{m-1}))}(j_m)+2^{m-1}\right)-n\log(nT)\left(\frac{t(j_{m-1})-t(j_m)}{B}\right)^{\rho}\\
&\ge\sum_{m=1}^{r-1}\left(\calL^{\calI(t(j_m), t(j_{m-1}))}(j_m)\right)+\calL^{\calI(t(i), t(j_{r-1})}(j)-1-n\log(nT)\left(\frac{t(j_{m-1})-t(j_m)}{B}\right)^{\rho}.
\end{align*}
Again conditioning on the event $\calE$, so that each expert has its loss well-approximated, we have
\begin{align*}
\calL^{\calI(t(i), t(j_0))}(i)&\ge\sum_{m=1}^{r-1}\left(\calL^{\calI(t(j_m), t(j_{m-1}))}(j_m)\right)+\calL^{\calI(t(i), t(j_{r-1})}(j)-3n\log(nT)\left(\frac{t(j_{m-1})-t(j_m)}{B}\right)^{\rho}.
\end{align*}
Hence with high probability, the set of experts $R_r\setminus\calF_q$ will be removed from the pool. 
Observe that the size of the filter $\calF_q$ is at most $\frac{2}{\log^4(nT)}$ with high probability by standard Chernoff bounds. 
Since we have $|R_r|\ge\frac{1}{4\log(nT)}\cdot M$, then by a union bound, we have that $|R_r\setminus\calF_q|\ge\frac{1}{6\log(nT)}$ with high probability. 

\textbf{Case 2.} 
In the other case, we have $|R_r|<\frac{1}{4\log(nT)}\cdot M$. 
Since $j_r$ is defined as the latest expert of $\calW_r\setminus R_r$ that is sampled into $\calF_q$ and each expert is sampled into $\calF_q$ with probability $\frac{1}{\log^4(nT)}$, then we have that with high probability, at most $\log^6(nT)$ experts in $\calW_r\setminus R_r$ arrive later than $j_r$. 
In other words, $|O_r|\le\log^6(nT)$ with high probability. 
Since we recursively defined $\calW_{r+1}=\calW_r\setminus(R_r\cup O_r\cup V_r(j_r))$, then we have
\begin{align*}
|\calW_{r+1}|&\ge|\calW_r|-|R_r|-|O_r|-|V_r(j_r)|\\
&\ge\left(1-\frac{r-1}{2\log(nT)}\right)\cdot M-\frac{1}{4\log(nT)}\cdot M-\log^6(nT)-\log^6(nT)\\
&\ge\left(1-\frac{r}{2\log(nT)}\right)\cdot M,
\end{align*}
where the last step is due to the assumption that $|M|\ge2\log^9(nT)$ in this case. 
We have $j_r\in\calF_q$ by definition and moreover $j_r\preceq j_{r-1}\preceq\ldots\preceq j_1$ since $\calW_r$ is a subset of $\calW_{r-1}$. 
Finally, since $V_r(j_r)\cap\calW_{r+1}=\emptyset$, then the inductive step is complete. 

Putting things together, observe that if $R_r<\frac{1}{4\log(nT)}\cdot M$ for all $r\in[\log(nT)+1]$, then $V_r(i)=\calW_r$ for all $i\in\calW_r$, since the loss of any expert $j$ is at most $T$. 
Therefore, $|R_r|=|\calW_r|\ge\left(1-\frac{\log(nT)}{2\log(nT)}\right)\cdot M=\frac{1}{2}\cdot M$, in which case we have a stronger guarantee that $|\calP_{C,q+1}|\le\frac{1}{2}\cdot|\calP_{C,q}|$. 
\end{proof}

\begin{lemma}
\lemlab{lem:num:experts}
For any epoch $\tau \in\left[\frac{T}{B}\right]$ and any $q\in[Q]$, we have that with high probability, $|\calP_q^{(\tau)}|\le2\log^9(nT)$.
\end{lemma}
\begin{proof}
We argue by induction on the level $q\in[Q]$ of the pool.  
Note that for any fixed $\tau\in\left[\frac{T}{B}\right]$, by a standard Chernoff bound, we have that $|P_q^{(\tau)}|\le2\log^9n(nT)$ with high probability. 
Now suppose the claim holds up to level $q-1$. 
Since $\calP_q$ is empty at the beginning and merged with $\calP_{q-1}$ every $2^{q-1}$ epochs, then by \lemref{lem:merge:size}, we have that with high probability $|\calP_q^{(\tau)}|\le2\log^9(nT)$.
\end{proof}

\begin{lemma}
\lemlab{lem:pool:mem}
With high probability, we have that for all times $t\in[T]$, $|\calP^{(t)}|=\O{\log^{10}(nT)}$, and thus the memory used by \algref{alg:eviction-with-fake-cost} is at most $\O{\log^{20}(nT)}$ words of space. 
\end{lemma}
\begin{proof}
By \lemref{lem:num:experts}, we have $|\calP_q^{(\tau)}|\le2\log^9(nT)$ for all $\tau\in\left[\frac{T}{B}\right]$ and any $q\in[Q]$. 
Since $Q=\O{\log(nT)}$ and $\calP^{(t)}$ is equal to the nearest $\calP^{(\tau)}$ for $\tau \in \left[\frac{T}{B}\right]$, we have $|\calP^{(t)}|\le\sum_{q\in[Q]}|\calP_q^{(\tau)}|=\O{\log^{10}(nT)}$. 
It remains to store the loss of each expert $i\in\calP^{(t)}$ across the lifespan of each expert $j\in\calP^{(t)}$, which requires $\O{\log^{10}(nT)}$ words of space for each expert $i$ due to the number of experts $j\in\calP^{(t)}$. 
Hence, the total memory used by \algref{alg:eviction-with-fake-cost} is at most $\O{\log^{20}(nT)}$ words of space. 
\end{proof}

\subsection{Regret analysis}
\label{subsec:baseline-regret-analysis}
We now move to the regret analysis. 
Compared to the case in \cite{pr23}, one complication in our algorithm is that to bound the \emph{actual} regret, we cannot simply use the cover in \Cref{def:cover} to conduct \emph{analysis}. 
This is due to the fact that we need to define \emph{good} and \emph{bad} epochs in the algorithm. 
Roughly speaking, a good epoch is defined as an epoch such that if the best expert $i^*$ joins the pool, it could eventually be covered by the filter.
However, if we use the covering notion as in \Cref{def:cover}, then what does it even mean to cover $i^*$ with \emph{random variables}?
As discussed in \Cref{sec:tech-overview}, one of the technical innovations in our analysis is to introduce the notion of \emph{conceptual cover}, defined as follows.
\begin{definition}[Conceptual cover of an expert]
\deflab{def:concept:cover}
Given a set $\calF$ of filter arms, and let expert $i$ be with lifespan $\calD$. 
Let $\calL^{\calD}(\benchmark)$ be constructed in a similar manner as $\widetilde{\calL^{\calD}(\benchmark)}$ but using the loss $\calL$ of each expert, rather than the estimated cost $\widetilde{\calL}$. 
We say that an arm $i$ is \emph{covered} if
\begin{align*}
\calL^{\calD}(i) \geq \calL^{\calD}(\benchmark) - 3C \cdot n\log(nT)\cdot (\card{\calD})^{\rho},
\end{align*}
where $C$ is an absolute constant.
\end{definition}

We now formalize the notion of the good and bad epochs in the same manner as in \cite{pr23}.
\begin{definition}[Active and passive experts]
    For any epoch $\tau\in [T/B]$, we say an expert is passive if the expert is not sampled during the filtration process, either into the filter or to estimate the size. 
\end{definition}

Fix the randomness used by the loss sequence, the active expert, and the \emph{exploration in EXP3}, for each epoch $a_\tau$ we assume the best expert $i^*$ joins the pool on epoch $a_\tau$.
We further define the \emph{eviction time} $t(a_{\tau})$ as the first time $i^*$ is \emph{conceptually covered} by the set of active experts. If $i^*$ is never conceptually covered, we add $a_\tau$ to the set of bad epochs $\calH$.

\begin{lemma}
\lemlab{lem:regret:partial}
Let $\calE$ be the event that the cost of each expert $i$ that has been in the pool for $\card{\calD(i)}$ time is estimated correctly up to a multiplicative $\left(1+\O{\frac{n\log(nT)}{\card{\calD(i)}^{1-\rho}}}\right)$-approximation. 
Condition on the event $\calE$. 
Furthermore, conditioning on a fixed loss sequence and set of sampled and active experts for each epoch, then with high probability, we have
\[\sum_{t=1}^{T/B}\sum_{b=1}^B\ell^{(t-1)B+b}(i_{(t-1)B+b})-\sum_{t=1}^{T}\ell^{t}(i^*)\le B\cdot|\calH|+\O{\frac{T}{B}\cdot\left(\gamma B\log(nT) + B^{\rho}n\log(nT)\right)}.\]
\end{lemma}
\begin{proof}
For each $t\in\left[\frac{T}{B}\right]$, let $i_t^*$ be the best expert of pool $\calP^{(t)}$ during epoch $t$. 
Let $\calH$ be the set of bad epochs. 
Then with high probability, we have
\begin{align*}
\sum_{t=1}^{T/B}\sum_{b=1}^B\ell^{(t-1)B+b}(i_{(t-1)B+b})&-\sum_{t=1}^{T}\ell^{t}(i^*)=\sum_{t=1}^{T/B}\sum_{b=1}^B\ell^{(t-1)B+b}(i_{(t-1)B+b})-\sum_{t=1}^{T/B}\calL^t(i^*)\\
&=\sum_{t=1}^{T/B}\sum_{b=1}^B\ell^{(t-1)B+b}(i_{(t-1)B+b})-\sum_{t=1}^{T/B}\calL^t(i_t^*)+\sum_{t=1}^{T/B}\calL^t(i_t^*)-\sum_{t=1}^{T/B}\calL^t(i^*)\\
&\le\frac{T}{B}\cdot\left(\gamma B\log(nT) + B^{\rho}n\log(nT)\right)+\sum_{t=1}^{T/B}\calL^t(i_t^*)-\sum_{t=1}^{T/B}\calL^t(i^*),
\end{align*}
conditioned on (1) the success of EXP3, since the loss over each epoch $t$ is competitive with the best expert $i_t^*$ in the pool $\calQ^{(t)}$ during epoch $t$, and on (2), the high probability event that the exploration rate $\gamma$ is selected at most $\gamma B\log(nT)$ times over any epoch of length $B$.

We now decompose the epochs $t\in\left[\frac{T}{B}\right]$ into the bad epochs $t\in\calH$ and the good epochs $t\notin\calH$. 
\begin{align*}
\sum_{t=1}^{T/B}\calL^t(i_t^*)-\sum_{t=1}^{T/B}\calL^t(i^*)&\le\sum_{t\in\calH}\left(\calL^t(i_t^*)-\calL^t(i_t^*)\right)+\sum_{t\notin\calH}\left(\calL^t(i_t^*)-\calL^t(i_t^*)\right)\\
&\le B\cdot|\calH|+\frac{T}{B}\cdot\left(\gamma B\log(nT) + B^{\rho}n\log(nT)\right),
\end{align*}
again conditioning on the success of EXP3. 
This completes the regret bound as claimed. 
\end{proof}
We now upper bound the number of bad epochs. 
\begin{lemma}
\lemlab{lem:bad:epochs}
Let $\calH$ be the set of bad epochs. 
With high probability, we have $|\calH|\le\O{n\log^{10}(nT)}$.
\end{lemma}
\begin{proof}
Conditioning on a fixed loss sequence $\ell^1,\ldots,\ell^{T}$ and a fixed sequence of sampled and active experts $Y_t$ for each epoch $t$, let $W_t$ be the set of sampled and passive experts, so that $W_t\cap Y_t=\emptyset$. 
Let $\calH$ be the set of bad epochs and let $\calE$ be the event that the cost of each expert $i$ that has been in the pool for $D(i)$ time is estimated correctly up to a multiplicative $\left(1+\O{\frac{n\log(nT)}{\card{\calD(i)}^{1-\rho}}}\right)$-approximation. 
For any bad epoch $t\in\calH$, we have that conditioned on $\calE$, we must have $i^*\notin Y_t$ because otherwise $i^*$ would be evicted by itself, contradicting the definition of bad epoch $t$. 
Then we have
\begin{align*}
\PPr{i^*\in W_t\,|\,Y_1,\ldots,Y_{T/B},\ell^1,\ldots,\ell^T}&=\PPr{i^*\in W_t\,|\, i^*\notin Y_t}\\
&\ge\PPr{i^*\in W_t}=\frac{1}{n}\cdot\left(1-\frac{1}{\log^4(nT)}\right)^{2K\cdot 2Q}\ge\frac{1}{2n}.
\end{align*}
Since $i^*\in W_t$ is an independent event across all $t\in\calH$ conditioned on the fixing of the loss sequence and the active and sampled experts, then by standard Chernoff bounds, we have
\begin{align*}
\PPr{|\calQ|\le\frac{\card{\calH}}{4n}\,\mid\,Y_1,\ldots,Y_{T/B},\ell^1,\ldots,\ell^T}&\le\PPr{\sum_{t\in\calH}\mathbbm{1}[i^*\in W_t]\le\frac{\card{\calH}}{4n}\,\mid\,Y_1,\ldots,Y_{T/B},\ell^1,\ldots,\ell^T}\\
&\le\exp\left(-\frac{|\calH|}{16n}\right),
\end{align*}
where $\calQ$ is the entire pool at time $t$. 
By \lemref{lem:pool:mem}, we have that $|\calQ|=\O{\log^{10}(nT)}$ with high probability. 
Therefore, we have that with high probability, $|\calH|\le\O{n\log^{10}(nT)}$. 
\end{proof}

\begin{lemma}
\lemlab{lem:gen:regret:mult}
Let $\calE$ be the event that the cost of each expert $i$ that has been in the pool for $D(i)$ time is estimated correctly up to a multiplicative $\left(1+\O{\frac{n\log(nT)}{\card{\calD(i)}^{1-\rho}}}\right)$-approximation, where we have $D(i)\ge B$. 
Condition on the event $\calE$. 
We have with high probability, the regret of \algref{alg:eviction-with-fake-cost} is at most
\begin{align*}
\paren{nB + \gamma\cdot T + nT\cdot B^{\rho-1}}\cdot \polylog{(nT)}.
\end{align*}
\end{lemma}
\begin{proof}
By \lemref{lem:regret:partial}, the regret is at most $B\cdot|\calH|+\O{\frac{T}{B}\cdot\left(\gamma B\log(nT) + B^{\rho}n\log(nT)\right)}$. 
By \lemref{lem:bad:epochs}, we have that with high probability, $|\calH|\le\O{n\log^{10}(nT)}$. 
Therefore, the regret is at most $Bn\cdot\polylog(nT)+\O{\frac{T}{B}\cdot\left(\gamma B\log(nT) + B^{\rho}n\log(nT)\right)}$, which could be simplified to the form as in the lemma statement.
\end{proof}

We now give the main lemma of the regret for our baseline algorithm.
\begin{lemma}
With parameters $\gamma=(\frac{n}{B})^{1-\rho}$ for $\rho=\frac{2}{3}$ and $B=\O{T^{3/4}}$, the regret of \algref{alg:eviction-with-fake-cost} is $n T^{3/4}\cdot \polylog(n)$ with probability at least $1-1/\poly(nT)$.
\end{lemma}
\begin{proof}
By \lemref{lem:regret:partial}, the regret is at most $B\cdot|\calH|+\O{\frac{T}{B}\cdot\left(B^{2/3}n\log(nT)\right)}$. 
By \lemref{lem:bad:epochs}, we have that with high probability, $|\calH|\le\O{n\log^{10}(nT)}$. 
Therefore, the regret is $n T^{3/4}\cdot \polylog(nT)$ for $B=\O{T^{3/4}}$. 
\end{proof}

\section{Achieving \texorpdfstring{$\sqrt{T}$}{sqrtT} Regret with Two-Query Signals in the Streaming Model}
\label{sec:two-query-alg}
En route to our sliding-window algorithm with interval regret and two queries on each day, we now consider the problem of online learning over $T$ days with two queries per day, i.e., the streaming regret minimization with two queries on each day.
We will show an algorithm that achieves $\Otilde(n\sqrt{T})$ regret with $\polylog(nT)$ space, which is an important intermediate step toward the sliding-window regret. 
To articulate our main ideas, we focus on obtaining the optimal regret on $T$ (which is $\sqrt{T}$) but not on $n$, and we will obtain \Cref{thm:regret-space-two-query-boosting} in this section. We defer the optimal boosting to \Cref{sec:monocarpic-boosting}.

We start with modifying the baseline algorithm to the two-query setting. 
Here, we are going to leverage the \emph{lossless} signal to provide low-variance estimations for the losses of each arm, hence achieving $T^{2/3}$ (instead of $T^{3/4}$) regret on a single baseline algorithm. 
The modified baseline algorithm is as \algref{alg:two-query:baseline}.

\begin{algorithm}[!htb]
\caption{Baseline algorithm $\baseline_0$}
\alglab{alg:two-query:baseline}
\algorithmicrequire{Time horizon $T$, estimated losses $\widetilde{\calL(i)}$ for all arms $i\in\calP$ over their duration in pool $\calP$}
\begin{algorithmic}[1]
\State{$\calP\gets\emptyset$, $B\gets\sqrt{T}\cdot\polylog(nT)$, $\eta=\frac{1}{\sqrt{T}}$}
\For{each epoch of length $B$}
\State{Sample each arm into $\calP$ with probability $\frac{1}{n}$}
\State{Maintain the same hierarchical structure as in Section~\ref{sec:PR-alg-est-eviction}}
\For{each time $t$ in the epoch}
\State{Update $\widetilde{\calL(i)}$ by uniformly sample an arm in $\calP$ and pull it with the \emph{lossless} signal}
\State{Play arm $i$ with probability proportional to $\exp(-\eta\widetilde{\calL(i)})$ as in \Cref{alg:EXP3-exploration-two-query}}
\NoNumber{\Comment{Do not update the estimated losses}}
\EndFor
\State{Using estimated losses, evict arms covered by the filter}
\EndFor
\end{algorithmic}
\end{algorithm}

\begin{lemma}
\lemlab{lem:bl:zero:regret}
Suppose that at all times $t\in[T]$, the estimates $\widetilde{\calL(i)}$ for the loss $\calL(i)$ of arm $i$ across its duration $\calD$ with $\card{\calD}=D$ in $\calP$ has additive error at most $\sqrt{D}\cdot\polylog(nT)$. 
Then the regret of \algref{alg:two-query:baseline} is at most $nT^{2/3}\cdot\polylog(nT)$. 
Furthermore, the memory used by \algref{alg:two-query:baseline} is at most $\polylog(nT)$ bits.
\end{lemma}
\begin{proof}
Observe that the probability vector computed by $\baseline$ in \algref{alg:two-query:baseline} is precisely the probability vector of MWU. 
Moreover, \algref{alg:eviction-with-fake-cost} uses EXP3 as a subroutine, which is simply MWU on estimated losses for each arm in the pool $\calP$. 
Hence conditioned on the estimates $\widetilde{\calL(i)}$ for the loss $\calL(i)$ of arm $i$ across its duration in $\calP$ having additive error at most $\sqrt{T}\cdot\polylog(nT)$, then \algref{alg:two-query:baseline} is identical to \algref{alg:eviction-with-fake-cost} with parameters $\gamma=0$ and $\rho=\frac{1}{2}$. 
Note that here, by \Cref{prop:two-query-EXP3}, the expected regret on each epoch by running the EXP3 algorithm is $B^\rho \polylog(nT)$. 
Therefore, we can apply \lemref{lem:gen:regret:mult} and conclude that the expected regret of \algref{alg:two-query:baseline} is at most 
\[Bn\cdot\polylog(nT)+\O{\frac{T}{B}\cdot\left(\gamma B\log(nT) + B^{\rho}n\log(nT)\right)} \leq \left(\frac{nT}{B}\cdot\sqrt{B}+Bn\right)\cdot\polylog(nT).\] 
Thus for $B=T^{2/3}$, we have that the expected regret of \algref{alg:two-query:baseline} is at most $nT^{2/3}\cdot\polylog(nT)$. 

For the memory complexity, conditioning on the assumption that the estimation $\widetilde{\calL(i)}$ for the loss $\calL(i)$ of arm $i$ across its duration in $\calP$ having additive error at most $\sqrt{T}\cdot\polylog(nT)$, we can  use \lemref{lem:merge:size} with $\rho=1/2$, i.e., we evict with $\sqrt{T}\polylog(nT)$ additive error, the guarantees of pool size merging still hold. Therefore, the memory complexity follows by at most $\polylog(nT)$ bits.
\end{proof}

We now ``boost'' this algorithm to regret $n T^{1/2} \poly(\polylog{nT})$ with an idea similar to those in \cite{pr23}. However, since we work with partial information, we cannot simply claim that the regret for each level of the baselines is low. Instead, we need to bound the cost of each arm obtained by the signals from the uniform exploration. 

The key observation here is that if we \emph{estimate} the losses of experts and algorithms via uniform sampling, the eviction time is still well-defined once we fix the randomness of the losses $\{\ell^{t}\}_{t=1}^{t}$, the set of active experts, and the \emph{second query}. 
As such, we are still able to conduct the analysis in the same manner of \cite{pr23}.

\begin{algorithm}[!htb]
\caption{Two-query algorithm $\baseline_k$}
\alglab{alg:two-query}
\algorithmicrequire{Time horizon $T$}
\begin{algorithmic}[1]
\State{Initialize an instance $\calA_1$ of $\baseline_{k-1}$ with time horizon $T$}
\State{$B\gets n^{(2-2^{k+2})/(2^{k+2}-1)}\cdot T^{1-\frac{1}{2^{k+2}-1}}$}
\For{each epoch of length $B$}
\State{$\widetilde{\calL(\calA_1)}, \widetilde{\calL(\calA_2)}\gets 0$}
\State{Initialize an instance $\calA_2$ of $\baseline_{k-1}$ with time horizon $B$}
\State{Initialize EXP3 on $\calA_1$ and $\calA_2$}
\For{each $t\in[T]$}
\State{Let $\calP$ be the pool of arms maintained by $\calA_1$ and $\calA_2$}
\State{\textbf{with} probability $\frac{1}{2}$}
\State{\hspace{1.5em}Observe a random arm $i\in\calP$ with the regretless signal}
\State{\hspace{1.5em}Increase $\widetilde{\calL(i)}$ by ${2|\calP|}\cdot\ell^t(i)$} 
\State{\hspace{1.5em}Update $\calA_1$ and $\calA_2$ with $\widetilde{\calL(i)}$}
\State{\textbf{otherwise}}
\State{\hspace{1.5em}Choose $j\in\{1,2\}$ uniformly at random}
\State{\hspace{1.5em}Observe the arm selected by $\calA_j$ with the regretless signal}
\State{\hspace{1.5em}Increase $\widetilde{\calL(\calA_j)}$ by $4\ell^t(i)$}
\State{Use MWU on $\widetilde{\calL(\calA_1)}$ and $\widetilde{\calL(\calA_2)}$ to choose an arm to pull}
\NoNumber{\Comment{Do not update the estimated losses}}
\EndFor
\EndFor
\end{algorithmic}
\end{algorithm}

We first show that the true losses of each arm over its duration in the pool is well-estimated by the algorithm.
\begin{lemma}
\lemlab{lem:two:arms:acc}
Consider \algref{alg:two-query} and suppose $\ell^t(i)\in[0,1]$ for all $t\in[T]$. 
For any arm $i\in[n]$ that has been in the pool $\calP$ over a duration $\calD$ of length $D$, we have
\[\left\lvert\widetilde{\calL(i)}-\sum_{t\in\calD}\ell^t(i)\right\rvert\le\sqrt{D}\cdot\polylog(nT),\]
with high probability.
\end{lemma}
\begin{proof}
For any $t\in\calD$, let $\widetilde{\ell^t(i)}$ be the contribution of the estimate due to time $t$ toward $\widetilde{\calL(i)}$, so that $\widetilde{\calL(i)}=\sum_{t\in D(i)}\widetilde{\ell^t(i)}$, where $D(i)$ is the lifespan of $i$. 
Let $\calP_t$ denote the state of $\calP$ at time $t$. 
Let $\calE$ denote the event that $|\calP_t|\le\polylog(nT)$ for all $t\in[T]$ so that $\PPr{\calE}\ge1-\frac{1}{\poly(nT)}$ by \lemref{lem:bl:zero:regret}. 
We have that $\widetilde{\ell^t(i)}={2|\calP_t|}\cdot\ell^t(i)$ with probability $\frac{1}{2|\calP_t|}$ and otherwise, $\widetilde{\ell^t(i)}=0$ with probability $1-\frac{1}{2|\calP_t|}$. 
Then conditioned on $\calE$, we have 
\[\Ex{\widetilde{\ell^t(i)}}=\ell^t(i),\qquad\Ex{(\widetilde{\ell^t(i)})^2}=(\ell^t(i))^2\cdot\polylog(nT).\]
Hence by Bernstein's inequality, c.f., \Cref{thm:bernstein}, we have that with high probability,
\[\left\lvert\widetilde{\calL(i)}-\sum_{t\in\calD}\ell^t(i)\right\rvert\le\sqrt{D}\cdot\polylog(nT),\]
when $\ell^t(i)\in[0,1]$ for all $t\in[T]$. 
\end{proof}
We next show that the true losses of each baseline is well-estimated by the algorithm.
\begin{lemma}
\lemlab{lem:two:meta:acc}
Consider \algref{alg:two-query} and suppose $\ell^t(i)\in[0,1]$ for all $t\in[T]$. 
For a fixed $t\in[T]$ and each $j\in\{1,2\}$, let $\calL(\calA_j)$ be the loss of $\calA_j$ up to and including time $t$, and let $\widetilde{\calL(\calA_j)}$ be the estimate. 
Then, for any epoch with length $B$, 
\[\left\lvert\widetilde{\calL(\calA_j)}-\calL(\calA_j)\right\rvert\le\sqrt{B}\cdot\polylog(nT),\]
with high probability.
\end{lemma}
\begin{proof}
Consider a fixed epoch of length $B$. 
Let $t\in[T]$ be fixed and let $\calP_t$ denote the state of $\calP$ at time $t$. 
Let $\widetilde{\calL}(\calA_j,t)$ denote the contribution of the estimate due to time $t$ toward $\widetilde{\calL(\calA_j)}$ and let $\calL(\calA_j,t)$ denote the true loss of $\calA_j$ at time $t$.  
Let $\calE$ denote the event that $|\calP_t|\le\polylog(nT)$ for all $t\in[T]$ so that $\PPr{\calE}\ge1-\frac{1}{\poly(nT)}$ by \lemref{lem:bl:zero:regret}. 
Observe that $\widetilde{\calL}(\calA_j,t)=4\cdot\calL(\calA_j,t)$ with probability $\frac{1}{4}$ and $\widetilde{\calL}(\calA_j,t)=0$ with probability $\frac{3}{4}$. 
Conditioned on $\calE$, we thus have 
\[\Ex{\widetilde{\calL}(\calA_j,t)}=\calL(\calA_j,t),\qquad\Ex{(\widetilde{\calL^t(i)})^2}\le 4,\]
provided that $\ell^t(i)\in[0,1]$ for all $t\in[T]$. 
Therefore, by Bernstein's inequality, c.f., \Cref{thm:bernstein}, 
\[\PPr{\left\lvert\widetilde{\calL(i)}-\sum_{t\in\calD}\ell^t(i)\right\rvert\ge\sqrt{B}\cdot\polylog(nT)}\le\frac{1}{\poly(nT)},\]
as desired.
\end{proof}
We now upper bound the regret of each baseline. 
\begin{lemma}
\label{lem:two-query-baseline-k-regret}
The regret of $\baseline_k$ is $n\cdot T^{2^{k+1}/(2^{k+2}-1)}\cdot\polylog(nT)$ for any $k\leq \polylog(nT)$. 
\end{lemma}
\begin{proof}
We first fix the times during which an arm from $\calP$ is uniformly sampled and observed, and consequently, the times during which a meta-expert $\calA_j$ is explored. 
Let $\calE_1$ be the event that the estimate of the loss of each arm in $\calP$ over a duration of length $D$ has an additive error at most $\sqrt{D}\cdot\polylog(nT)$. 
Similarly, let $\calE_2$ denote the event that the estimate of the loss of each meta-expert $\calA_j$ has an additive error at most $\sqrt{T}\cdot\polylog(nT)$. 
Let $\calE$ denote the event that both $\calE_1$ and $\calE_2$ occur. 
By the condition of $k\le\polylog(nT)$, there are at most $\polylog(nT)$ arms in $\calP$.  
As such, $\calE$ holds with high probability, as $\calE_1$ and $\calE_2$ both hold with high probability by \lemref{lem:two:arms:acc} and \lemref{lem:two:meta:acc}. 
Then by \lemref{lem:bl:zero:regret}, the expected regret of $\baseline_0$ is at most $nT^{2/3}\cdot\polylog(nT)$. 
Hence, the base case is complete. 

We now generalize the analysis of \lemref{lem:regret:partial} to handle the regret of $\baseline_k$ for $k>0$. 
Consider $\baseline_k$ and note that there is an instance of $\baseline_{k-1}$ with time horizon $T$, i.e., $\calA_1$.  
For each epoch $\tau \in\left[\frac{T}{B}\right]$, let $i_t^*$ be the best expert of pool $\calP^{(\tau)}$ during epoch $\tau$ for $\calA_1$, as in Section~\ref{sec:PR-alg-est-eviction}.  
Let $\calH$ be the set of bad epochs for $\calA_1$. 
Then with high probability, we have as before,
\begin{align*}
\sum_{\tau=1}^{T/B}\sum_{b=1}^B\ell^{(\tau-1)B+b}(i_{(\tau-1)B+b})&-\sum_{\tau=1}^{T}\ell^{\tau}(i^*)=\sum_{\tau=1}^{T/B}\sum_{b=1}^B\ell^{(\tau-1)B+b}(i_{(\tau-1)B+b})-\sum_{\tau=1}^{T/B}\calL^\tau(i^*)\\
&=\sum_{\tau=1}^{T/B}\sum_{b=1}^B\ell^{(\tau-1)B+b}(i_{(\tau-1)B+b})-\sum_{\tau=1}^{T/B}\calL^\tau(i_t^*)+\sum_{\tau=1}^{T/B}\calL^\tau(i_t^*)-\sum_{\tau=1}^{T/B}\calL^\tau(i^*)\\
&\le\frac{T}{B}\cdot\left(\sqrt{T}\log(nT)\right)+\sum_{\tau=1}^{T/B}\calL^\tau(i_t^*)-\sum_{\tau=1}^{T/B}\calL^\tau(i^*),
\end{align*}
due to the expected regret of MWU, conditioned on $\calE$ and the fact that the exploration rate $\gamma$ is set to zero. 
In particular, we remark that MWU will achieve expected regret $\sqrt{T}\cdot\polylog(nT)$ with respect to $\widetilde{\calL(\calA_1)}$ and $\widetilde{\calL(\calA_2)}$, which are within an additive $\sqrt{T}\cdot\polylog(nT)$ error of $\calL(\calA_1)$ and $\calL(\calA_2)$, conditoined on $\calE$. 

As before, we decompose the epochs $\tau\in\left[\frac{T}{B}\right]$ into the bad epochs $\tau\in\calH$ and the good epochs $\tau\notin\calH$. 
\begin{align*}
\sum_{\tau=1}^{T/B}\calL^\tau(i_t^*)-\sum_{\tau=1}^{T/B}\calL^\tau(i^*)&\le\sum_{\tau\in\calH}\left(\calL^\tau(i_t^*)-\calL^\tau(i_t^*)\right)+\sum_{\tau\notin\calH}\left(\calL^\tau(i_t^*)-\calL^\tau(i_t^*)\right).
\end{align*}
Note that by the inductive hypothesis, $\calA_2$ has expected regret at most $n\cdot B^{2^{k+1}/(2^{k+2}-1)}\cdot\polylog(nT)$. 
Hence by the guarantees of EXP3 on the two experts $\calA_1$ and $\calA_2$, we have 
\begin{align*}
\sum_{\tau=1}^{T/B}\calL^\tau(i_t^*)-\sum_{\tau=1}^{T/B}\calL^\tau(i^*)&\le n\cdot B^{2^{k}/(2^{k+1}-1)}\cdot\polylog(nT)\cdot|\calH|+\frac{T}{B}\cdot\left(\sqrt{2B}\log(nT)\right).
\end{align*}
By \lemref{lem:bad:epochs}, the number of bad epochs satisfies $|\calH|\le\O{n\log^{10}(nT)}$ with high probability, in which case the above expression is minimized at 
\[B=\left(\frac{1}{n}\right)^{\frac{2^{k+1}-1}{2^{k}}}\cdot T^{1-\frac{1}{2^{k+2}-1}}\]
and results in regret $n\cdot T^{2^{k+1}/(2^{k+2}-1)}\cdot\polylog(nT)$, as desired. 
\end{proof}

We can now recursively use the boosting step and obtain the regret of $\sqrt{T}\poly(n, \log{T})$ as our following theorem.
\begin{theorem}
\label{thm:regret-space-two-query-boosting}
For $k=\polylog(nT)$, with high probability, the regret is at most $n\sqrt{T}\cdot\polylog(nT)$, and the total space is at most $\polylog(nT)$. 
\end{theorem}
\Cref{thm:regret-space-two-query-boosting} is slightly weaker than \Cref{thm:two-query} on the polynomial dependency on $n$, and we will see in \Cref{sec:monocarpic-boosting} the optimal dependency with a more involved boosting procedure.

\section{Interval Regret with Bounded Memory and Queries}
\label{sec:interval-regret}
We now turn our focus to the sliding-window model and investigate the \emph{interval regret} for online learning with bounded memory and a small number of queries. 
As we have discussed, our strategy is to simulate the algorithm in \cite{Lu0CZWH24} with $\O{\sqrt{n\card{\calI}}} \cdot \polylog(nT)$ regret on all possible interval $\calI$.
We will show in this section that by using our algorithmic idea in \Cref{sec:two-query-alg}, we could similarly obtain $\O{n\sqrt{\card{I}}} \cdot \polylog(nT)$ regret in $\polylog(nT)$ memory and two queries each round. 

We use our two-query regret minimization algorithm in \Cref{sec:two-query-alg} for this section.
Similar to our exposition in \Cref{sec:two-query-alg}, we focus on the algorithm that obtains the optimal dependency on $T$ in this section (i.e., $\sqrt{T}$) since the analysis is more intuitive and easy to understand. 
We defer the (much) more involved algorithm with optimal dependency on $n$ to \Cref{sec:monocarpic-boosting}.

We first give the algorithm is as in \Cref{alg:two-query-interval-regret}.

\FloatBarrier

\begin{algorithm}[!htb]
\caption{Interval regret algorithm with two queries each day}
\label{alg:two-query-interval-regret}
\algorithmicrequire{Time horizon $T$}
\begin{algorithmic}[1]
\State{Let $\Nmeta=\log{nT}$, each $\kappa\in[\Nmeta]$, instantiate an instance $\ALG_{\kappa}$ of \algref{alg:two-query} with time horizon $2^\kappa$.}
\State{Let $\eta_{\kappa}=\frac{1}{\sqrt{n \cdot 2^\kappa}}$ be the ``learning rate'' for $\ALG_{\kappa}$.}
\State{Maintain $w_t(\kappa)$ as the weight for each algorithm $\ALG_{\kappa}$ on days $t\in [T]$; initialize with $w_1 = \eta_{\kappa}$.}
\State{Let $v^{t,\kappa}$ be the probability distribution over the arms on day $t$ for $\ALG_{\kappa}$.}
\For{each day $t\in [T]$}
\State{Run the exploitation algorithm to play $i_t$ using \Cref{alg:interval-regret-exploitation} and obtain $q^t$.}
\NoNumber{\Comment{This step could potentially incur losses.}}
\State{Run the exploration algorithm to update interval algorithms using \Cref{alg:interval-regret-exploration}.} 
\NoNumber{\Comment{This step has no loss.}}
\State{Let $(j_t, p_t)$ be the arm and probability returned by \Cref{alg:interval-regret-exploration}.}
\State{Let $\widetilde{\ell^t}(j_t) = \card{\calP_t} \cdot \ell^t(j_t)$ and $\widetilde{\ell^t}(i)=0$ for all $i\neq j_t$.}
\For{each $\kappa\in[\Nmeta]$}
\State{Let $r_t(\kappa)=\widetilde{\ell^t}(j_t)\cdot (q^t(\kappa)\cdot \sum_\kappa v^{t,\kappa} - v^{t,\kappa})$.}
\State{Update $w_{t+1}(\kappa)$ as follows}
\If{$t+1 \mod 2^\kappa$ is an integer}
\State{$w_{t+1}(\kappa) \gets \eta_\kappa$.}
\Else
\State{$w_{t+1}(\kappa)\gets (1+\eta_\kappa \cdot r_\kappa(t))w_{t}(\kappa)$.}
\EndIf
\EndFor
\EndFor
\end{algorithmic}
\end{algorithm}

\begin{algorithm}
    \caption{Interval Regret: Exploitation}
    \label{alg:interval-regret-exploitation}
    \algorithmicrequire{$\Nmeta$ algorithms $\{\ALG_\kappa\}_{\kappa=1}^{\Nmeta}$; weights $\{w_t(\kappa)\}_{\kappa=1}^{\Nmeta}$}
    \begin{algorithmic}[1]
        \State{Compute $q_{t}(\kappa) = \frac{w_t(\kappa)}{\sum_{\kappa}w_t(\kappa)}$ for all $\kappa\in [\Nmeta]$ as the probability to play $\ALG_{\kappa}$.}
        \State{Sample arm $i_t \sim \sum_{\kappa}q_{t}(\kappa) \cdot v^{t,\kappa}$ and play $i_t$.}
        \State{Return $q_t$ as the distribution over the $k$ interval algorithms.}
    \end{algorithmic}
\end{algorithm}

\begin{algorithm}
    \caption{Interval Regret: Exploration-and-update}
    \label{alg:interval-regret-exploration}
    \algorithmicrequire{$\Nmeta$ algorithms $\{\ALG_\kappa\}_{\kappa=1}^{\Nmeta}$}
    \begin{algorithmic}
        \State{Sample an interval algorithm $\kappa\in [\Nmeta]$ uniformly at random} 
        \State{\textbf{with} probability $\frac{1}{2}$}
        \State{\hspace{1.5em} Sample $j_t$ uniformly at random from the pool $\calP_{\kappa}$ of $\ALG_{\kappa}$ as in \algref{alg:two-query}.} 
        \State{\hspace{1.5em} Increase $\widetilde{\calL(j_t)}$ by ${2 \Nmeta\cdot \card{\calP_{\kappa}}}\cdot \ell^{t}(i)$.}
        \State{\hspace{1.5em} Update baseline algorithms in $\ALG_{\kappa}$ as in \algref{alg:two-query}.}
        \State{\hspace{1.5em} Let $p_t(j_t)=\frac{1}{2 \Nmeta\cdot \card{\calP_{\kappa}}}$}
        \State{\textbf{otherwise}}
        \State{\hspace{1.5em} Sample an algorithm $\calA_{k}$ from $\ALG_{\kappa}$ as in \algref{alg:two-query}.} 
        \State{\hspace{1.5em} Le $K$ be the total number of baseline algorithms in $\ALG_{\kappa}$ ($K=\polylog(nT)$ as in \Cref{thm:regret-space-two-query-boosting}).}
        \State{\hspace{1.5em} Observe the arm $j_t$ selected by $\calA_{k}$ and update $\widetilde{\calL(\calA_k)}$ by $2\Nmeta\cdot K\cdot \ell^{t}(j_t)$.}
        \State{\hspace{1.5em} Let $p_t(j_t)=\frac{1}{2\Nmeta\cdot K}$}
        \State{Return the sampled arm $j_t$ and the probability $p_t(j_t)$.}
    \end{algorithmic}
\end{algorithm}

\FloatBarrier

To analyze the regret of \Cref{alg:two-query-interval-regret}, we first establish the properties for any of $\ALG_{\kappa}$ (which is a copy of our \algref{alg:two-query}).
This is not entirely a black-box argument from \Cref{thm:regret-space-two-query-boosting} since the rule of playing arms has now changed. 
However, since we sample uniformly at random from the union of the pools at each time, and there are only $\Nmeta = \O{\log T}$ algorithms, which means the properties of \algref{alg:two-query} does \emph{not} change significantly.
\begin{lemma}
    \label{lem:two-query-alg-invertal}
    With high probability, each algorithm of $\ALG_{\kappa}$ for $\kappa \in [\Nmeta]$ uses a total space of at most $\polylog(nT)$ and has an expected regret of at most $n\sqrt{T}\polylog(nT)$.
\end{lemma}
\begin{proof}
    We first claim the guarantees for \lemref{lem:two:arms:acc} and \lemref{lem:two:meta:acc} continue to hold with probability at least $1-1/\poly(nT)$ (with slightly bigger $\polylog(nT)$ factors). 
    For the convenience of understanding, we recap the statements of the lemma.
    \begin{itemize}
        \item \lemref{lem:two:arms:acc}: For any duration $\calD$ of length $D$, there is $\left\lvert\widetilde{\calL(i)}-\sum_{t\in\calD}\ell^t(i)\right\rvert\le\sqrt{D}\cdot\polylog(nT)$ with probability at least $1-1/\poly(nT)$. 
        \item \lemref{lem:two:meta:acc}: Let $\widetilde{\calL(\calA_i)}$ and $\calL(\calA_i)$ be the estimated and actual losses of a baseline algorithm $\calA_i$ in \algref{alg:two-query}, and let $B$ be the epoch length of $\calA_{i}$.
        Then, for any $j\in \{i, i+1\}$, 
        there is $\left\lvert\widetilde{\calL(\calA_j)}-\calL(\calA_j)\right\rvert\le\sqrt{B}\cdot\polylog(nT)$ with probability at least $1-1/\poly(nT)$. 
    \end{itemize}
    We argue why \lemref{lem:two:arms:acc} holds, and the reasoning for \lemref{lem:two:meta:acc} follows from the same argument. 
    Define $\calF_\kappa$ be the probability for $\ALG_\kappa$ to be sampled.
    Since there are at most $\log(T)$ algorithms, we have that 
    \[\Pr\paren{\calF_{\kappa}}= \frac{1}{\log{T}}.\]

    Similar to the proof of \lemref{lem:two:arms:acc},
    let $\calE$ be the vent such that $\card{\calP_{t}}\leq \polylog(nT)$, and we know that this happens with probability at least $1-1/\poly(nT)$. We condition on the high-probability event. As such, we have $\widetilde{\calL(j_t)}={2\card{\calP_{t}}\cdot \log{T}}$ with probability $\frac{1}{2\cdot \card{\calP_{t}}\cdot \log{T}}$ and $0$ otherwise. 
    Therefore, we could similarly obtain
    \[\Ex{\widetilde{\calL^t(i)}}=\ell^t(i),\qquad\Ex{(\widetilde{\calL^t(i)})^2}=(\ell^t(i))^2\cdot\polylog(nT)\]
    with only $\log^2(T)$ difference from the proof in \lemref{lem:two:arms:acc}. As such, by applying Bernstein's inequality, we could again obtain $\left\lvert\widetilde{\calL(i)}-\sum_{t\in\calD}\ell^t(i)\right\rvert\le\sqrt{D}\cdot\polylog(nT)$.

    Observe that the correctness of \Cref{lem:two-query-baseline-k-regret} only requires the correctness of \lemref{lem:regret:partial}, \lemref{lem:bad:epochs}, \lemref{lem:two:arms:acc}, and \lemref{lem:two:meta:acc}. 
    We have proved \lemref{lem:two:arms:acc} and \lemref{lem:two:meta:acc} continue to hold.
    For \lemref{lem:regret:partial} and \lemref{lem:bad:epochs}, we could again fix the randomness for the arm signal without loss. As such, for any $\kappa \in [\Nmeta]$, we could fix the update steps in $\ALG_{\kappa}$ and conduct the same analysis as in \Cref{lem:two-query-baseline-k-regret}. The regret and memory guarantees follow from the correctness of the properties in \Cref{lem:two-query-baseline-k-regret}. 
\end{proof}

In what follows, we use $\expectrand{\textnormal{explore}}{\cdot}$ to denote the expectation with coins only on the exploration (randomness used in \Cref{alg:interval-regret-exploration}). 
Similarly, we use $\expectrand{\textnormal{exploit}}{\cdot}$ to denote the expectation with coins only on the exploitation (randomness used in \Cref{alg:interval-regret-exploitation}).
With \Cref{lem:two-query-alg-invertal}, we could establish the interval regret for each of the interval algorithms $\ALG_{\kappa}$ as follows.
\begin{lemma}
    \label{lem:two-query-alg-regret-interval-inner}
    Let $(i_1, i_2, \cdots, i_{\card{I}})$ be the random variables for the set of arms played by \Cref{alg:two-query-interval-regret} with $I=2^\kappa$. Recall that $v^{t,\kappa}$ is the distribution over the arms of $\ALG_{\kappa}$ on day $t$.
    Furthermore, recall that $i^*$ stands for the best arm and $\widetilde{\ell^t}(i)$ is the estimated cost of each day.
    Then, with probability at least $1-1/\poly(nT)$, we have
    \begin{align*}
        \expectrand{\textnormal{explore}}{\sum_{t\in I}\widetilde{\ell^t}(i_t)\cdot v^{t,\kappa}(i_t) - \sum_{t\in I} \widetilde{\ell^t}(i^*)} \leq n\sqrt{\card{I}} \cdot \polylog(nT),
    \end{align*}
    where the expectation is taken over the randomness of the lossless signals.
\end{lemma}
\begin{proof}
    Conditioning on the high-probability event of \Cref{lem:two-query-alg-invertal}, the regret of each $\ALG_k$ algorithm is $\sqrt{n\card{I}}\cdot \polylog(nT)$, which implies that
    \begin{align*}
        \expectrand{\textnormal{explore}}{\sum_{t\in I} \sum_{i=1}^{n}\widetilde{\ell^t}(i) v^{t,\kappa}(i) - \sum_{t\in I} \ell^t(i^*)}\leq n\sqrt{\card{I}} \cdot \polylog(nT).
    \end{align*}
    Furthermore, by the rules of our algorithm, we have $\widetilde{\ell^t}(i)=0$ if $i\neq i_t$. As such, we have
    \begin{align*}
         \expectrand{\textnormal{explore}}{\sum_{t\in I} \sum_{i=1}^{n}\widetilde{\ell^t}(i) v^{t,\kappa}(i) - \sum_{t\in I} \ell^t(i^*)} = \expectrand{\textnormal{explore}}{\sum_{t\in I}\widetilde{\ell^t}(i_t)\cdot v^{t,\kappa}(i_t) - \sum_{t\in I} \widetilde{\ell^t}(i^*)} \leq n\sqrt{\card{I}} \cdot \polylog(nT),
    \end{align*}
    as desired.
\end{proof}

We now bound the regret induced by the `outer' algorithm for interval regrets. This part is mostly standard following the same argument in \cite{DanielyGS15,Lu0CZWH24}.
\begin{lemma}
    \label{lem:interval-regret-outer-algorithm}
    For any fixed interval $I$, we have
    \begin{align*}
        \expect{\sum_{t\in I}\ell^t(i_t) - \sum_{t\in I}\ell^t(i^*)} = \expectrand{\textnormal{explore}}{\sum_{t\in I} \sum_{i=1}^{n} \widetilde{\ell^{t}}(i)\cdot \paren{\sum_\kappa q^t(\kappa) v^{t,\kappa}(i)} - \sum_{t\in I} \widetilde{\ell^t}(i^*)}.
    \end{align*}
    Furthermore, for any $\kappa$, there is
    \begin{align*}
    \expectrand{\textnormal{explore}}{\sum_{t\in I} \sum_{i=1}^{n} \widetilde{\ell^{t}}(i)\cdot \paren{\sum_\kappa q^t(\kappa) v^{t,\kappa}(i)} -  \sum_{t\in I}\widetilde{\ell^t}(i_t)\cdot v^{t,\kappa}(i_t)}\leq n\sqrt{\card{I}}\cdot \polylog(nT).
    \end{align*}
\end{lemma}
\begin{proof}
    We only give a proof sketch for this lemma and refer keen readers to \cite{Lu0CZWH24} for the full proof.
    The first statement directly follows from the exactly same argument as in \cite{Lu0CZWH24} by the decoupling of randomness. 
    Furthermore, by a similar inductive argument as in \cite{Lu0CZWH24}, we could obtain
    \begin{align*}
        &\expectrand{\textnormal{explore}}{\sum_{t\in I} \sum_{i=1}^{n} \widetilde{\ell^{t}}(i)\cdot \paren{\sum_\kappa q^t(\kappa) v^{t,\kappa}(i)} -  \sum_{t\in I}\widetilde{\ell^t}(i_t)\cdot v^{t,\kappa}(i_t)} \\
        & \leq \eta_{\kappa}\cdot \expectrand{\textnormal{explore}}{\paren{\sum_{t\in I} \sum_{i=1}^{n} \widetilde{\ell^{t}}(i)\cdot \paren{\sum_\kappa q^t(\kappa) v^{t,\kappa}(i)} -  \sum_{t\in I}\widetilde{\ell^t}(i_t)\cdot v^{t,\kappa}(i_t)}^2} + \frac{2\log(nT)}{\eta_{\kappa}}.
    \end{align*}
    Furthermore, we could bound the term on the right-hand side as follows
    \begin{align*}
        & \expectrand{\textnormal{explore}}{\paren{\sum_{t\in I} \sum_{i=1}^{n} \widetilde{\ell^{t}}(i)\cdot \paren{\sum_\kappa q^t(\kappa) v^{t,\kappa}(i)} -  \sum_{t\in I}\widetilde{\ell^t}(i_t)\cdot v^{t,\kappa}(i_t)}^2}\\
        &= \expectrand{\textnormal{explore}, <t}{\sum_{i=1}^{n} p_t(i)\cdot (\frac{\ell^t(i)}{p_t(i)})^2\cdot \paren{\paren{\sum_\kappa q^t(\kappa) v^{t,\kappa}(i)}-v^{t,\kappa}(i)}^2}\\
        &\leq \expectrand{\textnormal{explore}, <t}{\sum_{i=1}^{n} p_t(i)\cdot (\frac{\ell^t(i)}{p_t(i)})^2\cdot \max_{\kappa}\paren{v^{t,\kappa}(i)}^2}\\
        &\leq n\polylog(nT). \tag{by $v^{t,\kappa}(i)\leq 1$ and $p_t(i)\geq 1/\polylog(nT)$}
    \end{align*}
    This step is notably simpler than the analysis in \cite{Lu0CZWH24} since we have $p_t(i)\geq 1/\polylog(nT)$ (in \cite{Lu0CZWH24} if we sample uniformly at random we get a $n^2$ factor, which was the reason they proceeded with a much more complicated distribution).
    Therefore, we obtain that 
    \begin{align*}
        \expectrand{\textnormal{explore}}{\sum_{t\in I} \sum_{i=1}^{n} \widetilde{\ell^{t}}(i)\cdot \paren{\sum_\kappa q^t(\kappa) v^{t,\kappa}(i)} -  \sum_{t\in I}\widetilde{\ell^t}(i_t)\cdot v^{t,\kappa}(i_t)} \leq \paren{\eta_{\kappa}\cdot n\card{I} +\frac{1}{\eta_{\kappa}}}\cdot \polylog(nT),
    \end{align*}
    which gives the desired regret bound by using $\eta_{\kappa}=\O{\frac{1}{\sqrt{n\sqrt{I}}}}$, which is consistent with \Cref{alg:two-query-interval-regret}.
\end{proof}
Combining \Cref{lem:two-query-alg-regret-interval-inner} and \Cref{lem:interval-regret-outer-algorithm} gives us the desired regret bound of $(n\sqrt{\card{I}})\polylog(nT)$ and $\polylog(nT)$ memory with high probability.

\begin{theorem}
    \label{thm:interval-regret-low-space}
    With high probability, \Cref{alg:two-query-interval-regret} achieves $(n \sqrt{\card{I}})\polylog(nT)$ expected regret using $\polylog(nT)$ bits of memory.
\end{theorem}

\Cref{thm:interval-regret-low-space} is slightly weaker than \Cref{thm:two-query-interval} on the polynomial dependency on $n$, and we will see in \Cref{sec:monocarpic-boosting} the optimal dependency with a more involved boosting procedure.

\section{An Algorithm with \texorpdfstring{$T^{2/3}$}{T2over3} Regret in the Bandit Setting}
\label{sec:alg-bandits-full}
We move our attention to the single-query bandit setting in this section.
The crucial component of the algorithm in \Cref{sec:two-query-alg} is the ``free'' estimation of the loss of arms. 
We observe that such a process could be possibly simulated by setting $\gamma>0$, albeit with some loss on the regret. 
In this section, we will eventually show that this idea could lead to an algorithm with $n T^{2/3}\polylog(nT))$ regret.

To properly introduce the algorithm, we first introduce a variant of the covering as a generalization of \Cref{def:cover}.
\begin{definition}[Relaxed approximate cover of an expert]
\label{def:relaxed-cover}
Given a set $\calF$ of filter arms and an exploration parameter $\gamma_{arm}\in(0,1/2)$.  
We say that an arm $i$ with lifespan $\calD$ is \emph{covered} if after running \algref{alg:strong-benchmark} with $i$ and $\calF$ to get $\widetilde{\calL^{\calD}}(\benchmark)$, and we have that
\begin{align*}
    \widetilde{\calL^{\calD}}(i)\ge\widetilde{\calL^{\calD}}(\benchmark) - C \cdot\log(nT)\cdot \sqrt{\frac{\card{\calD}}{\gamma_{arm}}},
\end{align*}
where $C$ is an absolute constant. 
\end{definition}

Compared to the original definition in \Cref{def:cover}, our new eviction rule in \Cref{def:relaxed-cover} is notably more \emph{relaxed}: in \Cref{def:cover}, the parameter we picked after balancing is around $\card{\calD}^{2/3}$. The new definition in \Cref{def:relaxed-cover} allows a much bigger margin since $T\ge \card{\calD}$. Our algorithm is described in \algref{alg:k:for:boost}. 

\begin{algorithm}[!htb]
\caption{Baseline algorithm $\baseline_0$}
\alglab{alg:base:for:boost}
\algorithmicrequire{Time horizon $T$, estimated losses $\widetilde{\calL_i}$ for all arms $i\in\calP$ over their duration in pool $\calP$, epoch length $B_0$}
\begin{algorithmic}[1]
\State{$\calP\gets\emptyset$}
\For{each epoch of length $B_0$}
\State{Sample each arm into $\calP_0$ with probability $\frac{1}{n}$}
\State{Maintain the same hierarchical structure as in Section~\ref{sec:PR-alg-est-eviction}}
\For{each time $t$ in the epoch}
\State{With probability $\gamma_{arm}$}
\State{\hspace{1.5em} Uniformly sample an arm $i\in\calP_0$}
\State{\hspace{1.5em} Pull $i$ and update $\widetilde{\calL_i}\gets\widetilde{\calL_i}+\frac{|\calP_0|}{\gamma_{arm}}\cdot\ell^t(i)$}
\State{Otherwise, pull arm $i$ with probability proportional to $\exp(-\eta\widetilde{\calL_i})$}
\EndFor
\State{Using estimated losses, evict arms covered by the filter}
\EndFor
\end{algorithmic}
\end{algorithm}

\begin{algorithm}[!htb]
\caption{Baseline algorithm $\baseline_k$}
\alglab{alg:k:for:boost}
\algorithmicrequire{Time horizon $T$, estimated losses $\widetilde{\calL_i}$ for all arms $i\in\calP$ over their duration in pool $\calP$, epoch length $B_{k-1}$, $B_k$}
\begin{algorithmic}[1]
\State{$\calP\gets\emptyset$}
\For{each epoch of length $B_k$}
\State{Sample each arm into $\calP_k$ with probability $\frac{1}{n}$}
\State{Maintain the same hierarchical structure as in Section~\ref{sec:PR-alg-est-eviction}}
\For{each time $t$ in the epoch}
\State{Independently, divide the epoch into epochs of length $B_{k-1}$}
\State{At the start of each epoch of length $B_{k-1}$, initialize $\widetilde{\calL(\calA_k)}=0$, $\widetilde{\calL(\calA_{k-1})}=0$}
\State{Initialize an instance of $\baseline_{k-1}$ for each epoch of length $B_{k-1}$ with pool $\calP_{k-1}$}
\State{With probability $\gamma_{arm}$}
\State{\hspace{1.5em} Uniformly sample an arm $i\in\calP_k\cup\calP_{k-1}$}
\State{\hspace{1.5em} Pull $i$ and update $\widetilde{\calL_i}\gets\widetilde{\calL_i}+\frac{|\calP_k\cup\calP_{k-1}|}{\gamma_{arm}}\cdot\ell^t(i)$}
\State{With probability $\gamma_{meta}$}
\State{\hspace{1.5em} Uniformly sample a meta-expert $j\in\{k,k-1\}$}
\State{\hspace{1.5em} Follow $\baseline_j$ and update $\widetilde{\calL(\calA_j)}\gets\widetilde{\calL(A_j)}+\frac{2}{\gamma_{meta}}\cdot\ell^t(i)$}
\State{Otherwise, follow a meta-expert $j\in\{k,k-1\}$ with probability proportional to $\exp(-\eta\widetilde{\calL(\calA_j)})$}
\EndFor
\State{Using estimated losses, evict arms covered by the filter}
\EndFor
\end{algorithmic}
\end{algorithm}

We now define the good and bad epochs with the new covering rules. To this end, we need the argument with the good and bad epochs in the same manner of \Cref{subsec:baseline-regret-analysis}, albeit with a different form of \emph{conceptual cover}. 
\begin{definition}[Relaxed conceptual cover of an expert]
\label{def:relax:concept:cover}
Given a set $\calF$ of filter arms and an exploration parameter $\gamma_{arm}\in(0,1/2)$. 
Let $\calL^{\calD}(\benchmark)$ be constructed in a similar manner as $\widetilde{\calL^{\calD}}(\benchmark)$ but using the loss $\calL$ of each expert, rather than the estimated cost $\widetilde{\calL}$. 
We say that an arm $i$ is \emph{covered} if
\begin{align*}
\calL^{\calD}(i) \geq \calL^{\calD}(\benchmark) - 3C\cdot\log(nT)\cdot \sqrt{\frac{\card{\calD}}{\gamma_{arm}}},
\end{align*}
where $C$ is the same absolute constant as in \Cref{def:relaxed-cover}.
\end{definition}
Essentially, the definitions of \Cref{def:relaxed-cover} and \Cref{def:relax:concept:cover} exactly mirror the case for \Cref{def:cover} and \defref{def:concept:cover} with the changed ``slackness'' accounting for the additive error. Similarly, we define an epoch $E$ as a \emph{good epoch} if the best expert $i^*$ will eventually be conceptually covered and evicted out of the pool if it is sampled in epoch $E$ (using the new covering notion as in \Cref{def:relax:concept:cover}). 
Alternatively, we define an epoch $E$ as \emph{bad epoch} if it is not good, i.e., if $i^*$ not would be evicted by the conceptual eviction rule if it is sampled in epoch $E$.

In what follows, we first analyze the space complexity for the algorithm before analyzing the regret. 
For the space complexity, we essentially argue that \lemref{lem:merge:size} holds for \emph{all} $\baseline_k$ algorithms as long as $k\leq \polylog(nT)$. 
\begin{lemma}
    \lemlab{lem:merge:size:all:baseline}
    Let $\calP^{(k)}_{A}$, $\calP^{(k)}_{B}$, and $\calP^{(k)}_{C}$ be the pools in the \textsc{Merge} algorithm for $\baseline_k$ for $k\leq \log^{10}(nT)$. With high probability, there is
    \begin{align*}
        |\calP^{(k)}_C|\le\max\paren{2\log^9(nT),\frac{1}{4}(|\calP^{(k)}_A|+|\calP^{(k)}_B|)}.
    \end{align*}
\end{lemma}

The proof of \lemref{lem:merge:size:all:baseline} follows largely from the same proof of \lemref{lem:merge:size}, and we omit the details to avoid excessive repetition of the texts.
By \lemref{lem:merge:size:all:baseline}, we could show the memory efficiency of \algref{alg:k:for:boost} as follows.
\begin{lemma}
\lemlab{lem:boosting-pool-mem}
With high probability, the memory used by \algref{alg:k:for:boost} is at most $\O{\log^{30}(nT)}$ words.
\end{lemma}
\begin{proof}
With the same argument we used in \lemref{lem:pool:mem}, we could show that for each $\baseline_k$ algorithm, there are at most $\O{\log^{10}(nT)}$ arms in $\calQ^{(t)}$ for any time $t$, and the loss for each arm can be stored with $\O{(\log^{10}(nT)}$ words. Furthermore, we (deterministically) maintain at most $\log^{10}(nT)$ levels of $\baseline$. This leads to the desired statement of at most $\O{\log^{30}(nT)}$ words of memory.
\end{proof}

\begin{lemma}
\lemlab{lem:t23:arms:acc}
Consider \algref{alg:k:for:boost} and suppose $\ell^t(i)(t)\in[0,1]$ for all $t\in[T]$. 
For any arm $i\in[n]$ that has been in the pool $\calP$ over a duration $\calD$ of length $D$, we have
\[\left\lvert\widetilde{\calL_i}-\sum_{t\in\calD}\ell^t(i)\right\rvert\le\sqrt{\frac{D}{\gamma_{arm}}}\cdot\polylog(nT),\]
with high probability.
\end{lemma}
\begin{proof}
For any $t \in \calD$, define $\widetilde{\ell^{t}}(i)$ as the estimate corresponding to time $t$ that contributes to $\widetilde{\calL_i}$. 
Let $\calP_t$ represent the state of $\calP$ at time $t$. 
Consider the event $\calE$ where $|\calP_t|\le\polylog(nT)$ holds for all $t\in[T]$, with probability at least $1-\frac{1}{\poly(nT)}$ (\lemref{lem:t23:arms:acc}). 

Specifically, we have:
\[
\widetilde{\ell^{t}}(i)=
\begin{cases}
\frac{|\calP_t|}{2\gamma_{arm}}\cdot\ell^t(i) & \text{with probability }\gamma_{arm}\cdot\frac{2}{|\calP_t|}, \\[6pt]
0 & \text{with probability }1-\gamma_{arm}\cdot\frac{2}{|\calP_t|}.
\end{cases}
\]
Conditioned on $\calE$, it follows that:
\[\Ex{\widetilde{\calL_i}(t)}=\ell^t(i), \qquad 
\Ex{\left(\widetilde{\calL_i}(t)\right)^2}=\frac{\calL_i(t)^2}{\gamma_{arm}}\cdot\polylog(n).\]
Applying Bernstein's inequality, c.f. \Cref{thm:bernstein}, we obtain that with high probability:
\[\left\lvert\widetilde{\calL_i}-\sum_{t\in\calD} \ell^t(i)\right\rvert\le\sqrt{\frac{D}{\gamma_{arm}}}\cdot\polylog(nT),\]
provided that $\ell^t(i)\in[0,1]$ for all $t\in[T]$.
\end{proof}

\begin{lemma}
\lemlab{lem:t23:meta:acc}
Consider \algref{alg:k:for:boost} and suppose $\ell^t(i)\in[0,1]$ for all $t\in[T]$.
For a fixed $t\in[T]$ and each $j\in\{1,2\}$, let $L(\calA_j)$ be the loss of $\calA_j$ up to and including time $t$, and let $\widetilde{\calL}(\calA_j)$ be the estimate. 
Then
\[\left\lvert\widetilde{\calL}(\calA_j)-\calL(\calA_j)\right\rvert\le\sqrt{\frac{B_k}{\gamma_{meta}}}\cdot\polylog(nT),\]
with high probability.
\end{lemma}
\begin{proof}
Consider a fixed epoch of length $B_k$. 
Let $t\in[T]$ be given, and let $\calP_t$ represent the state of $\calP$ at time $t$.  
Define $\widetilde{\calL}(\calA_j,t)$ as the contribution of the estimate at time $t$ toward $\widetilde{\calL}(\calA_j)$, and let $L(\calA_j,t)$ denote the actual loss of $\calA_j$ at time $t$.  
Let $\calE$ be the event that $|\calP_t| \le \polylog(nT)$ holds for all $t \in [T]$, ensuring that $\PPr{\calE}\ge 1-\frac{1}{\poly(nT)}$ (\lemref{lem:t23:arms:acc}).  

We observe that $\widetilde{\calL}(\calA_j,t)=\frac{2}{\gamma_{meta}}\cdot L(\calA_j,t)$ with probability $\frac{\gamma_{meta}}{2}$, and $\widetilde{\calL}(\calA_j,t)=0$ with probability $1-\frac{\gamma_{meta}}{2}$.  
Conditioned on $\calE$, it follows that  
\[
\Ex{\widetilde{\calL}(\calA_j,t)}=\calL(\calA_j,t),\qquad \Ex{(\widetilde{\calL_i(t)})^2}\le\frac{4}{\gamma_{meta}},
\]  
assuming $\ell^t(i)\in[0,1]$ for all $t\in[T]$.  
Therefore, applying Bernstein’s inequality (cf. \Cref{thm:bernstein}), we obtain  
\[
\PPr{\left|\widetilde{\calL_i}-\sum_{t\in\calD} \ell^t(i)\right|\ge\sqrt{\frac{B_k}{\gamma_{meta}}} \cdot \polylog(nT)}\le\frac{1}{\poly(nT)},
\]  
as required.
\end{proof}

\begin{lemma}
\lemlab{lem:bl:blackbox:regret}
For any $k$, the regret of $\baseline_k$ on a good epoch, as defined in \Cref{def:relax:concept:cover}, with high probability, is at most
\begin{align*}
\left(\gamma_{meta}\cdot B_k+\gamma_{arm}\cdot B_k + \sqrt{\frac{B_k}{\gamma_{arm}}}+\sqrt{\frac{B_k}{\gamma_{meta}}}\right)\cdot\polylog(nT).
\end{align*}
\end{lemma}
\begin{proof}
In expectation, we have $\gamma_{arm}\cdot B_k+\gamma_{meta}\cdot B_k$ times of exploration in an epoch of length $B_k$ that each incur cost at most $1$. 
We upper bound the regret on those days simply by 
\[\gamma_{arm}\cdot B_k+\gamma_{meta}\cdot B_k.\]
It remains to consider the losses on days where a meta-expert $\{\calA_{k-1},\calA_k\}$ is selected and followed.

In a good epoch, there is an expert $i$ in $\calP$ that covers the best expert $i^*$. 
Hence, $\widetilde{\calL^{B_k}(i)}\le\widetilde{\calL^{B_k}(i^*)}+C\cdot \log(nT)\cdot\sqrt{\frac{B_k}{\gamma_{arm}}}$. 
By \lemref{lem:t23:arms:acc}, we have that with high probability, the estimates of the losses of all arms $i\in\calP$ are within additive error $\sqrt{\frac{D(i)}{\gamma_{arm}}}\cdot\polylog(nT)$. 
Therefore, we have $\calL^{B_k}(i)\le\calL^{B_k}(i^*)+3C\cdot n\log(nT)\cdot\sqrt{\frac{B_k}{\gamma_{arm}}}$. 
By the correctness of MWU (EXP3) on the loss sequence $\widetilde{\calL(i)}$ (\Cref{lem:EXP3-with-approx-loss}), the cost of $\calA_k$ is at most 
\[\sqrt{\frac{B_k}{\gamma_{arm}}}\cdot\polylog(nT)+3C\cdot \log(nT)\cdot\sqrt{\frac{B_k}{\gamma_{arm}}}+\min_{i\in\calP_{k-1}}\calL^{B_k}(i)\] 
by \Cref{lem:EXP3-with-approx-loss}.

Similarly, by \lemref{lem:t23:meta:acc}, the estimates of the losses of all meta-experts are within additive error $\sqrt{\frac{B_k}{\gamma_{meta}}}\cdot\polylog(nT)$. 
By the correctness of MWU (EXP3) on the meta-experts $\calA_k$ and $\calA_{k-1}$ (\Cref{lem:EXP3-with-approx-loss}), we have that the loss of the algorithm on exploitation days is at most 
\[\sqrt{\frac{B_k}{\gamma_{meta}}}\cdot\polylog(nT)+\sqrt{\frac{B_k}{\gamma_{arm}}}\cdot\polylog(nT)+\min_{i\in\calP_{k-1}}\calL^{B_k}(i)+3C\cdot \log(nT)\cdot\sqrt{\frac{B_k}{\gamma_{arm}}}.\]

Therefore, with high probability, the overall regret is at most 
\[\left(\gamma_{meta}\cdot B_k+\gamma_{arm}\cdot B_k + \sqrt{\frac{B_k}{\gamma_{arm}}}+\sqrt{\frac{B_k}{\gamma_{meta}}}\right)\cdot\polylog(nT),\]
as desired.
\end{proof}

The following proof follows similarly from \lemref{lem:bad:epochs}. 

\begin{lemma}
\lemlab{lem:gen:bad:epochs}
Let $\calH$ be the set of bad epochs by \algref{alg:k:for:boost}. 
With high probability, we have $|\calH|\le n\cdot\polylog(n)$.
\end{lemma}

We can now establish our main regret bound for the $k$-the recursive step.
\begin{lemma}
\lemlab{lem:gen:regret:add}
Let $R(B_k)$ be the regret of $\baseline_{k-1}$ on an epoch of length $B_k$. 
Let $\calE$ be the event that \lemref{lem:t23:arms:acc} and \lemref{lem:t23:meta:acc} hold. 
Conditioned on the event $\calE$, we have with high probability, the regret of \algref{alg:k:for:boost} is at most
\begin{align*}
\frac{T}{B_k}\cdot\left(\gamma_{meta}\cdot B_k+\gamma_{arm}\cdot B_k + n\cdot \sqrt{\frac{B_k}{\gamma_{arm}}}+\sqrt{\frac{B_k}{\gamma_{meta}}}\right)\cdot\polylog(nT)+n\cdot R(B_k)\cdot\polylog(nT).
\end{align*}
\end{lemma}
\begin{proof}
Observe that each epoch $i\in\left[\frac{T}{B_k}\right]$ is either a good epoch or a bad epoch. 
Moreover, given the fixing of the exploration times, as well as the fixing of the randomness for the pool sampling and filtration processes, the classification of an epoch into a good epoch or a bad epoch is well-defined. 
By \lemref{lem:bl:blackbox:regret}, each good epoch has regret at most $\left(\gamma_{meta}\cdot B_k+\gamma_{arm}\cdot B_k + \sqrt{\frac{B_k}{\gamma_{arm}}}+\sqrt{\frac{B_k}{\gamma_{meta}}}\right)\cdot\polylog(nT)$. 
On the other hand, each bad epoch can have regret up to $R(B_k)$. 
Fortunately, by \lemref{lem:gen:bad:epochs}, there are at most $n\cdot\polylog(nT)$ bad epochs with high probability. 
Therefore, the desired claim follows. 
\end{proof}

\begin{lemma}
\label{lem:boosting-one-query-main}
Let $\gamma_{arm}=\gamma_{meta}=\frac{1}{T^{1/3}}$. 
Let $F(k)=\frac{7\left(3 \cdot 7^{k} - 3^{k} \right)}{3 \cdot 7^{k+1} - 3^{k+1}}$ and let $G(k)=\frac{2\cdot 7^{k+1}}{3\cdot 7^{k+1}-3^{k+1}}$ and $B_k=n^{-\frac{1}{G(k-1)}}\cdot T^{F(k)}\cdot\polylog(nT)$. 
Then with high probability, the regret of \algref{alg:k:for:boost} for $B_k=T^{F(k)}$ is at most $nT^{G(k)}\cdot\polylog(nT)$. 
\end{lemma}
\begin{proof}
By \lemref{lem:gen:regret:add}, the regret of $\baseline_k$ is 
\[\frac{T}{B_k}\cdot\left(\gamma_{meta}\cdot B_k+\gamma_{arm}\cdot B_k + \sqrt{\frac{B_k}{\gamma_{arm}}}+\sqrt{\frac{B_k}{\gamma_{meta}}}\right)\cdot\polylog(nT)+n\cdot R(B_k)\cdot\polylog(nT).\]
We set $\gamma_{arm}=\gamma_{meta}=\frac{1}{T^{1/3}}$. 
Hence, the regret of $\baseline_k$ is 
\[T^{2/3}\cdot\polylog(nT)+T^{7/6}B_k^{-1/2}\cdot\polylog(nT)+n\cdot R(B_k)\cdot\polylog(nT).\]
For $k=0$, we have $R(B_k)=B_k$ and $F(0)=\frac{7}{9}$. 
Thus for $B_0=n^{-\frac{1}{G(-1)}}\cdot T^{F(0)}\polylog(nT)=T^{7/9}$, the regret is $nT^{7/9}\cdot\polylog(nT)=nT^{G(k)}\cdot\polylog(nT)$, which completes our base case.

Now, suppose the regret of $\baseline_{k-1}$ is $nT^{G(k-1)}\cdot\polylog(nT)$ and specifically, for an epoch of length $B_k$, the regret is $R(B_k)=n{B_k}^{G(k-1)}\cdot\polylog(nT)$. 
We set $B_k=n^{-\frac{1}{G(k-1)}}\cdot T^{F(k)}\cdot\polylog(nT)$, so that the regret of $\baseline_k$ is 
\[T^{2/3}\cdot\polylog(nT)+T^{7/6}B_k^{-1/2}\cdot\polylog(nT)+n\cdot T^{F(k)\cdot G(k-1)}\cdot\polylog(nT).\]
Note that $F(k)\cdot G(k-1)=G(k)$. 
Therefore, the regret of $\baseline_k$ is $n T^{G(k)}\cdot\polylog(nT)$, which completes our induction. 
\end{proof}

Combining \lemref{lem:boosting-pool-mem} and \Cref{lem:boosting-one-query-main} by letting $k=\polylog(nT)$ would immediately lead to our desired result in \Cref{thm:one-query}.

\Ttwooverthreeonequery*

\section{A Near-Optimal Regret Algorithm with Single-Query Signals and Random-Order Best Expert}
\label{sec:random-order-best-arm}
We now discuss the single-query algorithm that achieves near-optimal regret when the loss sequence of the best expert is random order, i.e., the algorithm and proof of \Cref{thm:one-query-best-random}.
This result could also be viewed as a near-optimal algorithm for \emph{adversarial bandits} with very mild distribution assumptions, i.e., only the best bandit has a random-order loss.

As we have observed in \Cref{sec:alg-bandits-full}, due to the losses incurred during the explorations, it is unclear how to achieve the optimal $\sqrt{T}$ regret for the single-query setting in the streaming model using the sampling-and-eviction framework.
As such, we proceed differently here with a ``binary search'' structure for the optimal loss of the best expert.
We start with $C\sqrt{nT}$ error for $C=1$, and we gradually increase the value of $C$ if no such arm satisfies the desired loss range.
Suppose the error of the best expert is $\gamma \sqrt{nT}$.
Since the best expert is in random order, we could also find an expert within the targeted error rate when our guess is around $\gamma \sqrt{nT}$.
On the other hand, if some expert becomes satisfactory \emph{before} $C$ increases to $\gamma$, it means the expert has a very low loss in some interval that outperforms the best expert. 
We could then account for this ``reverse regret'' in the interval, and even if the expert becomes bad later (which means the algorithm will ditch this expert and continue the search), the amount of ``reverse regret'' is enough for us to amortize the regret analysis to get the optimal.

The algorithm is as in \algref{alg:best-arm-random}. Compared to other algorithms we explored in this paper, the algorithm is also notably much simpler.

\FloatBarrier
\begin{algorithm}[!htb]
\caption{A near-optimal algorithm with bandit signals and random-order best expert}
\alglab{alg:best-arm-random}
\algorithmicrequire{Time horizon $T$}
\begin{algorithmic}[1]
\State{$C\gets 1$, $i\gets 1$,$B_C\gets\frac{100}{C+1}\sqrt{\frac{T}{n}}\log(nT)$}
\For{each time}
\State{Play expert $i$}
\State{Let $D_i$ be the interval since $i$ has been played}
\If{$|D_i|=\frac{1}{\eps^2}\cdot B_C$ for some $\eps\in(0,1]$ and $\frac{T}{|D_i|}\cdot\left(\sum_{t\in D_i}\ell^{t}(i)\right)-C\cdot\sqrt{nT}>\frac{C\eps}{2}\sqrt{nT}$}
\State{$i\gets i+1$}
\If{$i=n+1$}
\State{$C\gets C+1$, $i\gets 1$}
\EndIf
\EndIf
\EndFor
\end{algorithmic}
\end{algorithm}

Observe that the memory efficiency for \algref{alg:best-arm-random} is immediate, as in \Cref{lem:random-order-best-arm-space}.
\begin{lemma}
    \label{lem:random-order-best-arm-space}
    \algref{alg:best-arm-random} uses at most $\O{\log(nT)}$ bits of memory.
\end{lemma}
\begin{proof}
    \algref{alg:best-arm-random} cycle through the experts, and at any time, we only need to keep track of the statistics of a single expert. 
    The statistics only include the number of days, the cumulative loss, and some auxiliary parameters, which could all be recorded by $\O{\log(nT)}$ bits of memory.
\end{proof}

The rest of this section is to prove the regret of \algref{alg:best-arm-random}. To this end, we first introduce the following concentration bound.

\begin{proposition}[Tail bounds for sums of hypergeometric random variables]
\cite{hoeffding1994probability}
\label{thm:hoeffding:hypergeo}
Let $X \sim \text{Hypergeometric}(N, K, n)$ be a hypergeometric random variable with expectation $\Ex{X}=\frac{K}{N}\cdot n$. 
Then, for any $t<\frac{K}{N}$,
\[\PPr{|X-\Ex{X}|\ge tn}\le2\exp(-2t^2n).\]
\end{proposition}

We can then apply \Cref{thm:hoeffding:hypergeo} to obtain the concentration guarantees of the estimated losses of the best expert $i^*$ and the parameter $C$ in \algref{alg:best-arm-random} as in \lemref{lem:random:best:good} and \Cref{cor:random-order-C-bounded}.
\begin{lemma}
\lemlab{lem:random:best:good}
Suppose the best expert $i^*$ achieves $\gamma\cdot\sqrt{nT}$ loss for some $\gamma\in[C,C+1)$, where $C$ is an integer, and suppose the losses are in random order. 
Let $I$ be any interval such that $|I|=\frac{1}{\eps^2}\cdot B_C$ for some $\eps\in(0,1]$.  
Then with high probability,
\[\left\lvert\frac{T}{|I|}\cdot\left(\sum_{t\in I}\ell^{t}(i^*)\right)-\gamma\cdot\sqrt{nT}\right\rvert\le\frac{\gamma\eps}{3}\sqrt{nT}.\]
\end{lemma}
\begin{proof}
Proof holds from standard concentration inequalities for hypergeometric random variables, c.f., \Cref{thm:hoeffding:hypergeo}. 
\end{proof}

\begin{corollary}
\label{cor:random-order-C-bounded}
We have that with high probability, over the course of \algref{alg:best-arm-random}, it always holds that $C\le\gamma+1$. 
\end{corollary}
\begin{proof}
Observe that if $C\ge\gamma$, then by \lemref{lem:random:best:good}, we have that with high probability, 
\[\frac{T}{|I|}\cdot\left(\sum_{t\in I}\ell^{t}(i^*)\right)\le\gamma\cdot\sqrt{nT}\left(1+\frac{\eps}{3}\right)\le C\cdot\sqrt{nT}\left(1+\frac{\eps}{2}\right),\]
and so by a union bound over all $T$, $i^*$ will never be evicted with high probability. 
Hence, it follows that $C\le\gamma+1$. 
\end{proof}

The next lemma bounds the regret on each interval $D_i$. The lemma formalizes the intuition that if a sub-optimal expert survives for a long time for $C<\gamma$, then it must have demonstrated a high amount of ``reverse regret'' that could be charged in the final regret analysis.

\begin{lemma}
\lemlab{lem:random:one:block}
Suppose the best expert $i^*$ achieves $\gamma\cdot\sqrt{nT}$ loss and let $C\le\gamma$ be fixed, so that $B_C=\frac{100}{C+1}\sqrt{\frac{T}{n}}\log(nT)$. 
With high probability, the regret on an interval $D_i$ is 
\[\max\left(B_C,\frac{|D_i|}{T}\left(C\sqrt{nT}\left(1+\frac{2}{3}\sqrt{\frac{B_C}{|D_i|}}\right)-\gamma\sqrt{nT}\cdot\left(1-\frac{1}{3}\sqrt{\frac{B_C}{|D_i|}}\right)\right)\right).\] 
\end{lemma}
\begin{proof}
Let $i^*$ be the best expert. 
Observe that \algref{alg:best-arm-random} monotonically increases the value of $C$, so that the value of $B_C$ is monotonically decreasing. 
Then by \lemref{lem:random:best:good}, we have that with high probability,
\[\left\lvert\frac{T}{|D_i|}\cdot\left(\sum_{t\in D_i}\ell^{t}(i^*)\right)-\gamma\cdot\sqrt{nT}\right\rvert\le\frac{\gamma}{3}\sqrt{\frac{B_C}{|D_i|}}\sqrt{nT},\]
which implies that
\[\left(\sum_{t\in D_i}\ell^{t}(i^*)\right)\ge\frac{|D_i|}{T}\cdot\gamma\sqrt{nT}\cdot\left(1-\frac{1}{3}\sqrt{\frac{B_C}{|D_i|}}\right).\]
If $i$ is never evicted then we have
\[\frac{T}{|D_i|}\cdot\left(\sum_{t\in D_i}\ell^{t}(i)\right)-C\cdot\sqrt{nT}\le\frac{C}{2}\sqrt{\frac{B_C}{|D_i|}}\sqrt{nT}\]
using $\eps = \sqrt{B_C / \card{D_i}}$.
On the other hand, for any interval $D_i$ where $i$ is ultimately evicted, 
\[\frac{T}{|D_i|}\cdot\left(\sum_{t\in D_i}\ell^{t}(i)\right)-C\cdot\sqrt{nT}>\frac{C}{2}\sqrt{\frac{B_C}{|D_i|}}\sqrt{nT}.\]
Now, we did not evict $i$ at the previous time either because $|D_i-1|<B_C$ or because the inequality did not hold. 
In the first case, we have $\left(\sum_{t\in D_i}\ell^{t}(i)\right)\le B_C$ since $\ell^{t}(i)\in \{0,1\}$ for any $t$ and $i$. 
In the second case, we have 
\[\frac{T}{|D_i|}\cdot\left(\sum_{t\in D_i}\ell^{t}(i)\right)-C\cdot\sqrt{nT}\le\frac{2C}{3}\sqrt{\frac{B_C}{|D_i|}}\sqrt{nT},\]
so that
\[\sum_{t\in D_i}\ell^{t}(i)\le\frac{|D_i|}{T}\cdot C\sqrt{nT}\left(1+\frac{2}{3}\sqrt{\frac{B_C}{|D_i|}}\right),\]
Thus we have
\[\sum_{t\in D_i}\ell^{t}(i)-\sum_{t\in D_i}\ell^{t}(i^*)\le\max\left(B_C,\frac{|D_i|}{T}\left(C\sqrt{nT}\left(1+\frac{2}{3}\sqrt{\frac{B_C}{|D_i|}}\right)-\gamma\sqrt{nT}\cdot\left(1-\frac{1}{3}\sqrt{\frac{B_C}{|D_i|}}\right)\right)\right).\]
\end{proof}

We are now ready to give the main regret analysis for \algref{alg:best-arm-random} as follows.
\begin{lemma}
With high probability, \algref{alg:best-arm-random} achieves regret $\O{\sqrt{nT}\log^2(nT)}$. 
\end{lemma}
\begin{proof}
Note that $C$ can only increase up to at most $\gamma+1$. 
Similarly for any fixed value of $C$, $i$ can only cycle from $i=1$ to $i=n$ a single time. 
For a fixed $C$ and a fixed $i$, let $D_{C,i}$ be the continuous interval where arm $i$ is selected. 
By \lemref{lem:random:one:block}, the regret on $D_{C,i}$ is at most 
\[\xi_{C,i}:=\max\left(B_C,\frac{|D_{C,i}|}{T}\left(C\sqrt{nT}\left(1+\frac{2}{3}\sqrt{\frac{B_C}{|D_{C,i}|}}\right)-\gamma\sqrt{nT}\cdot\left(1-\frac{1}{3}\sqrt{\frac{B_C}{|D_{C,i}|}}\right)\right)\right),\]
with high probability. 
Let $\xi_1:=\sum_{C\le\frac{\gamma}{2}}\sum_{i\in[n]}\xi_{C,i}$, let $\xi_2:=\sum_{
C>\frac{\gamma}{2}}^{\flr{\gamma-1}}\sum_{i\in[n]}\xi_{C,i}$ and $\xi_3:=\sum_{
C=\flr{\gamma}}^{\flr{\gamma+1}}\sum_{i\in[n]}\xi_{C,i}$, so that the total regret is at most
\[\sum_{C\in\flr{\gamma+1}}\sum_{i\in[n]}\xi_{C,i}=\xi_1+\xi_2+\xi_3.\]
We upper bound each of the terms $\xi_1$, $\xi_2$, and $\xi_3$ as follows. 

Note that for $C\le\frac{\gamma}{2}$, we have that 
\begin{align*}
\frac{|D_{C,i}|}{T}\left(C\sqrt{nT}\left(1+\frac{2}{3}\sqrt{\frac{B_C}{|D_{C,i}|}}\right)-\gamma\sqrt{nT}\cdot\left(1-\frac{1}{3}\sqrt{\frac{B_C}{|D_{C,i}|}}\right)\right)\le0.
\end{align*}
so that
\[\max\left(B_C,\frac{|D_{C,i}|}{T}\left(C\sqrt{nT}\left(1+\frac{2}{3}\sqrt{\frac{B_C}{|D_{C,i}|}}\right)-\gamma\sqrt{nT}\cdot\left(1-\frac{1}{3}\sqrt{\frac{B_C}{|D_{C,i}|}}\right)\right)\right)=B_C,\]
so that $\xi_{C,i}\le B_C$. 
Thus we have 
\[\xi_1=\sum_{C\le\frac{\gamma}{2}}\sum_{i\in[n]}\xi_{C,i}\le\sum_{C=1}^\infty nB_C\le200\sqrt{nT}\log^2(nT),\]
since $B_C=\frac{100}{C+1}\sqrt{\frac{T}{n}}\log(nT)$.

For $C\in\left(\frac{\gamma}{2},\gamma-1\right]$, we have
\begin{align*}
\frac{|D_{C,i}|}{T}&\left(C\sqrt{nT}\left(1+\frac{2}{3}\sqrt{\frac{B_C}{|D_{C,i}|}}\right)-\gamma\sqrt{nT}\cdot\left(1-\frac{1}{3}\sqrt{\frac{B_C}{|D_{C,i}|}}\right)\right)\\
&\le\frac{|D_{C,i}|}{T}\cdot\sqrt{nT}(C-\gamma)+\frac{|D_{C,i}|}{T}\sqrt{\frac{B_C}{|D_{C,i}|}}\cdot\gamma\sqrt{nT}.
\end{align*}
Let $\Gamma:=\frac{|D_{C,i}|}{T}\sqrt{\frac{B_C}{|D_{C,i}|}}\cdot\gamma\sqrt{nT}$. 
Since $B_C=\frac{100}{C+1}\sqrt{\frac{T}{n}}\log(nT)$, we have
\begin{align*}
\Gamma&=\gamma\sqrt{n\frac{|D_{C,i}|}{T}}\sqrt{B_C}\le\gamma\sqrt{\frac{n|D_{C,i}|}{T}}\sqrt{\frac{200}{\gamma}\sqrt{\frac{T}{n}}\log(nT)}\\
&\le\frac{200\gamma^{1/2}n^{1/4}|D_{C,i}|^{1/2}}{T^{1/4}}\log^{1/2}(nT).
\end{align*}
Since $\gamma\le\sqrt{\frac{T}{n}}$, then
\[\Gamma\le200|D_{C,i}|^{1/2}\log^{1/2}(nT).\]
Now, we have 
\[\xi_{C,i}\le\max\left(B_C,\frac{|D_{C,i}|}{T}\cdot\sqrt{nT}(C-\gamma)+200|D_{C,i}|^{1/2}\log^{1/2}(nT)\right).\]
Observe that since $C-\gamma$ is negative, then $\frac{|D_{C,i}|}{T}\cdot\sqrt{nT}(C-\gamma)+200|D_{C,i}|^{1/2}\log(nT)$ is maximized when $|D_{C,i}|\lesssim\frac{T}{n(C-\gamma)^2}\log(nT)$. 
Therefore,
\[\xi_{C,i}\lesssim\max\left(B_C,\frac{1}{(\gamma-C)}\sqrt{\frac{T}{n}}\log(nT)\right).\]
Hence,
\[\xi_2=\sum_{
C>\frac{\gamma}{2}}^{\flr{\gamma-1}}\sum_{i\in[n]}\xi_{C,i}\lesssim
\sum_{
C>\frac{\gamma}{2}}^{\flr{\gamma-1}}\sum_{i\in[n]}B_C+\frac{1}{(\gamma-C)}\sqrt{\frac{T}{n}}\log(nT)\lesssim\sqrt{nT}\log^2(nT),\]
since $B_C=\frac{100}{C+1}\sqrt{\frac{T}{n}}\log(nT)$ and $\frac{1}{\gamma-C}\leq \O{\frac{1}{\gamma}}\le \O{1}$ (using $C\geq \gamma/3$). 

Finally, for $C\in(\gamma-1,\gamma+1]$, we have 
\[\frac{|D_{C,i}|}{T}\left(C\sqrt{nT}\left(1+\frac{2}{3}\sqrt{\frac{B_C}{|D_{C,i}|}}\right)-\gamma\sqrt{nT}\cdot\left(1-\frac{1}{3}\sqrt{\frac{B_C}{|D_{C,i}|}}\right)\right)\le\frac{|D_{C,i}|}{T}\cdot(12\sqrt{nT})\]
and thus
\[\xi_{C,i}\le\max\left(B_C,\frac{|D_{C,i}|}{T}\cdot(12\sqrt{nT})\right)\le B_C+\frac{|D_{C,i}|}{T}\cdot(12\sqrt{nT}).\]
We have $B_C=\frac{100}{C+1}\sqrt{\frac{T}{n}}\log(nT)$. 
Hence,
\[\xi_3\lesssim\sqrt{nT},\]
so that the total regret is at most 
\[\xi_1+\xi_2+\xi_3\lesssim\sqrt{nT}\log^2(nT).\]
\end{proof}

Combining the above steps leads to our desired theorem statement for the random-order best expert with the single-query signals.

\onequerybestrandom*

\FloatBarrier
\newcommand{\monocarpic}{\ensuremath{\textnormal{\textsc{MonocarpicExpert-Bandit}}}\xspace}
\newcommand{\tauepoch}{\ensuremath{\tau_{\textnormal{epoch}}}\xspace}
\newcommand{\intervalregretmonocarpic}{\ensuremath{\textnormal{\textsc{intervalregret}}}\xspace}
\newcommand{\singleinterval}{\ensuremath{\textnormal{\textsc{SingleInterval}}}\xspace}
\newcommand{\EXP}{\ensuremath{\textnormal{EXP}}\xspace}
\newcommand{\flag}{\ensuremath{\textsc{flag}}\xspace}

\section{Boosting Beyond \texorpdfstring{$\Omega(n)$}{OmegaN} with Two-Query Signals}
\label{sec:monocarpic-boosting}
We now present the algorithm and analysis for the near-optimal boosting for the algorithm with two-query signals.
Our algorithms in \Cref{sec:two-query-alg} and \Cref{sec:interval-regret} already achieved the optimal regret on the exponent of $T$ (resp. $\card{I}$), but not yet optimal on $n$.
In this section, we will show that the involved monocarpic expert boosting in \cite{pr23} also works with \emph{estimated} loss sequences with low variance, which would lead to the following theorem.

\rootIinterval*

In this section, we mainly analyze the algorithm for regret minimization over the $T$ days (\Cref{thm:two-query}). The conversion of this algorithm to the interval regret follows the same analysis as in \Cref{sec:interval-regret}, and we will provide a discussion toward the end of the section.

\begin{restatable}{theorem}{rootTtwoquery}
\label{thm:two-query} 
There exists an online learning algorithm that given any instance of $n$ experts and $T$ days such that $T\geq n$ and the query access of two expert (i.e., playing one expert and querying another without loss), achieves $\sqrt{n T}\cdot \polylog(T)$ regret using $\polylog(nT)$ words of memory with probability at least $1-1/\poly(nT)$. 
\end{restatable}

The algorithm and analysis of these boosting algorithms are quite involved; nevertheless, their correctness mostly follows from the guarantees in \cite{pr23} and the ideas we already discussed in previous sections.

\subsection{The boosting algorithm}
We first give the algorithm for the monocarpic expert boosting similar to \cite{pr23}, albeit with modifications tailored to the partial information setting. 
The algorithm is as \Cref{alg:monocarpic-boosting}.

\FloatBarrier
\begin{algorithm}
\caption{Boosting algorithm with Monocarpic Expert, cf. \cite{pr23}}
\label{alg:monocarpic-boosting}
\begin{algorithmic}[1]
\State{Let $R=\log{T}$ be the total number of threads to maintain}
\State Maintain $\baseline^{+}(r)$ for all $ r \in [R]$ \Comment{$\log{T}$ such copies}
\State $\mathcal{P} \leftarrow \mathcal{P}_{1,\cdot} \cup \dots \cup \mathcal{P}_{R,\cdot}$ \Comment{$\mathcal{P}$ is the combined pool}
\State{Maintain $\widetilde{\calL^{t}}(i)$ as the estimated losses for every $t\in [T]$ and $i\in \calP$.}
\For{$t = 1, 2, \dots, T$}
    \State Run $\baseline^{+}(r, \widetilde{\calL^{t}})$ $(r \in [R])$ \Comment{Pool maintaining; $\widetilde{\calL^{t}}$ is the vector for the estimated loss.}
    \State Run $\monocarpic$ over $\mathcal{P}$. \Comment{Substitution of the EXP3}
    \If{\Cref{alg:interval-regret-monocarpic} enters \Cref{line:skip_w_a_b_update}} \Comment{I.e., \Cref{line:w_a_b_update} and onward for $t$ is \emph{not} executed}
        \State{Update $\widetilde{\calL(i)}$ by uniformly sampling an arm in $\calP$, pulling it with the \emph{lossless} signal, and reweighting with the inverse of the sampling probability.}\label{line:arm-loss-update}
    \EndIf
    \If{the pool contains more than $\log^{C}(nT)$ experts} \Comment{$C$ to be specified in \Cref{lem:arm-num-bound-monocarpic-boosting}}
        \State{Stop and declare ``fail'' for the algorithm.}
    \EndIf
\EndFor
\end{algorithmic}
\end{algorithm}

\Cref{alg:monocarpic-boosting} contains two part: the $\baseline^{+}$ algorithm and the $\monocarpic$ algorithm. Similar to the case of \cite{pr23}, the $\baseline^{+}$ uses the same structure as in our $\baseline$ algorithm (\algref{alg:eviction-with-fake-cost}), but it does \emph{not} run the EXP3 procedure, and instead use the $\monocarpic$ algorithm to handle regret minimization within the epochs.

\begin{algorithm}
\caption{$\baseline^{+}(r, \widetilde{\calL^{t}})$}
\begin{algorithmic}[1]
\State{Parameter: $B_r=\frac{T}{n 2^r-1}$; $T_1 = T$; $T_r = B_{r-1}$}
\For{$s = 1, 2, \dots, T/T_r$} \Comment{s-th restart}
    \State {Initiate pools $\mathcal{P}_{r,\cdot} \leftarrow \emptyset$ and sub-pools $\mathcal{P}_{r,k} \leftarrow \emptyset$ ($k \in [0:K]$)} \Comment{$K=\log{T}$ as in \algref{alg:eviction-with-fake-cost}}
    \For{$\tauepoch = 1, 2, \dots, T_r / B_r$} \Comment{Epoch $\tauepoch$}
        \If{$r$ is the lowest thread with a new epoch}
            \For{$r' = R, \dots, r+1$}
                \State $\mathcal{P}_{r,0} \leftarrow \textsc{Merge}(\mathcal{P}_{r,0} \cup \mathcal{P}_{r',\cdot})$.
                \NoNumber{\Comment{Inherit from higher thread pools and perform \algref{alg:merge-new}}}
            \EndFor
        \EndIf
        \State Sample each arm to $\mathcal{P}_{r,0}$ with probability $1/n$. 
        \If{$\tau=C\cdot 2^{C'}$ for some integer $C, C'$}
        \State{Let $\pw(\tauepoch)$ be the largest integer such that $\tau=C\cdot 2^{\pw(\tauepoch)}$ for some integer $C$.}
        \For{$k = 0, 1, \dots, \pw(\tauepoch)$} \Comment{Same pool updates as in \algref{alg:eviction-with-fake-cost}}
            \State{Update $\calP_{k+1}$ with \algref{alg:merge-new} as the merge of $\calP_{k+1}$ and $\calP_{k}$.}
            \State{Use the loss estimations $\widetilde{\calL^{t}}(i)$ for cover and eviction.}
            \State{$\calP_{k}\gets \emptyset$.}
        \EndFor
        \EndIf
    \EndFor
\EndFor
\end{algorithmic}
\end{algorithm}

\begin{algorithm}
\caption{$\monocarpic$}
\label{alg:monocarpic-regret}
\begin{algorithmic}[1]
\State Initialize $\mathcal{U}_k \leftarrow \emptyset$, $\EXP_k$ ($k \in [K]$) \Comment{$K = \log_2(T)$}
\State EXP $\leftarrow$ $\intervalregretmonocarpic(\EXP_1, \dots, \EXP_K, T, \flag=\text{True})$ 
\NoNumber{\Comment{Interval Regret with true losses}}
\For{$t = 1, 2, \dots, T$}
    \State Add newly activated experts to $\mathcal{U}_1$.
    \State Play the decision sampled from EXP.
    \If{$t=C\cdot 2^{C'}$ for some integer $C, C'$} \Comment{Check for each day as opposed to each epoch}
    \For{$l = 1, 2, \dots, \text{pw}(t)$} \Comment{Update membership}
        \State $\mathcal{U}_{k+1} \leftarrow \mathcal{U}_{k+1} \cup \mathcal{U}_k$, $\mathcal{U}_k \leftarrow \emptyset$.
        \State Remove inactive experts in $\mathcal{U}_{k+1}$.
    \EndFor
    \EndIf
\EndFor
\Procedure{$\EXP_k$}{} \Comment{$k \in [K]$}
    \For{$s = 1, 2, \dots, T/2^{k-1}$} \Comment{s-th restart}
        \State Run $\intervalregretmonocarpic(\mathcal{U}_k, 2^{k-1}, \flag=\text{Fake})$ 
        \NoNumber{\Comment{Interval Regret with fake losses}}
    \EndFor
\EndProcedure
\end{algorithmic}
\end{algorithm}

\begin{algorithm}
\caption{$\intervalregretmonocarpic(\mathcal{U}, T, \flag)$}
\label{alg:interval-regret-monocarpic}
\begin{algorithmic}[1]
\State Initialize $w_{a,b} \leftarrow 1$ over $\singleinterval_{a,b}$ for $a \in [K], b \in [T/2^a]$)
\State{Let $D_{a,b} = [2^a(b-1)+1:2^a b]$ for each $a \in [K], b \in [T/2^a]$}
\For{$t = 1, 2, \dots, T$}
    \State{Run updates $\singleinterval_{a,b}(D_{a,b}, \calU, \widetilde{\calL^{D_{a,b}}})$ on interval $D_{a,b}$.}
    \If{$\flag==$True}
         \State{Sample action $i_{t,a,b}$ from $\{w_{a,b}\}_{h(t,a,b)=1}$.} \Comment{$h(t,a,b)=1$ if $t \in [2^a(b-1)+1 : 2^a b]$}
         \NoNumber{\Comment{This is the only line that uses the \emph{played} query with loss}}
        \State{Let $\widetilde{\ell}(i_{t,a,b})\gets \ell^{t}(i_{t,a,b})$.}
    \Else
        \State{With probability $1/2$:} 
        \State{\hspace{1.5em} Sample $k'\in [K]$ uniformly at random, and sample action $i_{t,a,b}$ by $\EXP_{k'}$.}  
        \NoNumber{\Comment{Use the \emph{second query} without losses}}
        \State{\hspace{1.5em} Observe the loss $\ell(i_{t,a,b})$ as the loss from $\EXP_{k'}$.} 
        \State{\hspace{1.5em} Let $\widetilde{\ell}(i_{t,a,b})\gets \ell(i_{t,a,b})\cdot K$.}
        \If{no $i_{t,a,b}$ is sampled}
            \label{line:skip_w_a_b_update}\State{Continue to day $t+1$ (without making updates to $w_{a,b}$)}
        \EndIf
    \EndIf
    \State \label{line:w_a_b_update}Compute the expected loss 
    \begin{align*}
        \bar{\ell}^t \leftarrow \sum_{a,b:h(t,a,b)=1} \frac{w_{a,b}}{\sum_{a',b':h(t,a',b')=1} w_{a',b'}} \cdot \widetilde{\ell}(i_{t,a,b})
        \end{align*}
    \State Assign loss $\hat{\ell}_t(a,b) = \begin{cases} \ell^t(i_{t,a,b}) & h(t,a,b)=1 \\ \bar{\ell}^t & h(t,a,b)=0 \end{cases}$
    \State Update the weight distribution using SQUINT 
    \[w_{a,b} \leftarrow \expectrand{\eta}{\eta \cdot \exp\paren{\eta \sum_{\tau=1}^{t-1} v_{\tau}(a,b) - \eta^2 \sum_{\tau=1}^{t-1} v_{\tau}^2(a,b)}},\] 
    where $v_{\tau}(a,b) = \bar{\ell}_{\tau} - \hat{\ell}_{\tau}(a,b)$.
\EndFor
\end{algorithmic}
\end{algorithm}

\begin{algorithm}
\caption{$\singleinterval_{a,b}(D_{a,b}, \calU, \widetilde{\calL^{D_{a,b}}})$}
\algorithmicrequire{Duration $D_{a,b}$, pool of arms $\calU$, and $\widetilde{\calL^{D_{a,b}}}(i)$ for $i \in \calU$}
\begin{algorithmic}[1]
\For{$t \in D_{a,b}$}
    \State Run EXP3 with the estimation of losses $\widetilde{\calL^{D_{a,b}}}(i)$ for each $i \in \calU$. 
\EndFor
\end{algorithmic}
\end{algorithm}
\FloatBarrier

Compared to the monocarpic boosting algorithm in \cite{pr23}, the most significant difference of our algorithm lies in the algorithms of $\intervalregretmonocarpic$\footnote{Not to be confused with our interval regret algorithm (w.r.t. the best arm in the interval) in \Cref{sec:interval-regret}.} and $\singleinterval$.  
Since we only have two queries each day, we could no longer perform updates for all the $\EXP_{k}$ algorithms in \Cref{alg:monocarpic-regret}. Instead, we proceed differently by \emph{sampling} an arm uniformly at random each day to \emph{estimate} the loss, and perform the SQUINT algorithm directly on the \emph{estimated losses}.
Since the estimated losses have low variance, we could use a ``partial-to-full'' type of reduction argument to show that the regrets for the $\EXP_{k}$ algorithms for $k\in [K]$ are still low, which leads to the interval guarantees as shown in \cite{pr23}.

\subsection{Technical lemmas and the analysis of the \texorpdfstring{$\monocarpic$}{monocarpicexpert} algorithm}
Before we show the formal proof, we first establish the bound on the estimation error of the loss sequence for any arm in the algorithm.
\begin{lemma}
    \label{lem:monocarpic-arm-loss-estimation}
    Let $i$ be any arm in the pools of \Cref{alg:monocarpic-boosting}. For any given interval $\calI$, with probability at least $1-1/\poly(nT)$, we have that
    \begin{align*}
        \card{\widetilde{\calL^{D}}(i) - \sum_{t\in \calI} \ell^{t}(i)} \leq \sqrt{\card{\calP}\cdot \card{\calI}}\cdot \polylog(nT),
    \end{align*}
    where $\card{\calP}$ is the size of the union of pools.
    Furthermore, if we only maintain $\polylog(nT)$ intervals in parallel, it only takes $\polylog(nT)$ memory.
\end{lemma}
\begin{proof}
    We assume w.log. that $\card{\calI}\geq \polylog(nT)$ since otherwise we trivially have the error bound of $\polylog$ since $\card{\calP}, \card{\calI}\geq 1$.
    On each day, we have at least $1/2$ probability to sample an arm uniformly at random from the pool. As such, we can apply \Cref{prop:two-query-EXP3} (also \Cref{alg:EXP3-exploration-two-query}) to obtain the lemma.
\end{proof}

We proceed differently from \cite{pr23} as we first bound the total number of arms in the pool. The value of this will be evident later.
\begin{lemma}[cf. Lemma 4.12 of \cite{pr23}]
    \label{lem:arm-num-bound-monocarpic-boosting}
    With probability at least $1 - 1/\poly(nT)$, there are at most $\polylog(nT)$ arm in each pool $\mathcal{P}_{r,\cdot}$ for any $r \in [R]$.
    Therefore, with probability at least $1 - 1/\poly(nT)$, the size of the union of the pools is at most $\card{\calP}\leq \log^{C}(nT)$, and the algorithm never declares ``fail''.
\end{lemma}

\begin{proof}
The proof follows the same logic as in \cite{pr23}.
We first show that at any time, with probability at least $1-1/\poly(nT)$, we have $\card{\mathcal{P}_{r,k}}\leq \polylog(nT)$ for any $r\in [R]$ and $k\in [K]$. 
We prove this by an inductive argument on $r$, and we insist on evicting using the covering notion (\Cref{def:cover}) with $\sqrt{D}\polylog(nT)$ error. 
For $r=R$, there is nothing to inherit from higher-level pools. Since we sample each arm with probability $1/n$, initially, our pool size is at most $\O{\log(nT)}$ with probability at least $1-1/\poly(nT)$. 
Furthermore, due to \Cref{lem:monocarpic-arm-loss-estimation} and the induction hypothesis, the estimation error for each expert is at most $\sqrt{\card{\calI}}\cdot \polylog(nT)$. 
Note that for the analysis in \lemref{lem:merge:size} and \lemref{lem:num:experts} to work, we only need the pool size $\card{\calP}$ to be less than $O(n^2)$.
Therefore, the correctness of \lemref{lem:merge:size} (merge algorithm) still holds by the induction hypothesis, and we could use the same proof as in \lemref{lem:merge:size} and \lemref{lem:bl:zero:regret} to obtain that
\begin{align*}
    |\calP^{(k)}_C|\le\max\paren{2\log^9(nT),\frac{1}{4}(|\calP^{(k)}_A|+|\calP^{(k)}_B|)}
\end{align*}
holds for all $k\in [K]$. Therefore, with probability at least $1-1/\poly(nT)$, the pool size of $r=R$ is at most $\polylog(nT)$.

Now, for the purpose of induction, suppose the induction holds for thread $r+1$. For the $r$-th thread, $\mathcal{P}_{r,0}$ is initiated as sampling with probability $1/n$ and inheriting from thread $r$. By the induction hypothesis, there is
\begin{align*}
    |\mathcal{P}_{r',\cdot}| \leq \polylog(nT)
\end{align*}
for any $r'>r$. Therefore, by running the merge algorithm (\algref{alg:merge-new}), for $r$ in the induction and any $k$, we have
\begin{align*}
    \card{\calP_{r, k}} \leq \max\{2 \log^9(nT), \frac{1}{4}(|\mathcal{P}_{r,k}| + \polylog(nT))\}\leq \polylog(nT).
\end{align*}
Finally, since $R=\log(nT)$ and $K=\log(nT)$, we have that $\card{\calP}=\card{\cup_{r}\cup_{k}\calP_{r,k}}\leq \polylog(nT)$, as desired.
\end{proof}

We now formally prove the interval regret guarantees of the $\monocarpic$ algorithm. Here, we will use the fact that the number of arms in the pool is at most $\polylog(nT)$ to bound the variance of estimation.

\begin{lemma}[cf. Lemma 4.11 of \cite{pr23}]
\label{lem:monocarpic-boosting-interval-regret}
Let $T \geq 1$ and $\card{\calU}\leq \polylog(nT)$. For any expert $i \in \mathcal{U}$ and time interval $\mathcal{I} \subseteq [T]$ such that $\card{\calI}\geq \polylog(nT)$, algorithm \intervalregretmonocarpic guarantees that with probability at least $1 - 1/(nT)^{\omega(1)}$
\[
\sum_{t \in \mathcal{I}} \ell^{t}(i_t) - \sum_{t \in \mathcal{I}} \ell^{t}(i) \leq \O{\sqrt{\card{\calI}}}\cdot \polylog(nT)
\]
holds against an adaptive adversary. Moreover, \intervalregretmonocarpic uses up to $\polylog(nT)$ words of memory.
\end{lemma}
\begin{proof}
We first prove the case with $\flag=\text{true}$ since it is essentially the same as in \cite{pr23}.
The following analysis is mostly from \cite{pr23}, and we provide them for the sake of completeness.
Fix any $a \in [L], b \in [T/2^a]$ and expert $i \in \mathcal{U}$.
Since we have a probability of $1/2$ to sample an arm in the pools to estimate the loss of arms, and conditioning on the high-probability event that the total number of experts in the pools is at most $\polylog(nT)$ (\Cref{lem:arm-num-bound-monocarpic-boosting}), we can argue that for any interval $\calI$ with size at least $\polylog(nT)$, we have
\begin{align*}
    \card{\widetilde{\calL}^{\calI}(i)- \sum_{t\in \calI} \ell^{t}(i)} \leq \sqrt{\card{\calI}} \cdot \polylog(nT)
\end{align*}
with probability at least $1-1/\poly(nT)$.
Here, $\widetilde{\calL}^{\calI}(i)$ is the estimated loss of arm $i$ obtained from \Cref{line:arm-loss-update} of \Cref{alg:monocarpic-boosting}.
We could then apply a union bound over $T^2$ intervals and argue that the $\sqrt{\card{\calI}} \cdot \polylog(nT)$ error holds for all intervals.
By the regret guarantees of EXP3 with learning as exploration (\Cref{prop:two-query-EXP3}), with probability at least $1 - 1/\poly(nT)$, we have
\begin{align*}
\sum_{t=2^a(b-1)+1}^{2^a b} \ell^{t}(i_{t,a,b}) - \sum_{t=2^a(b-1)+1}^{2^a b} \ell^{t}(i) \leq \O{\sqrt{2^a}} \cdot \polylog(nT).
\end{align*}
When $\flag=\text{true}$, since we have the probability to sample $i_{t,a,b}$ is $p_{t,a,b} = \frac{w_{t,a,b}}{\sum_{a',b':h(t,a',b')=1} w_{t,a',b'}}$, we have that
\begin{align*}
\bar{\ell^{t}} & = \sum_{a,b:h(t,a,b)=1} \frac{w_{t,a,b}}{\sum_{a',b':h(t,a',b')=1} w_{t,a',b'}} \widetilde{\ell}(i_{t,a,b}) \\
& = \sum_{a,b:h(t,a,b)=1} \frac{w_{t,a,b}}{\sum_{a',b':h(t,a',b')=1} w_{t,a',b'}} \ell(i_{t,a,b}) \tag{since the setting we are using the true loss}\\
& = \sum_{a,b:h(t,a,b)=1} p_{t,a,b} \ell(i_{t,a,b})=\expect{\ell(i_{t})},
\end{align*}
where $i_t$ is the action taken by the outer ($\intervalregretmonocarpic$) algorithm.
Therefore, by the guarantees of the SQUINT algorithm (\Cref{lem:squint}), there is
\begin{align*}
\expect{\sum_{t \in [2^a(b-1)+1 : 2^a b]} \ell^t(i_t) - \sum_{t \in [2^a(b-1)+1 : 2^a b]} \ell^t(i_{t,a,b})}= \O{\sqrt{2^a \log(nT)}}.
\end{align*}
Applying concentration inequlaitities and $|\mathbb{E}[\ell^t(i_t)] - \ell^t(i_t)| \le 2$, one obtains
\begin{equation*}
\sum_{t \in [2^a(b-1)+1 : 2^a b]} \ell^t(i_t) - \sum_{t \in [2^a(b-1)+1 : 2^a b]} \ell^t(i_{t,a,b}) \leq \O{\sqrt{2^a \log(nT)}}
\end{equation*}
holds with probability at least $1 - 1/\poly(nT)$. As such, by combining the inner and outer regrets using triangle inequalities, we have that with probability at least $1 - 1/\poly(nT)$,
\begin{align*}
& \sum_{t \in [2^a(b-1)+1 : 2^a b]} \ell^t(i_t) - \sum_{t \in [2^a(b-1)+1 : 2^a b]} \ell^t(i) \\
& = \sum_{t \in [2^a(b-1)+1 : 2^a b]} \ell^t(i_t) - \sum_{t \in [2^a(b-1)+1 : 2^a b]} \ell^t(i_{t,a,b}) + \sum_{t \in [2^a(b-1)+1 : 2^a b]} \ell^t(i_{t,a,b}) - \sum_{t=2^a(b-1)+1}^{2^a b} \ell^{t}(i) \\
& \leq O\paren{\sqrt{2^a} \cdot \polylog(nT)}.
\end{align*}
We now move to the case when we have $\flag\neq\text{True}$. This part is slightly different from \cite{pr23} as now we are required to perform SQUINT on the fake losses.
We first note that the guarantees on $\sum_{t=2^a(b-1)+1}^{2^a b} \ell^{t}(i_{t,a,b}) - \sum_{t=2^a(b-1)+1}^{2^a b} \ell^{t}(i) \leq \O{\sqrt{2^a}} \cdot \polylog(nT)$ remain unchanged.
Each day, we have $1/2$ probability to sample an arm from the union of the pools. 
There are at most $\polylog(nT)$ pools, and each pool contains at most $\polylog(nT)$ arms (\Cref{lem:arm-num-bound-monocarpic-boosting}), the probability to sample each arm is at least $1/\polylog(nT)$. 
Let $i_{t,a,b}$ be any of the experts sampled by the algorithm, we have 
\begin{align*}
\Ex{\widetilde{\ell^t}(i_{t,a,b})}=\ell^t(i_{t,a,b}),\qquad\Ex{(\widetilde{\ell^t}(i_{t,a,b}))^2}=(\ell^t(i_{t,a,b}))^2\cdot\polylog(nT). 
\end{align*}
Let $\widetilde{\calL^{a,b}}(i):=\sum_{i_{t,a,b}=i}P\cdot \ell^{t}(i)$ be the loss estimation of expert $i$ over the $t$ steps in interval $D_{a,b}$.
For $T\geq \polylog(nT)$, with probability at least $1-1/\poly(nT)$, the loss of each arm in each pool $\calU$ has been sampled at least $\polylog(nT)$ times.
Therefore, we could apply Bernstein's inequality (\Cref{thm:bernstein}), and obtain that with probability at least $1-1/\poly(nT)$,
\begin{align*}
    \Pr\paren{\card{\widetilde{\calL^{a,b}}(i)-\sum_{t\in D_{a,b}}\ell^{t}(i)} \leq \sqrt{\card{D_{a,b}}}\cdot \polylog(nT)} \geq 1-1/\poly(nT).
\end{align*}

We can then apply the union bound to at most $T^2$ consecutive intervals and obtain the error $\sqrt{\card{D_{a,b}}}$ holds for all intervals.
Therefore, we could decompose the regret of $\sum_{t \in [2^a(b-1)+1 : 2^a b]} \ell^t(i_t) - \sum_{t \in [2^a(b-1)+1 : 2^a b]} \ell^t(i_{t,a,b})$ as follows.
\begin{align*}
    & \expect{\sum_{t \in [2^a(b-1)+1 : 2^a b]} \ell^t(i_t) - \sum_{t \in [2^a(b-1)+1 : 2^a b]} \ell^t(i_{t,a,b})}\\
    & = \expect{\sum_{t \in [2^a(b-1)+1 : 2^a b]} \ell^t(i_t) - \sum_{t \in [2^a(b-1)+1 : 2^a b]}  \widetilde{\ell^{t}}(i_{t})} + \expect{\sum_{t \in [2^a(b-1)+1 : 2^a b]}  \widetilde{\ell^{t}}(i_{t}) - \sum_{t \in [2^a(b-1)+1 : 2^a b]}  \widetilde{\ell^{t}}(i_{t,a,b})} \\
    & \quad + \expect{\sum_{t \in [2^a(b-1)+1 : 2^a b]}  \widetilde{\ell^{t}}(i_{t,a,b})) - \sum_{t \in [2^a(b-1)+1 : 2^a b]} \ell^t(i_{t,a,b})}.
\end{align*}
Conditioning on the high-probability event of the bounded approximation error, we could bound the first term as the summation of losses, i.e.
\begin{align*}
    \expect{\sum_{t \in [2^a(b-1)+1 : 2^a b]} \ell^t(i_t) - \sum_{t \in [2^a(b-1)+1 : 2^a b]} \ell^t(i_{t,a,b})}&\leq  \expect{\sum_{t \in [2^a(b-1)+1 : 2^a b]} \sum_{i\in \text{union of pools}} \card{\ell^t(i) - \widetilde{\ell^{t}}(i)}}\\
    &\leq \sqrt{2^a}\cdot \polylog(nT),
\end{align*}
where the last inequality used the number of arms in the pool conditioning on the high-probability event in \Cref{lem:arm-num-bound-monocarpic-boosting}. The third term could similarly be bounded by $\sqrt{2^a}\cdot \polylog(nT)$. 
Finally, for the second term, we note that $\bar{\ell}^t$ remains an unbiased estimator of $\widetilde{\ell^{t}}$, i.e., 
\begin{align*}
\bar{\ell^{t}} & = \sum_{a,b:h(t,a,b)=1} \frac{w_{t,a,b}}{\sum_{a',b':h(t,a',b')=1} w_{t,a',b'}} \widetilde{\ell}(i_{t,a,b}) \\
& = \sum_{a,b:h(t,a,b)=1} p_{t,a,b} \widetilde{\ell}(i_{t,a,b})=\expect{\widetilde{\ell}(i_{t})}.
\end{align*}
Therefore, we could rescale the loss and run SQUINT on the fake losses, which would lead to at most $\polylog(nT)$ blow-up in the regret since $\tilde{\ell^{t}}(i)\leq \polylog(nT)$ with probability at least $1-1/\poly(nT)$ (due to the fact that $K\leq \polylog(nT)$). Hence, we reached to the conclusion that the regret when $\flag\neq \text{True}$ is also $\sqrt{2^a}\polylog(nT)$ on intervals $D_{a,b}$.

\paragraph{Wrapping up the analysis of \Cref{lem:monocarpic-boosting-interval-regret}.}
Same as \cite{pr23}, to conclude the regret analysis, we note for any interval $\mathcal{I} = [t_1 : t_2] \subseteq T$ one can split $\mathcal{I}$ into $X \le 2 \log_2(|\mathcal{I}|)$ disjoint intervals $\mathcal{I} = \mathcal{I}_1 \cup \mathcal{I}_2 \cup \dots \cup \mathcal{I}_X$ such that (1) $\mathcal{I}_x (x \in [X])$ exactly spans the lifetime of some meta expert and (2) there are at most two length-$2^x$ intervals. Then we conclude
\begin{align*}
\sum_{t \in \mathcal{I}} \ell^t(i_t) - \sum_{t \in \mathcal{I}} \ell^t(i) & \leq \sum_{x=1}^{X}\O{\sqrt{|\mathcal{I}_x|}}\polylog(nT) \\
&\le \sum_{x=1}^{\log(|\mathcal{I}|)}\O{\sqrt{2^x}}\polylog(nT) = \O{\sqrt{|\mathcal{I}|}}\polylog(nT).
\end{align*}

For the memory efficiency of $\intervalregretmonocarpic$, we could use the method to maintain weights $\{w_{a,b}\}$ for at most $\log(T)$ meta experts. In addition to the information required to track in \cite{pr23}, we only track the losses of arms in the pool at any time, which gives a $\polylog(nT)$ overhead in the memory. Therefore, by the assumption that $\card{\calU}\leq \polylog(nT)$, the total memory usage is at most $\polylog(nT)$.
\end{proof}

With \Cref{lem:monocarpic-boosting-interval-regret}, we can now conclude the main property of the $\monocarpic$ algorithm. 

\begin{proposition}[cf. Theorem 4.10 in \cite{pr23}]
\label{prop:monocarpic-alg-property}
Let $T \geq 1$. For any expert $i$ that is alive over interval $\mathcal{I} \subseteq [T]$, the $\monocarpic$ algorithm guarantees that with probability at least $1 - 1/\poly(nT)$,
\[
\sum_{t \in \mathcal{I}'} \ell^{t}(i_t) - \sum_{t \in \mathcal{I}'} \ell^t(i) \leq \sqrt{\card{\mathcal{I}'}}\cdot \polylog(nT).
\]
holds for every interval $\mathcal{I}' \subseteq \mathcal{I}$, even if the adversary is adaptive.
Furthermore, let $M$ be the largest number of alive experts at any point, then the $\monocarpic$ algorithm uses up to $M \polylog(nT)$ words of memory.
\end{proposition}
\begin{proof}
The proof is essentially the same as the proof of Theorem 4.10 in \cite{pr23}.
Let interval $\calI$ be the life span of an expert $i$, following the same analysis in \cite{pr23}, we could split the interval $\calI'\subseteq \calI$ to $X=\log(\card{\calI'})$ subsequences with base-2 length. 
Let these intervals be $\mathcal{I}' = \mathcal{I}'_1 \cup \mathcal{I}'_2 \cup \dots \cup \mathcal{I}'_X$, and $i_{t,k}$ be the action $\EXP_{k}$ takes on day $t$. 
Furthermore, let $k(x)$ be the corresponding level of interval $x\in X$. 
By \Cref{lem:monocarpic-boosting-interval-regret}, we could obtain
\begin{align*}
\sum_{t \in \mathcal{I}'} \ell^t(i_t) - \sum_{t \in \mathcal{I}'} \ell^t(i) &= \sum_{x=1}^{X} \sum_{t \in \mathcal{I}'_x} (\ell^t(i_t) - \ell^t(i)) \\
&= \underbrace{\sum_{x=1}^{X} \sum_{t \in \mathcal{I}'_x} (\ell^t(i_t) - \ell^t(i_{t,k(x)}))}_{\text{interval regret over $\EXP_{k}$ algorithms}} + \underbrace{\sum_{x=1}^{X} \sum_{t \in \mathcal{I}'_x} (\ell^t(i_{t,k(x)}) - \ell^t(i))}_{\text{interval regret for a single $\EXP$ algorithm}} \\
& \le \sum_{x=1}^{X} \sqrt{|\mathcal{I}'_x|}\cdot \polylog(nT) \tag{applying \Cref{lem:monocarpic-boosting-interval-regret}}\\
& \le \sqrt{|\mathcal{I}'|}\cdot \polylog(nT). \tag{there are at most $\log(T)$ layers}
\end{align*}

The memory analysis follows from the fact that we only ever main $\log(nT)$ $\EXP$ algorithm, and each of them only takes $\polylog(nT)$ memory as described in \Cref{lem:monocarpic-boosting-interval-regret}. Therefore, each alive expert takes $\polylog(nT)$ memory, which results in $M\cdot \polylog(nT)$ memory for $M$ alive experts.
\end{proof}

\subsection{The analysis of the optimal boosting algorithm (\texorpdfstring{\Cref{alg:monocarpic-boosting}}{alg:monocarpic-boosting})}
We now analyze the memory and the regret of the boosting algorithm as in \Cref{alg:monocarpic-boosting}.

\subsubsection*{Memory analysis}
The memory analysis now becomes very simple following \Cref{lem:arm-num-bound-monocarpic-boosting} and \Cref{prop:monocarpic-alg-property}.
\begin{lemma}
    \label{lem:monocarpic-optimal-boosting-space}
    With probability at least $1-1/\poly(nT)$, \Cref{alg:monocarpic-boosting} uses at most $\polylog(nT)$ words of memory. 
\end{lemma}
\begin{proof}
    By \Cref{lem:arm-num-bound-monocarpic-boosting}, the pool is of size at most $\polylog(nT)$, which is an upper bound for the number of alive experts. Therefore, by \Cref{prop:monocarpic-alg-property}, the total memory used is at most $\polylog(nT)$ with probability at least $1-1/\poly(nT)$.
\end{proof}

\subsubsection*{Regret analysis}
The regret analysis for the monocarpic expert boosting is fairly complicated, as in \cite{pr23}. However, most of the analysis follows directly from \cite{pr23}. We first introduce the notions that are used in \cite{pr23} as follows.

\paragraph{The additional notation used by \cite{pr23}.} 
Let $i^* \in [n]$ be the best expert, and define $\calK_1, \cdots, \calK_R$ as follows.
\[\calK_1 = [0:n-1], \quad \calK_2 = [0:n-1] \times \{0,1\}\]
\[\calK_R = [0:n-1] \times \underbrace{\{0,1\} \times \dots \times \{0,1\}}_{R-1 \text{ times}}.\]
Furthermore, let $K$ be the union of $\calK_1, \dots, \calK_R$, i.e.,
\[\calK = \calK_1 \cup \dots \cup \calK_R.\]
For any timestep $a = (a_1, \dots, a_{r(a)}) \in K$ (where $r(a)$ is defined such that $a \in \calK_{r(a)}$), the timestep $a$ uniquely identifies an epoch of $\baseline^+(r(a))$, i.e., the $a_{r(a)}$-th epoch of the $(\sum_{r=1}^{r(a)-1} a_r 2^{r(a)-r-1})$-th restart of the algorithm.

\begin{definition}[The $\oplus$ Operator, Definition 4.14 of \cite{pr23}]
For any timestep $a \in \calK$, we write
\[ a' = (a'_1, a'_2, \dots, a'_{r(a')}) = a \oplus 1 \in \calK \]
as the unique timestep that satisfies
\[ \sum_{i=1}^{r(a')} a'_i B_i = \sum_{i=1}^{r(a)} a_i B_i + B_{r(a)} \quad \text{and} \quad a'_{r(a')} \ne 0 \quad . \]
That is to say, $a' = a \oplus 1$ is the next number under 2-base with 0 truncated at the end (except the first coordinate, which belongs to $[0:n-1]$).
\end{definition}

With the same way of \cite{pr23}, we let $\calK(a)$ contain all time steps that succeed $a$ under the $\oplus$ operation, i.e.,
\[ \calK(a) := \{a\} \cup \{a \oplus 1\} \cup \{(a \oplus 1) \oplus 1\} \cup \dots \subseteq \calK \]

\paragraph{Random bits and epoch assignment in \cite{pr23}.} 
Below is the random bit assignment scheme in \cite{pr23}. We use them without any change, so we do \emph{not} explain the intuitions. We refer keen readers to \cite{pr23} for details.

\begin{definition}[Random Bits and Active/Passive Experts, Definition 4.15 in \cite{pr23}] \label{def:active_passive_full}
For any timestep $a \in \calK$, let the random bits $\xi_a = (\xi_{a,1}, \dots, \xi_{a,n})$, where $\xi_{a,i} = (\xi_{a,i,1}, \xi_{a,i,2})$ is used for expert $i (i \in [n])$. The first coordinate $\xi_{a,i,1} \in \{0, 1\}$ is a Bernoulli variable with mean $1/n$ to sample experts. The second part of random bits $\xi_{a,i,2} \in \{0,1\}^{R \times 2L \times 2K}$ are from Bernoulli random variables with mean $\frac{1}{\log^4(nT)}$ to estimate size of the pool and perform the filtering process.

At any timestep $a \in \calK$ an expert $i \in [n]$ is said to be passive, if $\xi_{a,i,2} = \vec{0}$. It is said to be an active expert otherwise.
\end{definition}

Next, we discuss the way to split the sequence of epochs into a collection of disjoint subsequences in the same way of \cite{pr23}.
Here, we only need to argue that the filter set, the alive active experts, and the size estimations can be fixed once the randomness from other sources is fixed. All other steps follow mechanically from \cite{pr23}.
Nevertheless, here, we need to be more careful since there are only two queries each day. 
As such, our main observation here is slightly different from \cite{pr23}.

\begin{observation}[cf. Observation 4.16 in \cite{pr23}; this is different from \cite{pr23}]
\label{obs:loss-sequence-fixing-obs}
Suppose the loss sequence $\{\ell^{t}\}_{t \in [T]}$ and the set of sampled and active experts $\{Y_a\}_{a \in K}$ are fixed.
Furthermore, suppose the \emph{randomness of the second query on each day and the loss sequences} are also fixed.
Then, at any time during the execution of \Cref{alg:monocarpic-boosting}, the estimate size $s$, the filter set $\calF$, and the set of alive active experts are also fixed, regardless of the set of sampled and passive experts.
\end{observation}
Note that Observation~\ref{obs:loss-sequence-fixing-obs} additionally requires the fixing of the randomness of the second query. This is important since the decision on each $\EXP_{k}$ for $k\in [K]$ (not to be confused with $\calK$, which is the time steps) is a function of the second query; if we do not fix such a source of randomness, the different threads might interfere. 

We can now define the eviction time and the epoch assignment algorithm \emph{exactly the same} as in \cite{pr23} as in \Cref{def:eviction_time_full} and \Cref{alg:epoch_assignment}.

\begin{definition}[Eviction time, monocarpic boosting, Definition 4.17 in \cite{pr23}] \label{def:eviction_time_full}
The eviction time of the best expert $i^*$ is defined as follows. 
Assume $i^*$ enters the pool at time step $a\in \calK$, the eviction time $t(a)\in \calK(a) \cup \{+\infty\}$ is defined as the earliest timestep such that $i^*$ is covered by the set of alive active experts at the end of $t(a)$. If $i^*$ would not be covered, then set $t(a) = +\infty$.
\end{definition}

\begin{algorithm}
\caption{Epoch assignment for analysis, Algorithm 12 in \cite{pr23}}
\label{alg:epoch_assignment}
\begin{algorithmic}[1]
\State Initialize $\mathcal{H}_r \leftarrow \emptyset$ ($r \in [R]$), $\tau \leftarrow 1$, $a_1 \leftarrow 0$
\While{$\cup_{\tau' \le \tau} \mathcal{I}_{\tau'} \ne [T]$}
    \If{$t(a_{\tau}) = +\infty$} \Comment{$i^*$ survives till the end}
        \State $\mathcal{H}_{r(a_{\tau})} \leftarrow \mathcal{H}_{r(a_{\tau})} \cup \{a_{\tau}\}$ \Comment{$a_{\tau}$ is a bad epoch at thread $r(a_{\tau})$}
        \If{$r(a_{\tau}) = R$}
            \State $\mathcal{I}_{\tau} \leftarrow \{a_{\tau}\}$, $a_{\tau+1} \leftarrow a_{\tau} \oplus 1$, $\tau \leftarrow \tau + 1$ \Comment{Stop at the top thread}
        \Else
            \State $a_{\tau} \leftarrow (a_{\tau}, 0)$ \Comment{Move to next thread}
        \EndIf
    \Else
        \State $a_{\tau+1} \leftarrow t(a_{\tau}) \oplus 1$, $\mathcal{I}_{\tau} \leftarrow [a_{\tau} : t(a_{\tau})]$, $\tau \leftarrow \tau + 1$
    \EndIf
\EndWhile
\end{algorithmic}
\end{algorithm}

Let $\tau_{max}$ be the total number of intervals in the partition generated by \Cref{alg:epoch_assignment}. The following lemma follows from \cite{pr23} for $\{\mathcal{I}_{\tau}\}_{\tau \in [\tau_{max}]}$.

\begin{lemma}[Lemma 4.81 in \cite{pr23}]
\label{lem:interval_properties}
We have
\begin{itemize}
    \item The intervals $\{\mathcal{I}_{\tau}\}_{\tau \in [\tau_{max}]}$ are disjoint and $\bigcup_{\tau \in [\tau_{max}]} \mathcal{I}_{\tau} = [T]$.
    \item Let $L_1 = \{\mathcal{I}_{\tau}\}_{\tau \in [\tau_{max}]}$ and let $L_r := \{\mathcal{I} \in L_1 : |\mathcal{I}| < B_{r-1}\}$ ($r \in [2:R]$) contain intervals of length less than $B_{r-1}$, then
    \[ \sum_{\mathcal{I} \in L_r} |\mathcal{I}| \le |\mathcal{H}_{r-1}| \cdot B_{r-1}. \]
\end{itemize}
\end{lemma}

Following the flow of \cite{pr23}, next, we prove the size of bad epochs $\mathcal{H}_r$ is at most $\polylog{nT}$ with high probability.
\begin{lemma}[cf. Lemma 4.19 in \cite{pr23}]
\label{lem:bad_epoch_size_full}
With probability at least $1 - 1/\poly(nT)$, $|\mathcal{H}_r| \leq n \polylog(nT)$ holds for any $r \in [R]$.
\end{lemma}

\begin{proof}
The proof is quite similar to the proof of \lemref{lem:bad:epochs} and the proof of Lemma 4.19 in \cite{pr23}, and we only give the sketch to avoid unnecessary repetitions.
We first condition on the loss sequence of $\ell^{1}, \ell^{2}, \cdots, \ell^{T}$, the sampled and active experts $\{Y_a\}_{a\in \calK}$, and all the randomness used for the second query of the algorithm. Then, we can argue that once $i^*\in W_a$, i.e., the set of sampled and passive expert of timestep $a$, it would survive till the end. Furthermore, we could also obtain
\begin{align*}
    \PPr{i^*\in W_a\,|\,Y_1,\ldots,Y_{T/B},\ell^1,\ldots,\ell^T, \text{randomness of the second query}}\geq \frac{1}{2n}
\end{align*}
by the same argument as in \lemref{lem:bad:epochs}. Therefore, we obtain that for any fixed $r$, we have that
\begin{align*}
    \Pr\paren{\card{\calP}\leq \frac{\card{\calH_r}}{4n}\mid Y_1,\ldots,Y_{T/B},\ell^1,\ldots,\ell^T, \text{randomness of the second query}} \leq \exp\paren{-\card{\calH_r}/16n}.
\end{align*}
By the bound of $\card{\calP}\leq \polylog{nT}$ as in \Cref{lem:arm-num-bound-monocarpic-boosting}, we conclude that with probability at least $1-1/\poly(nT)$, we have $\card{\calH_r} \leq n \cdot \polylog(nT)$, as desired.
\end{proof}

Next, in the same way as \cite{pr23}, we apply the guarantees for $\monocarpic$ (\Cref{prop:monocarpic-alg-property}) to bound the costs on epochs that are \emph{not bad on thread $R$}.

\begin{lemma} \label{lem:interval_regret_full}
With probability at least $1 - 1/\poly(nT)$, for any $\tau \in [\tau_{max}]$ and $a_{\tau} \notin \mathcal{H}_R$,
\[ \sum_{t \in \mathcal{I}_{\tau}} \ell^{t}(i_t) - \sum_{t \in \mathcal{I}_{\tau}} \ell^{t}(i^*) \le \sqrt{\card{\mathcal{I}_{\tau}}} \cdot \polylog(nT).\]
\end{lemma}

\begin{proof}
Similar to the proof in \cite{pr23}, we condition on the event of \Cref{lem:arm-num-bound-monocarpic-boosting}. Given any interval $\mathcal{I}_{\tau}$ starting with $a_{\tau}$, ending with $t(a_{\tau})$ and $a_{\tau} \notin \mathcal{H}_R$, the expert $i^*$ with entering time $a_{\tau}$ is covered by $\mathcal{P}$ at the end of $t(a_{\tau})$. Let $i_1^*, \dots, i_s^*$ be the set of experts that cover $i^*$. Note that with the cover notion of cover in \Cref{def:cover} and with the correct choice of $\rho$, we we can partition the interval $\mathcal{I}_{\tau} = \mathcal{I}_{\tau,1} \cup \dots \cup \mathcal{I}_{\tau,s}$, and obtain
\begin{align*}
\sum_{t \in \mathcal{I}_{\tau}} \widetilde{\calL^{\mathcal{I}_{\tau}}}(i^*) \ge \sum_{j=1}^{s} \sum_{t \in \mathcal{I}_{\tau,j}} \widetilde{\calL^{\mathcal{I}_{\tau}}}(i_j^*) - C\log{n} \cdot \sqrt{\card{\mathcal{I}_{\tau}}}
\end{align*}
Furthermore, by further conditioning on the high-probability event in \Cref{lem:monocarpic-arm-loss-estimation}, we have
\begin{align*}
    \card{\widetilde{\calL^{\mathcal{I}_{\tau}}}(i) - \sum_{t\in \mathcal{I}_{\tau}} \ell^{t}(i)} \leq \sqrt{\card{\mathcal{I}_{\tau}}} \cdot \polylog(nT).
\end{align*}
Therefore, we can lower bound the cost of $i^*$ as a function of the costs of the filter arms, i.e.,
\begin{equation} \label{eq:covering-loss-monocarpic}
\sum_{t \in \mathcal{I}_{\tau}} \ell^{t}(i^*) \geq \sum_{j=1}^{s} \paren{\sum_{t \in \mathcal{I}_{\tau,j}} \ell^{t}(i_j^*) - C_1\cdot \sqrt{\mathcal{I}_{\tau, j}}\cdot \polylog(nT)}
\geq \sum_{j=1}^{s}\sum_{t \in \mathcal{I}_{\tau,j}} \ell^{t}(i_j^*) - C_2 \cdot \sqrt{\mathcal{I}_{\tau}}\cdot \polylog(nT)
\end{equation}
for some constant $C_1$ and $C_2$ such that $C_2 \geq C_1$.
In the calculation, we used $s\leq \card{\calP} \leq \polylog(n)$ for the second step.
Therefore, by the regret guarantee of $\monocarpic$, with probability at least $1-1/\poly(nT)$, we have that
\begin{align*}
\sum_{t \in \mathcal{I}_{\tau}} \ell^{t}(i_t) - \sum_{t \in \mathcal{I}_{\tau}} \ell^{t}(i^*) &= \sum_{t \in \mathcal{I}_{\tau}} \ell^{t}(i_t) - \sum_{j=1}^{s} \sum_{t \in \mathcal{I}_{\tau,j}} \ell^{t}(i_j^*) + \sum_{j=1}^{s} \sum_{t \in \mathcal{I}_{\tau,j}} \ell^{t}(i_j^*) - \sum_{t \in \mathcal{I}_{\tau}} \ell^{t}(i^*) \\
&= \sum_{j=1}^{s}\sum_{t \in \mathcal{I}_{\tau,j}} \ell^{t}(i_t) - \sum_{j=1}^{s} \sum_{t \in \mathcal{I}_{\tau,j}} \ell^{t}(i_j^*) + \sum_{j=1}^{s} \sum_{t \in \mathcal{I}_{\tau,j}} \ell^{t}(i_j^*) - \sum_{t \in \mathcal{I}_{\tau}} \ell^{t}(i^*) \tag{intervals are disjoint}\\
&\leq \sum_{j=1}^{s} \sqrt{|\mathcal{I}_{\tau,j}|}\cdot \polylog(nT) + \sum_{j=1}^{s} \sum_{t \in \mathcal{I}_{\tau,j}} \ell^{t}(i_j^*) - \sum_{t \in \mathcal{I}_{\tau}} \ell^{t}(i^*) \tag{by \Cref{prop:monocarpic-alg-property}} \\
&\le \sum_{j=1}^{s} \sqrt{|\mathcal{I}_{\tau,j}|}\cdot \polylog(nT) + \sqrt{\mathcal{I}_{\tau}}\cdot \polylog(nT) \tag{by \Cref{eq:covering-loss-monocarpic}}\\
&\le \sqrt{\mathcal{I}_{\tau}}\cdot \polylog(nT),
\end{align*}
where the last step again used the fact $s\leq \card{\calP} \leq \polylog(n)$. This is as desired by the lemma statement.
\end{proof}

We now bound the regret of \Cref{alg:monocarpic-boosting} in the same manner of \cite{pr23}.

\begin{lemma} \label{lem:full_algo_regret}
With probability at least $1 - 1/\poly(nT)$, the regret of \Cref{alg:monocarpic-boosting}  is at most
\begin{align*}
\sum_{t \in [T]} \ell^{t}(i_t) - \sum_{t \in [T]} \ell^{t}(i^*) \leq \sqrt{nT}\cdot \polylog(nT).
\end{align*}
\end{lemma}

\begin{proof}
The proof follows from the same argument as in \cite{pr23} with the changes in the statements we made.
Same as their proof, we first fix the loss sequence $\{\ell^{t}\}_{t \in [T]}$ and the set of sampled and active experts $\{Y_a\}_{a \in \calK}$. 
Furthermore, we fix the randomness of the second query, which allows the regret to be split to a collection of intervals $\{\mathcal{I}_{\tau}\}_{\tau \in [\tau_{max}]}$ with epoch assignment algorithm.

We now categorize the regrets to whether their length is more than $B_{R-1}$. With probability at least $1 - 1/\poly(nT)$, we have
\begin{align*}
\sum_{t \in [T]} \ell^{t}(i_t) - \sum_{t \in [T]} \ell^{t}(i^*) &= \sum_{r=1}^{R-1} \sum_{\mathcal{I} \in L_r \setminus L_{r+1}} \sum_{t \in \mathcal{I}} (\ell^{t}(i_t) - \ell^{t}(i^*)) + \sum_{\mathcal{I} \in L_R} \sum_{t \in \mathcal{I}} (\ell^{t}(i_t) - \ell^{t}(i^*))  \tag{by using $[T] = \cup_{\mathcal{I} \in L_1} \mathcal{I}$ of \Cref{lem:interval_properties}}\\
&\le \sum_{r=1}^{R-1} \sum_{\mathcal{I} \in L_r \setminus L_{r+1}} \sum_{t \in \mathcal{I}} (\ell^{t}(i_t) - \ell^{t}(i^*)) + \O{|\mathcal{H}_{R-1}|}. \tag{by using the second property of \Cref{lem:interval_properties}}
\end{align*}

By applying \Cref{lem:interval_regret_full}, we could obtain that
\begin{equation}
 \label{eq:regret-decompose-final-monocarpic-lemma}
    \sum_{t \in [T]} \ell^{t}(i_t) - \sum_{t \in [T]} \ell^{t}(i^*)  \leq \sum_{r=1}^{R-1} \sum_{\mathcal{I} \in L_r \setminus L_{r+1}} \sqrt{|\mathcal{I}|} \cdot \polylog(nT) + \O{|\mathcal{H}_{R-1}|}.
\end{equation}
By exactly the same calculation as in \cite{pr23}, there is 
\begin{align*}
    \sum_{\mathcal{I} \in L_r \setminus L_{r+1}} \sqrt{|\mathcal{I}|} \leq 
    \begin{cases} 
    |\mathcal{H}_{r-1}| \cdot \sqrt{2T/n} & r \ge 2 \\ 
    \sqrt{nT} & r=1 
    \end{cases}
\end{align*}
for any $r\in [r-1]$. 
Therefore, we can now use \Cref{eq:regret-decompose-final-monocarpic-lemma} to finally bound the regret as
\begin{align*}
    \sum_{t \in [T]} \ell^{t}(i_t) - \sum_{t \in [T]} \ell^{t}(i^*)  & \leq \sum_{r=1}^{R-1} \sum_{\mathcal{I} \in L_r \setminus L_{r+1}} \sqrt{|\mathcal{I}|} \cdot \polylog(nT) + \O{|\mathcal{H}_{R-1}|} \tag{\Cref{eq:regret-decompose-final-monocarpic-lemma}}\\
    &\leq \sqrt{nT}\cdot \polylog(nT) + \sqrt{\frac{2T}{n}}\cdot \polylog(nT)\cdot \sum_{r=2}^{R}\sum_{\calI\in L_r\setminus L_{r+1}} \card{\calH_{r-1}}\\
    &\leq \sqrt{nT}\cdot \polylog(nT) + \sqrt{\frac{2T}{n}}\cdot \polylog(nT)\cdot R\cdot n\polylog(nT) \tag{by \Cref{lem:bad_epoch_size_full}}\\
    &\leq \sqrt{nT}\cdot \polylog(nT),
\end{align*}
which is as desired by the lemma statement.
\end{proof}

\paragraph{Finalizing the proof of \Cref{thm:two-query}.} Combining \Cref{lem:monocarpic-optimal-boosting-space} (for space) and \Cref{lem:full_algo_regret} (for regret) leads the desired statement as in \Cref{thm:two-query}.

\subsection{Discussion on interval regret of \texorpdfstring{$\tilde{O}(\sqrt{nT})$}{sqrtnt}}
We briefly discuss how to generalize the above result to the interval regret to lead to \Cref{thm:two-query-interval}. 

\paragraph{The algorithm.} The algorithm for interval regret follows the same structure of \Cref{alg:two-query-interval-regret} (not to be confused with \Cref{alg:interval-regret-monocarpic}).
We maintain the same structure as in \Cref{alg:two-query-interval-regret} with $\Nmeta = \log(T)$ different interval lengths and initialize $\ALG_{\kappa}$ as a copy of \Cref{alg:monocarpic-boosting} with a number of days as $2^{\kappa}$.
We let $\{w_t(\kappa)\}_{\kappa=1}^{\Nmeta}$ to be the weights on each interval algorithm. 
During each day, we play an interval algorithm following the distribution over the $\frac{w_t(\kappa)}{\sum_{\kappa}w_t(\kappa)}$, and follow the decision of $\ALG_{\kappa}$ if it is sampled (i.e., follow the decision of \Cref{alg:monocarpic-regret}). However, we skipped all \emph{update} steps here. 
On the other hand, for the update step, we use the following steps. 
We toss a fair coin, and if the coin displays ``head'' (with probability $1/2$), we sample an arm uniformly at random and use the loss $\ell^{t}(i)$ to update \Cref{line:arm-loss-update} of \Cref{alg:monocarpic-boosting}.
If the coin displays ``tail'' (with probability $1/2$), we update \Cref{alg:interval-regret-monocarpic} using $\flag=\textnormal{Fake}$.
All the update steps for pool management (i.e., \textsc{Merge} and membership updates) are conducted by using the estimations $\widetilde{\calL}(i)$ for each arm $i$.

\paragraph{The analysis.} This follows the same logic we used in the proof of \Cref{thm:interval-regret-low-space}.
The first step is to show that the correctness of the algorithm continues to hold with the modifications.
Indeed, the guarantees in \Cref{lem:monocarpic-arm-loss-estimation} remain true since we keep sampling arms uniformly at random.
For the updates in \Cref{alg:interval-regret-monocarpic}, we have a proof in \Cref{lem:monocarpic-boosting-interval-regret} that the updates using the estimated losses give the desired interval regret on the alive expert. 
Furthermore, all other update steps are \emph{unchanged}, which means the space and regret bounds continue to hold.

The last missing piece of the analysis is to prove the correctness of the outer algorithm, i.e., to obtain the guarantees in \Cref{lem:interval-regret-outer-algorithm}.
In the proof, the key property we used is for each arm $i$ in the pool, we have that the arm $i$ is sampled in the update step is at least $\frac{1}{\polylog(nT)}$. 
Note that by \Cref{lem:arm-num-bound-monocarpic-boosting}, this property remains true in the algorithm described above. As such, we conclude the analysis of \Cref{thm:two-query-interval}.

\FloatBarrier

\section*{Acknowledgments}
David P. Woodruff is supported in part Office of Naval Research award number N000142112647, a Simons Investigator Award, and NSF CCF-2335412. 
Samson Zhou is supported in part by NSF CCF-2335411. 
Samson Zhou gratefully acknowledges funding provided by the Oak Ridge Associated Universities (ORAU) Ralph E. Powe Junior Faculty Enhancement Award.

\def\shortbib{0}
\bibliographystyle{alpha}
\bibliography{references}

\appendix
\section{Omitted Proofs in \texorpdfstring{\Cref{sec:prelim}}{secprelim}}
\label{app:omitted-proof-prelim}
We remark that the statements and proofs of \Cref{prop:exp3-exploration} and \Cref{prop:two-query-EXP3} are standard; we include them here for the sake of completeness.

\subsection*{Proof of \Cref{prop:exp3-exploration}}
\begin{proof}
    The first step is to decompose the regret into the regrets on the exploration and exploitation steps.
    Let $\calD_{\textnormal{explore}}$ be the exploration days in \Cref{alg:EXP3-exploration}, and $\calD_{\textnormal{exploit}}$ be the days \Cref{alg:EXP3-exploration} sample arms from distribution.
    Since the losses are in $\{0,1\}$, and there are $\gamma T$ steps of exploration in expectation, we could straightforwardly bound that
    \begin{align*}
        \expect{\sum_{t\in \calD_{\textnormal{explore}}} \ell^t(i_t)-\sum_{t\in \calD_{\textnormal{explore}}} \ell^t(i^*)}\leq \gamma T.
    \end{align*}
    What is left to be proved is the regret on the exploitation days.  On an exploitation date, let us consider an imaginary loss sequence $\{\widehat{\ell^{t}}(i)\}$ being obtained by the following process:
    \begin{align*}
        \widehat{\ell^{t}}(i) = 
        \begin{cases}
            \widetilde{\ell^{t}}(i), \text{ if $t \in \calD_{\textnormal{explore}}$};\\
            0, \text{ otherwise}
        \end{cases}
    \end{align*}
    We then decompose the regret to  
    \begin{align*}
        \expect{\sum_{t\in \calD_{\textnormal{exploit}}} \ell^t(i_t)-\sum_{t\in \calD_{\textnormal{exploit}}} \ell^t(i^*)} \leq \expect{\sum_{t=1}^{T} \ell^t(j_t)-\sum_{t=1}^{T} \widehat{\ell^{t}}(j_t)} + \expect{\sum_{t=1}^{T} \widehat{\ell^{t}}(i_t)-\sum_{t=1}^{T} \ell^t(i^*)}.
    \end{align*}
    We bound $\expect{\sum_{t=1}^{T} \ell^t(j_t)-\sum_{t=1}^{T} \widehat{\ell^{t}}(j_t)}$ using concentration inequality. In particular, note that $\widehat{\ell^{t}}$ $\widehat{\calL^{t}}:=\sum_{\tau=1}^{t}\widehat{\ell^{t}}$ are unbiased estimators on day $t$ with a low variance, which means
    \begin{align*}
    \expect{\widehat{\ell^{t}}}=\ell^t(i),\qquad \expect{(\widehat{\ell^t(i)})^2}=(\ell^t(i))^2\cdot \frac{n}{\gamma}.
    \end{align*}
    Therefore, by applying Bernstein's inequality, we have
    \begin{align*}
        \Pr\paren{\card{\widehat{\calL^{t}}(i)-\sum_{t=1}^{T}\ell^t(i)} \geq \sqrt{T/\gamma}\cdot n \cdot \polylog(nT)}\leq \frac{1}{\poly(nT)},
    \end{align*}
    which gives the regret bound for the first term.

    To bound the regret for the second term, i.e., $\expect{\sum_{t=1}^{T} \widehat{\ell^{t}}(i_t)-\sum_{t=1}^{T} \ell^t(i^*)}$, we use a partial-to-full reduction type of argument to bound the regret (in the same spirit of proving EXP3 regret using MWU).
    We first recall the guarantees of the full information \hedge algorithm as follows.
    \begin{proposition}
        \label{prop:hedge}
        Consider the following \hedge algorithm:
        \begin{tbox}
            Algorithm \hedge with learning rate $\eta$
            \begin{itemize}
                \item At time $t$, the algorithm receive loss vector $\overline{\ell^{t}} \in [0, U]^n$.
                \item Let $\calL^{t}(i)$ be the total loss of $i$ until time $t$.
                \item The algorithm plays $i_t$ by sampling from the following distribution.
                \begin{align*}
                    \overline{P_t}(i) = \frac{\exp(-\eta \cdot \calL^{t}(i))}{\sum_{i=1}^n \exp(-\eta \cdot \calL^{t}(i))}.
                \end{align*}
                \item The algorithm update all $\calL^{t+1}(i)$ for $i\in [n]$ by observing losses in $\overline{\ell^{t}}$. 
            \end{itemize}
        \end{tbox}
        The \hedge algorithm achieves an expected regret of at most $\O{\sqrt{UT \log{n}}}$ as long as 
        \[\expect{\sum_{i=1}^{n}\overline{P_t}(i) (\overline{\ell^{t}}(i))^2} \leq U,\]
        even if the losses are chosen adaptively.
    \end{proposition}
    
    We observe that essentially, \Cref{alg:EXP3-exploration} could be considered as a \hedge algorithm with $\overline{\ell^t}$ as $\widehat{\ell^{t}}$ and $\eta=\gamma$.
    Since $\widehat{\ell^{t}}$ is an unbiased estimator of $\ell^{t}$, let $R^{\hedge}$ be the regret of the hedge algorithm with the above setting, we have
    \[\expect{\sum_{t=1}^{T} \widehat{\ell^{t}}(i_t)-\sum_{t=1}^{T} \ell^t(i^*)} = \expect{R^{\hedge}}.\]
    Therefore, what remains is to bound the regret of the \hedge algorithm.
    By the rule of sampling, we have $\widehat{\ell^{t}}\leq n/\gamma$. 
    Furthermore, we could bound $\expect{\sum_{i=1}^{n}\overline{P_t}(i) (\overline{\ell^{t}}(i))^2}$ as follows.
    \begin{align*}
        \expect{\sum_{i=1}^{n}\overline{P_t}(i) (\overline{\ell^{t}}(i))^2}& =\expect{\sum_{i=1}^{n}P_t(i) (\widehat{\ell^{t}}(i))^2} = \expect{\sum_{i=1}^n\frac{\gamma}{n}\cdot (\frac{n}{\gamma})^2}\leq n^2/\gamma.
    \end{align*}
    Therefore, by \Cref{prop:hedge}, the expected regret for the second part is also at most $\O{n\sqrt{T\log{n}/\gamma}}$. 
    This implies 
    \begin{align*}
    \expect{\sum_{t\in \calD_{\textnormal{exploit}}} \ell^t(i_t)-\sum_{t\in \calD_{\textnormal{exploit}}} \ell^t(i^*)} & \leq \expect{\sum_{t=1}^{T} \ell^t(j_t)-\sum_{t=1}^{T} \widehat{\ell^{t}}(j_t)} + \expect{\sum_{t=1}^{T} \widehat{\ell^{t}}(i_t)-\sum_{t=1}^{T} \ell^t(i^*)}\\
    & \leq n\sqrt{\frac{T\log{n}}{\gamma}} \cdot \polylog(nT).
    \end{align*}

    Combining this with the loss on the exploration days gives the desired result.
\end{proof}

\subsection*{Proof of \Cref{prop:two-query-EXP3}}
\begin{proof}
    We use the following lemma from \cite{Lu0CZWH24} to prove \Cref{prop:two-query-EXP3}.
    \begin{lemma}[\cite{Lu0CZWH24}, rephrased]     
        \label{lem:inverse-propensity-regret}
        Let $\hat{\ell}_t$ be an unbiased estimator of the loss vector $\ell^t$, such that for some distribution $z_t$, with probability $z_t(i)$ $\hat{\ell}_t(i)=\frac{1}{z_t(i)}\ell^t(i)$ and $\hat{\ell}_t(j)=0$ for $j\neq i$.
        Furthermore, suppose $z_t(i)$ satisfies $P_t(i)\leq C\cdot z_t(i)$ for all $i$.
        Then, \Cref{alg:EXP3-exploration-two-query} using a learning rate $\eta = \sqrt{\log{n}/CTn}$ attains an expected regret of $\O{\sqrt{CnT\log{n}}}$.
    \end{lemma}
    Since we sample each arm uniformly at random, we have that $z_t(i)=1/n$. Since $P_t(i)\leq n$ for any $i$ and $t$, we have $C\leq n$, which gives $\O{n\cdot \sqrt{T\log{n}}}$ expected regret by setting $\eta=\frac{1}{n}\cdot \sqrt{\frac{\log{n}}{T}}$.
\end{proof}

\end{document}